\documentclass[ba,preprint]{imsart}
\pubyear{2021}
\volume{TBA}
\issue{TBA}
\firstpage{1}
\lastpage{1}

\usepackage[utf8]{inputenc} 
\usepackage[T1]{fontenc}    

\usepackage[UKenglish]{babel}

\usepackage{microtype}      
\usepackage{subfigure}
\usepackage{booktabs}       
\usepackage{nicefrac}       

\usepackage{url}            

\usepackage{multicol} 

\usepackage{todonotes}

\usepackage{import} 
\usepackage{afterpage}

\usepackage[colorlinks,citecolor=blue,urlcolor=blue,filecolor=blue,backref=page]{hyperref}

\usepackage{graphicx}
\graphicspath{{images/}} 
\usepackage[section]{placeins}

\usepackage{algorithm}
\usepackage{algorithmic}

\usepackage{amsmath,amssymb,amsthm,amsfonts} 
\usepackage{arydshln} 
\usepackage{bbm} 
\usepackage{bm} 

\usepackage{cancel}

\usepackage{cleveref}
\theoremstyle{definition}
\newtheorem{definition}{Definition}

\theoremstyle{remark}
\newtheorem{remark}{Remark}

\theoremstyle{definition}

\crefname{assumption}{assumption}{assumptions}

\DeclareMathOperator*{\argmin}{arg\,min}

\DeclareMathOperator{\E}{\mathbb{E}}
\DeclareMathOperator{\elbo}{\text{ELBO}}
\DeclareMathOperator{\smi}{\text{smi}}
\DeclareMathOperator{\pow}{\text{pow}}
\DeclareMathOperator{\cut}{\text{cut}}
\DeclareMathOperator{\msmi}{\text{msmi}}
\DeclareMathOperator{\elpd}{\text{elpd}}
\DeclareMathOperator{\bayes}{\text{bayes}}
\def\bmo{{}} 

\usepackage{mathtools}
\DeclarePairedDelimiterX{\infdivx}[2]{(}{)}{%
  #1\;\delimsize\|\;#2%
}
\newcommand{\kl}{D_{KL}\infdivx}
\usepackage[T1]{fontenc}
\usepackage{titlesec, blindtext, color}
\newcommand{\hsp}{\hspace{20pt}}
\titleformat{\chapter}[hang]{\Huge\bfseries}{\thechapter\hsp\textcolor{gray}{|}\hsp}{0pt}{\Huge\bfseries}
\usepackage[titletoc]{appendix} 
\usepackage{etoolbox}
\AtBeginEnvironment{appendices}{\crefalias{section}{appendix}}

\usepackage[acronym]{glossaries}
\newacronym{smi}{SMI}{Semi-Modular Inference}
\newacronym{vsmi}{VSMI}{Variational Semi-Modular Inference}
\newacronym{sgd}{SGD}{Stochastic Gradient Descent}
\newacronym{vi}{VI}{Variational Inference}
\newacronym{vae}{VAE}{Variational Auto-Encoder}
\newacronym{mfvi}{MFVI}{Mean Field Variational Inference}
\newacronym{mcmc}{MCMC}{Markov Chain Monte Carlo}
\newacronym{bbvi}{BBVI}{Black Box Variational Inference}
\newacronym{svi}{SVI}{Stochastic Variational Inference}
\newacronym{dsvi}{DSVI}{Doubly Stochastic Variational Inference}
\newacronym[longplural={Gaussian Processes}]{gp}{GP}{Gaussian Process}
\newacronym{elbo}{ELBO}{Evidence Lower Bound}
\newacronym{elpd}{ELPD}{Expected Log-pointwise Predictive Density}
\newacronym{gbi}{GBI}{Generalised Bayesian Inference}
\newacronym{lpd}{LPD}{Log-pointwise Predictive Density}
\newacronym{nf}{NF}{Normalizing Flow}
\newacronym{iaf}{IAF}{Inverse Autoregressive Flow}
\newacronym{kl}{KL}{Kullback-Leibler Divergence}
\newacronym{vmp}{VMP}{Variational Meta-Posterior}
\newacronym{nsf}{NSF}{Neural Spline Flow}
\newacronym{mlp}{MLP}{Multi-Layer Perceptron}
\newacronym{hpv}{HPV}{Human Papilloma Virus}

\usepackage{algorithm}
\usepackage{algorithmic}

\usepackage{natbib} 
\usepackage{breakcites} 

\startlocaldefs
\numberwithin{equation}{section}
\theoremstyle{plain}

\newtheorem{proposition}{Proposition}
\def\Y{\mathcal{Y}}
\def\Z{\mathcal{Z}}
\def\phispace{{\Omega_\Phi}}
\def\thetaspace{{\Omega_\Theta}}
\endlocaldefs

\begin{document}

\begin{frontmatter}
  \title{Scalable Semi-Modular Inference with Variational Meta-Posteriors.}
  \runtitle{Scalable Modular Bayesian Inference}

  \begin{aug}
    \author{
      \fnms{Chris U.} \snm{Carmona}%
      \thanksref{addr1,t1}%
      \ead[label=e1]{carmona@stats.ox.ac.uk}%
      \ead[label=u1,url]{chriscarmona.me}
    }
    \and
    \author{
      \fnms{Geoff K.} \snm{Nicholls}%
      \thanksref{addr1}%
      \ead[label=e2]{nicholls@stats.ox.ac.uk}%
    }
    \runauthor{Carmona and Nicholls}

    \address[addr1]{
      Department of Statistics,
      University of Oxford,
      Oxford, UK \\
      \printead{e1} \printead{u1},
      \printead{e2}
    }

    \thankstext{t1}{Research supported by the Mexican Council for Science and Technolog (CONACYT) and The Central Bank of Mexico (Banco de Mexico).}

  \end{aug}

  \begin{abstract}
    The Cut posterior and related \acrfull*{smi} are Generalised Bayes methods for Modular Bayesian evidence combination. Analysis is broken up over modular sub-models of the joint posterior distribution. Model-misspecification in multi-modular models can be hard to fix by model elaboration alone and the Cut posterior and \acrshort*{smi} offer a way round this. Information entering the analysis from misspecified modules is controlled by an influence parameter $\eta$ related to the learning rate. This paper contains two substantial new methods. First, we give variational methods for approximating the Cut and \acrshort*{smi} posteriors which are adapted to the inferential goals of evidence combination. We parameterise a family of variational posteriors using a \acrlong*{nf} for accurate approximation and end-to-end training. Secondly, we show that analysis of models with multiple cuts is feasible using a new \acrlong*{vmp}. This approximates a family of \acrshort*{smi} posteriors indexed by $\eta$ using a single set of variational parameters.
  \end{abstract}



\end{frontmatter}

\section{Introduction} \label{sec:intro}

Evidence combination is a fundamental operation of statistical inference.
When we have multiple observation models for multiple data sets, with some model parameters appearing in more than one observation model, we have a multi-modular setting in which data sets identify modules.
The modules are connected in the graphical model for the joint posterior distribution of the parameters.

Large-scale multi-modular models are susceptible to model contamination, as a hazard for misspecification in each module accumulates as modules are added.
If any module is significantly misspecified, it may undermine inference on the joint model \citep{Liu2009modularization}.
Model elaboration \citep{Smith1986,Gelman2014BDA3} may be impractical or at least very challenging.
In this setting we may consider ``inference elaboration'', and turn to statistically principled alternatives to Bayesian inference such as Generalised Bayes \citep{Zhang2006,Grunwald2017a,Bissiri2016}.
Some modular inference frameworks allow the analyst to break up the workflow, whilst still implementing a valid belief update \citep{Bissiri2016,Nicholls2022smi}.
This is discussed in \cite{Nicholson2021covid} in the broader context of modular ``interoperability''.

Modular Bayesian Inference \citep{Liu2009modularization, Plummer2015cut,Jacob2017together, Carmona2020smi,Nicholls2022smi} addresses misspecification in a multi-modular setting by controlling feedback from misspecified modules (see \cref{sec:mod_bayes}).
Recent applications of multi-modular inference note \citep[eg.][]{Nicholson2021covid, Teh2021covid} and demonstrate \citep[eg.][]{Carmona2020smi,Yu2021variationalcut,Styring2022urban} the potential benefits of partially down-weighting the influence of modules, rather than completely removing feedback.
Modular Bayesian Inference is characterised by a modified posterior known as the \emph{Cut} posterior \citep{Plummer2015cut}. This removes feedback from identified misspecified modules. \emph{Semi-Modular} posteriors \citep{Carmona2020smi} interpolate between Bayes and Cut, controlling feedback using an influence parameter $\eta\in [0,1]$. This can be identified with the learning rate parameter in a power posterior \citep{Walker2001,Zhang2006,Grunwald2017a}.
The Cut and \acrshort*{smi} posteriors are valid belief updates in the sense of \citep{Bissiri2016}, and part of a larger family of valid inference procedures \citep{Nicholls2022smi}.

These approaches raise computational challenges due to intractable parameter-dependent marginal factors.
In the case of Cut posteriors, nested Monte Carlo samplers are given in \cite{Plummer2015cut} and
\cite{Liu2020sacut} and
used in \cite{Carmona2020smi} to target \acrshort*{smi} posteriors.
These samplers suffer from double asymptotics, though work well in practice on some target posterior distributions.
\cite{Jacob2020couplings} give unbiased Monte Carlo samplers for the Cut posterior and \cite{Pompe2021cut} analyse the asymptotics of the Cut posterior and give two methods to target it: a Laplace approximation and Posterior Bootstrap.

Work to date on Modular Inference focuses on models with a small number of modules and a single pre-specified cut module.
We simultaneously adjust the contribution of multiple modules. This allows us to take an exploratory approach and ``discover'' the misspecified modules.
Dealing with multiple cuts and a vector of \emph{influence parameters} is challenging, as each additional ``cut'' increases the dimension of $\eta$ and the space of candidate posteriors.
One natural approach for selecting a candidate posterior is to take a grid of $\eta$-values, sample each distribution (Importance Sampling is not straightforward, as ratios of candidate posteriors are intractable) and evaluate a performance metric on each distribution.
This search strategy works for a single cut, but is already inefficient and quickly becomes cumbersome when the number of cuts increases.

In this work, we give a novel variational framework for \acrshort*{smi} posteriors
which scales to handle multiple cuts.
The usual \acrfull*{elbo} training utility is intractable, due to the same parameter-dependent marginals that make \acrshort*{mcmc} sampling difficult.
Moreover, the resulting approximation does not meet the original objective of having controlled feedback between modules (see \cref{sec:vi_modular}).
Our solution takes a variational family with a pattern of conditional independence  between \emph{shared, extrinsic} and \emph{module-specific, intrinsic} parameters that matches the \acrshort*{smi} target, and uses the \emph{stop-gradient} operator to define a modified variational objective. The resulting variational framework gives good approximation, controllable feedback and end-to-end optimisation.
In parallel independent work, \cite{Yu2021variationalcut} give variational methods for Cut-posteriors.
Our approaches match at the Cut posterior: just as \acrshort*{smi} interpolates Cut and Bayes, so variational \acrshort*{smi} interpolates variational-Cut and variational-Bayes exactly.

One of the goals of \acrshort*{smi} is to correct for model misspecification, so it is important to get a good variational fit to the \acrshort*{smi}-posterior and not make matters worse with a poor approximation.
We leverage recent work on relatively expressive variational families using \acrfullpl*{nf} \citep{Rezende2015nf, Papamakarios2021normalizing}, as we get better uncertainty quantification than less expressive mean-field approximations.
In particular, we take Flow-based models with universal-approximation capabilities \citep[see][]{Huang2018flows,Durkan2019neural,Papamakarios2021normalizing} as our default variational families.
The conditional independence structure required by the \acrshort*{smi} posterior is achieved by defining the \emph{Conditioner} functions of the flow.

We exploit the continuity of the \acrshort*{smi} posterior with varying $\eta$ and introduce the \acrfull*{vmp}, a variational approximation to the \emph{entire collection} of posteriors indexed by $\eta$, using a single set of parameters.
We train a function that takes $\eta$ as input and produces the variational parameters for the corresponding \acrshort*{smi} posterior.
We call this function the \acrshort*{vmp}-map.
The \acrlong*{vmp} is key to scalability (as illustrated in our example with 30 potential cuts in \cref{subsec:exp_rnd_eff}).

The remaining task is to select an \acrshort*{smi} posterior ($ie$, $\eta$) for downstream analysis.
The performance metric deciding the level of influence will depend on the inferential goals.
Selection criteria \citep{wu-martin-20} developed for choosing the learning rate in the power posterior, such as matching information gain \citep{holmes-walker-17}, and predictive performance \citep{Vehtari2016,Jacob2017together,wu-martin-21} are relevant.
\cite{Yu2021variationalcut} leverage tractable variational distributions to compute calibrated test statistics \citep{nott21-prior-check} measuring evidence against Bayes and for Cut.
We use the \acrfull*{elpd} \citep{Vehtari2016}, which scores predictive performance.
Variational methods commonly achieve predictive accuracy comparable with \acrshort*{mcmc} despite the variational approximation \citep{Wang2019VariationalMisspecification} so this is a happy marriage.
We estimate the \acrshort*{elpd} using the WAIC \citep{Watanabe2012}. Fast sampling is available for the variational posterior density and this supports \acrshort*{elpd}-estimation for multiple cuts.


In summary, our contributions include:\\[-0.25in]
\begin{itemize}
  \item a variational framework for approximation of \acrshort*{smi} posteriors suitable for modular Bayesian inference;
  \item approximation of \acrshort*{smi} posteriors with \acrlongpl*{nf}, underlining the importance of flexible variational families;
  \item the \acrfull*{vmp}, a family of variational posteriors indexed by $\eta$ which approximates a family of \acrshort*{smi} posteriors using a single set of parameters;
  \item end-to-end training algorithms using the \emph{stop-gradient} operator;
  \item variational methods for identifying misspecified modules and modulating feedback which scale to handle multiple cuts;
  \item illustrations of the method on real and synthetic data.\\[-0.25in]
\end{itemize}
We provide code reproducing all results and figures \footnote{\url{https://github.com/chriscarmona/modularbayes}}.

\section{Modular Bayesian Inference} \label{sec:mod_bayes}

In order to fix ideas, we illustrate multi-modular inference using the model structure displayed in \cref{fig:toy_multimodular_model}. 
This structure is already quite rich, as more complex models may sometimes be reduced to this form by grouping together sub-modules into nodes appropriately.
Our methods extend straightforwardly to more complex models in a similar fashion to earlier work in this field.

\begin{figure}[!htb]
  \centering
  \def\svgwidth{0.25\textwidth}
  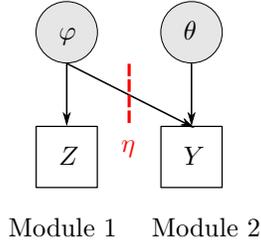
  \caption{
    Graphical representation of a simple multi-modular model.
    Grey circles denote unknown quantities to be inferred, and white boxes are fixed quantities.
    The dashed red line indicates cut feedback from the $Y$-module into the $Z$-module.
    The addition of the $\eta$ symbol indicates modulated feedback via \acrshort*{smi}.
  }
  \label{fig:toy_multimodular_model}
\end{figure}

This generic setting has two modules with data $Y=(Y_1,...,Y_n),\ Y\in \Y^n$ and $Z=(Z_1,...,Z_m),\ Z\in \Z^m$ and continuous parameters $\varphi\in \phispace$ and $\theta\in \thetaspace$ of dimension $p_\varphi$ and $p_\theta$ respectively.
The generative models for parameters and data are $p(Z\mid \varphi)\,p(\varphi)$ and $p(Y\mid \varphi,\theta)\,p(\varphi,\theta)$.
The Bayesian posterior for this model can be written
\begin{align}
  p(\varphi,\theta \mid Z, Y) & = p(\varphi \mid Z, Y) \; p(\theta\mid Y, \varphi ) \label{eq:bayespost} \\
                              & \propto p(\varphi,\theta, Z, Y) \label{eq:bayespost2}
\end{align}
where the last line is the natural form for further computation using the joint distribution
\begin{align}\label{eq:joint_toymodel}
  p(\varphi,\theta, Z, Y) = p(\varphi,\theta)\; p(Z\mid \varphi)\; p(Y\mid \varphi,\theta)
\end{align}
\Cref{eq:bayespost} is given for contrast with the Cut model and SMI below.
Note that,
\begin{equation}\label{eq:phi_post}
  p(\varphi \mid Z, Y)\propto p(\varphi) \; p(Z\mid \varphi) \; p(Y\mid \varphi)
\end{equation}
and
\[
  p(\theta\mid Y, \varphi )\propto p(\theta\mid \varphi)\; \frac{p(Y\mid \varphi,\theta)}{p(Y\mid \varphi)}
\]
with
\[
  p(Y\mid \varphi)=\int p(Y\mid \varphi,\theta)\; p(\theta\mid \varphi) \; d\theta.
\]
In \cref{eq:bayespost} the value of $\varphi$ informs $\theta$.
In \cref{eq:phi_post} the marginal likelihood $p(Y\mid \varphi)$ can be thought of as ``feedback'' of information from the $Y$ module into the $Z$-module \citep{Liu2009modularization, Plummer2015cut, Jacob2017together}.
Any remaining normalising constants depend only on the data $Y,Z$.

\subsection*{Cutting Feedback}

Several different methods have been proposed to bring the generative models together in a joint distribution for the parameters given data.
Besides Bayesian inference itself, these include Markov Melding \citep{Goudie2019melding} (which focuses on settings where priors conflict across shared parameters) and Multiple Imputation \citep{Meng1994imputation}, which discusses inference for ``uncongenial'' modules, relevant here.
\cite{Nicholson2021covid} discusses the broader concept of ``interoperability'' of models in multi-modular settings.
In this paper we focus on \acrfull*{smi} defined in \cite{Carmona2020smi} and Cut model inference \citep{Plummer2015cut}, which is a special case.

Cut-model inference has proven useful in many settings, including complex epidemic models for the Covid pandemic \citep{Teh2021covid,Nicholson2021covid} and modular models linking isotope analysis and fertiliser use in Archaeological settings \citep{Styring2017extensification}, pharmaco-kinetic and -dynamic models \citep{Lunn2009pkpd} in pharmacological analysis, and health affects and air pollution \citep{Blangiardo11}.

Suppose the generative model $p(Y\mid \varphi,\theta)\; p(\varphi,\theta)$ in the $Y$-module is misspecified via $p(\theta\mid \varphi)$ or $p(Y\mid \varphi,\theta)$.
We hope to get a more reliable estimate of $\varphi$ by ``cutting'' the feedback from this module into the $\varphi$-estimation.
This is indicated by the dashed red line in \cref{fig:toy_multimodular_model}. Operationally, we drop the factor $p(Y\mid \varphi)$. Following \cite{Plummer2015cut},
\begin{align} \label{eq:cut_posterior}
  p_{\cut}(\varphi,\theta \mid Y,Z) & = p(\varphi \mid Z) \; p(\theta \mid Y,\varphi)
  \\ &
  \propto \frac{p(\varphi, \theta, Y, Z)}{p(Y \mid \varphi) }.  \nonumber
\end{align}
Cutting feedback leaves the Cut posterior with the intractable factor $p(Y \mid \varphi)$.
Inference with a Cut-posterior is a two-stage operation which can be seen as Bayesian Multiple Imputation. In the first stage we \emph{impute} $\varphi\sim p(\cdot \mid Z)$. This distribution of imputed $\varphi$ values is passed to the second \emph{analysis} stage where $\varphi$ are treated as randomly variable ``imputed data'' alongside $Y$, informing $\theta\sim p(\cdot \mid Y, \varphi)$. Looking ahead to \acrshort*{smi}, this setup is shown graphically in \cref{fig:toy_multimodular_model_2stg}, where $\varphi$ is imputed on the left (appearing in a grey circle as a parameter) and then conditioned on the right (appearing in a white square like $Y$). In a Cut-posterior $\eta=0$, and the $\tilde\theta,Y$ elements of the graph on the left are absent.

The Cut model posterior is a ``belief update'', in the sense of \cite{Bissiri2016}.
It is a rule $\psi$ for updating a prior measure of belief, $p_0(\varphi,\theta)$ say, using a loss $l(\varphi,\theta;Y,Z)$ connecting data and parameter (the -ve log-likelihood is a cannonical loss) to determine a posterior belief measure $p_1$ say.
They write $p_1=\psi(l,p_0)$.
\cite{Bissiri2016} require belief updates $\psi(l,p_0)$ to be coherent: in our notation, if the data are all conditionally independent given the parameters, and $Y=(Y^{(1)},Y^{(2)})$ and $Z=(Z^{(1)},Z^{(2)})$ are arbitrary partitions of the data in each module into two sets, then we should arrive at the same posterior $\psi(l(\varphi,\theta;Y,Z),p_0)$ if we take all the data $(Y,Z)$ together or if we update the prior to an intermediate posterior using $(Y^{(1)},Z^{(1)})$ and then update that intermediate posterior using the rest of the data, $(Y^{(2)},Z^{(2)})$, that is,
\begin{equation}\label{eq:coherent_bissiri}
  \psi(l(\varphi,\theta;Y,Z),p_0)=\psi(l(\varphi,\theta;Y^{(2)},Z^{(2)}),\psi(l(\varphi,\theta;Y^{(1)},Z^{(1)}),p_0)).
\end{equation}
They show with some generality that a valid belief update must be a Gibbs posterior if it is to be coherent, that is,
\[
  \psi(l(\varphi,\theta;Y,Z),p_0)\propto \exp(-l(\varphi,\theta;Y,Z))\,p_0(\varphi,\theta).
\]
Bayesian inference is coherent because the corresponding loss $l_{\text{bayes}}=-\log(p(Y\mid \varphi,\theta))-\log(p(Z\mid \varphi))$ is additive for independent data.
\cite{Carmona2020smi} show that the belief update determined by the Cut-model posterior is coherent and \cite{Nicholls2022smi} show it is valid.
This is surprising, as the loss $l_{\cut}=l_{\text{bayes}}+\log(p(Y\mid \varphi))$ is not simply additive.
This holds because the ``prior'' appearing in the marginal $p(Y^{(2)}\mid Y^{(1)},\varphi)$ in the second belief update is the posterior from the first stage and not $p_0(\theta\mid \varphi)$.

The Cut posterior can also be characterised {\it via} a constrained optimisation \citep{Yu2021variationalcut}.
Consider the class of all joint densities,
\[
  \mathcal{F}_{\cut} = \{q(\varphi,\theta) : q(\varphi) = p(\varphi\mid Z)\},
\]
for which the $\varphi$-marginal $q(\varphi)$ equals $p(\varphi\mid Z)$.
Densities in $\mathcal{F}_{\cut}$ are candidate Cut posteriors.
\cite{Yu2021variationalcut} show that, among densities in $\mathcal{F}$, the Cut posterior in \cref{eq:cut_posterior} is the best approximation to the Bayes posterior as measured by \acrshort*{kl} divergence, that is,
\begin{equation}\label{eq:cut_constrained_optim}
  p_{\cut}(\varphi,\theta \mid Y,Z)=\arg\min_{q\in \mathcal{F}_{\cut}} \kl{q(\varphi,\theta)}{ p(\varphi,\theta\mid Y,Z)}.
\end{equation}
They use this characterisation to motivate a framework for variational approximation of the Cut posterior.
Our motivation for variational \acrshort*{smi} starts from an equivalent characterisation of \acrshort*{smi}.

Statistical inference for the Cut posterior is challenging due to the marginal likelihood factor $p(Y\mid \varphi)$.
Several approaches have been suggested.
\cite{Plummer2015cut} gives a nested \acrshort*{mcmc} scheme: run \acrshort*{mcmc} targeting $p(\varphi\mid Z)$; for each sampled $\varphi$ a separate \acrshort*{mcmc} run targets $p(\theta\mid Y,\varphi)$; this yields $(\varphi,\theta)\sim p_{\cut}$, at least approximately.
Nested \acrshort*{mcmc} for Cut models suffers from double asymptotics but is adequate in some cases \citep{Styring2017extensification,Teh2021covid,Moss2022}.
A recent nested \acrshort*{mcmc} variant \citep{Liu2020sacut} shows efficiency gains for high dimensional targets.
An exact unbiased variant of \acrshort*{mcmc} based on coalescing coupled chains \citep{Jacob2020couplings} removes the double asymptotics of the nested sampler.



\subsection{Semi-Modular Inference}

\acrfull*{smi} \citep{Carmona2020smi} is a modification of Bayesian multi-modular inference which allows the user to adjust the flow of information between data and parameters in separate modules.
Cut models stop misspecification in one module from causing bias in others.
However, this often leads to variance inflation.
Semi-modular posteriors determine a family of candidate posterior distributions indexed by an influence parameter $\eta\in [0,1]$.
They interpolate between the Cut and Bayesian posteriors, expanding the space of candidate distributions and including Bayesian inference and Cut-model inference as special cases.

\subsubsection{Modulating feedback from data modules}

Cut models and \acrshort*{smi} are typically presented using models like \cref{fig:toy_multimodular_model} and cutting or modulating feedback from the $Y,\varphi,\theta$ module into the $Z,\varphi$ module. However, some effective applications of Cut models and \acrshort*{smi} cut or modulate feedback from modules that have no data \citep{Jacob2017together,Styring2017extensification,Carmona2020smi,Yu2021variationalcut,Styring2022urban}. We return to this below and in \cref{sec:prior_feedback}.

The \textbf{\acrshort*{smi} posterior} for the cut in \cref{fig:toy_multimodular_model} is defined as
\begin{equation} \label{eqn:smi_01}
  p_{\smi,\eta}(\varphi,\theta,\tilde\theta\mid Y,Z) = p_{\pow, \eta}(\varphi,\tilde\theta \mid Y,Z) p(\theta \mid \varphi,Y)
\end{equation}
where $p_{\pow, \eta}( \varphi , \tilde\theta \mid Z, Y )$ is the power posterior
\begin{align}
  p_{\pow, \eta}( \varphi , \tilde\theta \mid Z, Y ) & \propto p_{\pow, \eta}( \varphi , \tilde\theta, Z, Y )
  \intertext{with}
  p_{\pow, \eta}( \varphi , \tilde\theta, Z, Y )     & = p(Z\mid \varphi) p( Y \mid \varphi, \tilde \theta )^\eta \;  p(\varphi,\tilde\theta). \label{eq:powjoint_toymodel}
\end{align}
Taking $\eta=1$ in the $\eta$-smi posterior and integrating over $\tilde\theta$  gives the conventional posterior in \cref{eq:bayespost} so that $p_{smi,1}(\varphi,\theta\mid Y,Z)=p(\varphi,\theta\mid Y,Z)$, while $\eta=0$ gives the Cut posterior in \cref{eq:cut_posterior}, with $p_{smi,0}(\varphi,\theta\mid Y,Z)=p_{\cut}(\varphi,\theta\mid Y,Z)$.

The \acrshort*{smi}-posterior in \cref{eqn:smi_01} is motivated in a similar way to the Cut-posterior.
The extra degree of freedom $\eta$ in the power posterior $p_{\pow, \eta}$ down-weights the feedback from the $Y$-module on $\varphi$.
It is chosen to give the best possible imputation of $\varphi$ in the first phase of the inference.
The parameters $\tilde\theta$ can be thought of as auxiliary parameters introduced for the purpose of imputing $\varphi$.
This two-stage process is represented in \cref{fig:toy_multimodular_model_2stg}.
As for the Cut-posterior, $\varphi$-values from the imputation stage are treated as ``imputed data'' in the second stage, so they appear as random variables (in a grey circle) on the left, and as conditioned data (in a white square) on the right.
\begin{figure}[!htb]
  \centering
  \small
  \def\svgwidth{0.5\textwidth}
  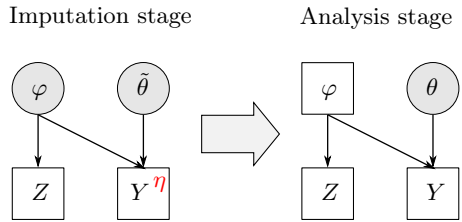
  \caption{
    Graphical representation of the implicit two-stage inference process in \acrlong*{smi}.
    Grey circles denote unknown quantities to be infered, and white boxes are fixed quantities.
  }
  \label{fig:toy_multimodular_model_2stg}
\end{figure}

In sample-based inference for \acrshort*{smi}, variants of nested \acrshort*{mcmc} \citep{Plummer2015cut} which target $p_{\pow, \eta}(\varphi,\tilde\theta \mid Y,Z)$ and then sample $p(\theta \mid Y,\varphi)$ for each sampled $\varphi$ have the same strengths and weaknesses as they do for the Cut posterior. Efficiency considerations are discussed in \cite{Carmona2020smi}.

\acrshort*{smi} can be characterised in the same way as the Cut model in \cref{eq:cut_constrained_optim}.
Consider the class of joint densities,
\[
  \mathcal{F}_{\smi,\eta} = \{q(\varphi,\theta,\tilde\theta) : q(\varphi,\tilde\theta) = p_{\pow, \eta}(\varphi,\tilde\theta \mid Z, Y)\},
\]
in which the marginal $q(\varphi,\tilde\theta)$ equals the power posterior $p_{\pow, \eta}(\varphi,\tilde\theta \mid Z, Y)$.
Densities in $\mathcal{F}_{\smi,\eta}$ are candidate \acrshort*{smi} posteriors.
At $\eta=1$ this is a duplicated Bayes posterior,
\[
  p_{\smi,\eta=1}(\varphi,\theta,\tilde\theta \mid Y,Z)=p(\varphi,\tilde\theta \mid Y,Z) p(\theta \mid Y,\varphi),
\]
in which both $p(\varphi,\tilde\theta \mid Y,Z)$ and $\int p_{\smi,1}(\varphi,\theta,\tilde\theta \mid Y,Z)\,d\tilde\theta$ equal $p(\varphi,\theta \mid Y,Z)$ in \cref{eq:bayespost}.
\begin{proposition}
  The \acrshort*{smi} posterior in \cref{eqn:smi_01} minimises the following \acrshort*{kl}-divergence over distributions in $\mathcal{F}_{\smi,\eta}$,
  \begin{equation}\label{eqn:smi_as_F_minimiser}
    p_{\smi,\eta}(\varphi,\theta,\tilde\theta \mid Y,Z)=\arg\min_{q\in \mathcal{F}_{\smi,\eta}} \kl{q(\varphi,\theta,\tilde\theta)}{ p(\varphi,\theta,\tilde\theta\mid Y,Z)}.
  \end{equation}
\end{proposition}
\begin{proof}
  The following is similar to the proof of the corresponding result for the Cut model in \cite{Yu2021variationalcut}.
  For $q\in \mathcal{F}_{\smi,\eta}$, we have
  \[
    q(\varphi,\theta,\tilde\theta)=p_{\pow, \eta}(\varphi,\tilde\theta \mid Y,Z)q(\theta \mid \varphi,\tilde\theta),
  \]
  so it is sufficient to show that the \acrshort*{kl} divergence to the posterior is minimised by $q(\theta  \mid \varphi,\tilde\theta)=p(\theta \mid Y,\varphi)$ (as that gives $q=p_{\smi,\eta}$).
  We have,
  \begin{align*}
    \kl{q(\varphi,\tilde\theta)\,q(\theta \mid \varphi,\tilde\theta)}{ p(\varphi,\theta,\tilde\theta \mid Y,Z)} & = \kl{q(\varphi,\tilde\theta)}{ p_{\pow, \eta}(\varphi,\tilde\theta \mid Y,Z)}                                \\
                                                                                                                & \qquad +\quad E_{q(\varphi,\tilde\theta)}[\kl{q(\theta \mid \varphi,\tilde\theta)}{p(\theta \mid Y,\varphi)}] \\
                                                                                                                & =E_{p_{\pow, \eta} }[\kl{q(\theta \mid \varphi,\tilde\theta)}{p(\theta \mid Y,\varphi)}],
  \end{align*}
  and the argument of the expectation is non-negative and zero when $q(\theta \mid \varphi,\tilde\theta)=p(\theta \mid Y,\varphi), \varphi\in\Omega_\varphi, \theta, \tilde\theta\in \Omega_{\theta}$, so $q=p_{\smi,\eta}$ minimises the original target.
\end{proof}

\subsubsection*{Modulating prior feedback}\label{subsec:mod_prioir_feed_paper}

If a Cut is applied to a \textit{prior density} $p(\theta|\varphi)$, as in \cite{Liu2009modularization,Jacob2017together,Styring2017extensification} and we simply remove the prior factor at the imputation stage then all that remains in the imputation posterior distribution is the base measure. A detailed example is given in \cref{subsec:exp_rnd_eff}.
The ``imputation prior'' has been replaced with a constant, and this may be inappropriate in some settings.
However, we are free to choose the imputation prior and we should use this freedom, as \cite{Moss2022} illustrate. Here we outline how this is done in \acrshort*{smi}. See \cref{sec:prior_feedback} for detail.

Consider the generative model $\varphi \sim p(\cdot),\ \theta  \sim p(\cdot \mid \varphi)$ and $Y_i \sim p(\cdot \mid \varphi,\theta),\ i=1,...,n.$
This model is shown in the leftmost graph in \cref{fig:cut-prior} in \cref{sec:prior_feedback}. The posterior is
\[
  p(\varphi,\theta \mid Y)\propto p(Y \mid \varphi,\theta)p(\varphi)p(\theta \mid \varphi).
\]
The \acrshort*{smi}-posterior is
\begin{equation}\label{eq:smi_posterior_CUTPRIOR}
  p_{\smi,\eta}(\varphi,\theta,\tilde\theta \mid Y)=p_{\pow,\eta}(\varphi,\tilde\theta \mid Y)p(\theta \mid Y,\varphi),
\end{equation}
where now
\[
  p_{\pow,\eta}(\varphi,\tilde\theta \mid Y)\propto p(\varphi)p_\eta(\tilde\theta \mid \varphi)p(Y \mid \varphi,\tilde\theta).
\]
The imputation prior $p_\eta(\tilde\theta \mid \varphi)$ must satisfy $p_{\eta=1}(\tilde\theta \mid \varphi)=p(\tilde\theta \mid \varphi)$. Like the Bayes prior $p(\theta \mid \varphi)$, the ``Cut prior'', $p_{\eta=0}(\tilde\theta \mid \varphi)=\tilde p(\tilde\theta)$ say, is a modelling choice.
Typically $p(\theta \mid \varphi)$ is a Subjective Bayes prior elicited from physical considerations, but is misspecified, and $\tilde p(\tilde\theta)$ is a non-informative Objective Bayes prior.

This \acrshort*{smi}-posterior belief update which cuts feedback in a prior is order coherent in the sense of \cite{Bissiri2016} and \cite{Nicholls2022smi}.
\begin{proposition}\label{prop:smi_cut_prior_is_OK}
  The \acrshort*{smi} posterior in \cref{eq:smi_posterior_CUTPRIOR} with cut prior feedback is an order coherent belief update.
\end{proposition}
\begin{proof}
  See \cref{proof:cut_prior_is_OK}.
\end{proof}

Taking a \emph{normalised} family $p_\eta(\tilde\theta \mid \varphi),\ \eta\in [0,1]$ of interpolating priors ensures that the marginal prior for $p(\varphi)$ in the imputation doesn't depend on $\eta$. An un-normalised family such as $p_\eta(\varphi,\tilde\theta)\propto \tilde p(\tilde\theta)^{1-\eta} p(\tilde\theta \mid \varphi)^\eta$ has all the desired interpolating properties, but the marginal $p(\varphi)$ in the imputation stage will then depend on $\eta$.
In some settings (for example when working with normal priors with fixed variance) the two prior parameterisations may be equivalent as $\eta$ scales the variance.

\section{Variational Modular Inference}\label{sec:vi_modular}

We define a variational approximation for modular posteriors based on the reparametrisation approach.
Our strategy has an end-to-end training implementation which avoids two-stage procedures, but converges to the same solution.

\subsection[Variational Inference and Normalizing Flows]{Variational Inference and Normalizing Flows} \label{sec:vi_nf}

Applications of Variational Inference
\citep{Jordan1999variational, Wainwright2008, Blei2017variational} were initially focused on \acrfull*{mfvi}.
This class of variational approximations is competitive with \acrshort*{mcmc} for prediction \citep{Wang2019VariationalMisspecification} but has disadvantages for uncertainty quantification in well specified models, making it less appealing for Bayesian inference for problems with small data sets where \acrshort*{mcmc} is feasible and well calibrated uncertainty measures are important.

Advances in variational methods have been motivated by its use in generative models in the Machine Learning literature and in particular in the context of \acrfullpl*{vae} \cite{Kingma2013vae,Kingma2019vaeintro} and applications in machine vision.
Variational families based on \acrfullpl*{nf} \citep{Rezende2015nf, Papamakarios2021normalizing, Kobyzev2020normalizing} developed in that context offer generative models which are much more expressive than \acrshort*{mfvi} and give better calibrated measures of uncertainty.
Adoption of \acrshortpl*{nf} in applications of statistical modelling and inference, where \acrshort*{mcmc} and \acrshort*{mfvi} are the de-facto approaches, has been more limited.
\acrfull*{svi} \citep{Hoffman2013svi} and \acrfull*{bbvi} \citep{Ranganath2014bbvi} offer efficient procedures to fit variational families which apply directly to \acrshort*{nf} parameterisations.
Recent advances include new methods for evaluating convergence and adequacy of variational approximation \citep{Yao2018yesbut,Xing19,Agrawal2020,Dhaka2020robust}.

\subsection{Variational Bayes in multi-modular models}\label{subsec:vi_multi}

We begin by giving a standard variational approximation to the Bayes posterior for the multi-modular model.
For concreteness, we use the multi-modular model in \cref{fig:toy_multimodular_model}.
Having established our methods on this class of models, extensions to other dependence structures are straightforward, as we illustrate in \cref{subsec:exp_rnd_eff}.

We take a parametrisation of the variational posterior in terms of a product
\begin{equation}\label{eq:q_modular_product}
  q_{\beta}(\varphi,\theta) = q_{\beta_1}(\varphi)q_{\beta_2}(\theta \mid \varphi),
\end{equation}
with each factor using a disjoint subset $\beta_1\in \Lambda_1$, $\beta_2\in \Lambda_2$ of a set of variational parameters $\beta=(\beta_1,\beta_2)$, with $\beta\in B$ and $B=\Lambda_1\times\Lambda_2$. Here $\Lambda_1=\Re^{L_1}$, $\Lambda_2=\Re^{L_2}$ and $B=\Re^{L_1+L_2}$ are typically high dimensional real spaces of variational parameters.

Our notation implies a \emph{flow}-based approach but captures a number of other parameterisations.
Let $\epsilon = (\epsilon_1, \epsilon_2)$ be a vector of continuous random variables distributed according to a \emph{base} distribution $p(\epsilon)=p(\epsilon_1, \epsilon_2)$, with $\epsilon_1\in \Re^{p_\varphi}$ and $\epsilon_2\in \Re^{p_\theta}$. We can for example take $p(\epsilon_1, \epsilon_2)$ to be the $\varphi,\theta$-prior.
Consider a diffeomorphism, $T: \Re^{p_\varphi+p_\theta}\to \Re^{p_\varphi+p_\theta}$ defined by concatenating the two diffeomorphisms expressing $\varphi$ and $\theta$, so that
\begin{align}
  \varphi_{(\beta_1,\epsilon)} & =T_1(\epsilon_1;\beta_1)\nonumber                                                                        \\
  \theta_{(\beta_2,\epsilon)}  & =T_2(\epsilon_2; \beta_2, \epsilon_1)\nonumber                                                           \\
  T(\epsilon; \beta)           & = \left(T_1(\epsilon_1;\beta_1), T_2(\epsilon_2;\beta_2, \epsilon_1). \right),\label{eq:bayes_transform}
\end{align}
For flow-based densities, $T_1: \Re^{p_\varphi}\to \Re^{p_\varphi}$ and $T_2: \Re^{p_\theta}\to \Re^{p_\theta}$ have properties listed in \cref{subsec:smi_flow}  \cite[see][Sec.~3]{Kobyzev2020normalizing} which allow us to sample, differentiate and evaluate the densities $q_\beta(\varphi,\theta)$ defined below. However, other familiar variational families such as \acrshort*{mfvi} can be expressed using \cref{eq:bayes_transform}.
Note that $T_2$ is a \emph{conditional transformation} that depends on $\epsilon_1$, so it can express correlation between $\varphi_{(\beta_1,\epsilon)}$ and $\theta_{(\beta_2,\epsilon)}$ (see \cref{sec:nf}).
In a normalising flow, $T_1$ and $T_2$ are compositions of diffeomorphisms, each with their own parameters. This increases the flexibility of the transformation.

The Jacobian matrix, $J_{T}=\partial T/\partial\epsilon$ is block lower triangular, so its determinant
is a product of determinants of $J_{T_1} = \partial T_1/\partial \epsilon_1$ and $J_{T_2} = \partial T_2/\partial \epsilon_2$,
\begin{equation*}
  \left\vert J_T \right\vert = \left\vert J_{T_1} \right\vert \left\vert J_{T_2} \right\vert
\end{equation*}
with no cross dependence on $\beta_1,\beta_2$, so that $\nabla_{\beta_2} \log \left\vert J_{T_1} \right\vert =0$ and $\nabla_{\beta_1} \log \left\vert J_{T_2} \right\vert =0$.
The joint variational distribution produced by the flow is then
\begin{align*}
  q_{\beta}(\varphi, \theta) & = p(\epsilon_1, \epsilon_2) \left\vert J_T \right\vert ^{-1} \\
                             & =q_{\beta_1}(\varphi)q_{\beta_2}(\theta\mid \varphi),
\end{align*}
where
\begin{align}
  q_{\beta_1}(\varphi)            & = p(\epsilon_1) \left\vert J_{T_1} \right\vert ^{-1}, \label{eqn:q_lambda1_from_p_jacob}              \\
  q_{\beta_2}(\theta\mid \varphi) & =p(\epsilon_2\mid \epsilon_1) \left\vert J_{T_2} \right\vert ^{-1}.\label{eqn:q_lambda2_from_p_jacob}
\end{align}
We need to be able to evaluate the determinants of the Jacobians $J_{T_1}$ and $J_{T_2}$. This works for a \acrshort*{nf} because the matrices are lower trianglular. However, other simpler designs such as \acrshort*{mfvi} also admit straightforward evaluation.

The optimal variational parameters minimise the \acrshort*{kl} divergence to the posterior, but will not in general be unique.
Let
\begin{align}
  D^*_B & =\min_{\beta\in B} \kl{ q_{\beta}(\varphi, \theta) }{ p(\varphi, \theta \mid Z, Y) },\label{eqn:variational_bayes_standard_Dstar} \\
  \intertext{and}
  B^*   & = \{\beta\in B:  \kl{ q_{\beta}(\varphi, \theta) }{ p(\varphi, \theta \mid Z, Y) }=D^*_B\},\label{eqn:variational_bayes_standard}
\end{align}
and let $\beta^*=(\beta^*_1,\beta^*_2)\in B^*$ be a generic set of parameter values minimising the \acrshort*{kl} divergence.
The definition in \cref{eqn:variational_bayes_standard} is equivalent to maximising the \acrshort*{elbo},
\begin{equation}\label{eq:elbo_modular}
  \elbo_{\bayes} = \E_{(\varphi, \theta) \sim q_{\beta}(\varphi, \theta)}[ \log p(\varphi, \theta, Z, Y) - \log q_{\beta}(\varphi, \theta) ].
\end{equation}
Using the \emph{reparametrisation trick} and expanding the joint distribution
\begin{align*}
  \elbo_{\bayes} = \E_{\epsilon \sim p(\epsilon)}[ & \log p(\varphi_{(\beta_1,\epsilon)}, \theta_{(\beta_2,\epsilon)}, Z, Y) - \log q_{\beta}(\varphi_{(\beta_1,\epsilon)}, \theta_{(\beta_2,\epsilon)}) ]                                      \\
  = \E_{\epsilon \sim p(\epsilon)}[                & \log p(Z \mid \varphi_{(\beta_1,\epsilon)}) + \log p(Y \mid \varphi_{(\beta_1,\epsilon)}, \theta_{(\beta_2,\epsilon)}) + \log p(\varphi_{(\beta_1,\epsilon)}, \theta_{(\beta_2,\epsilon)}) \\
                                                   & - \log p(\epsilon) + \log \left\vert J_{T} \right\vert  ],
\end{align*}
and the gradients of the \acrshort*{elbo} with respect to the variational parameters $(\beta_1,\beta_2)$ are
\begin{align}
  \nabla_{\beta_1} \elbo_{\bayes} = \E_{\epsilon \sim p(\epsilon)}[ & \nabla_{\varphi} \left\{ \log p(Z \mid \varphi) + \log p(Y \mid \varphi, \theta) + \log p(\varphi) \right\} \nabla_{\beta_1} \{ \varphi \} \nonumber             \\
                                                                    & + \nabla_{\beta_1} \log \left\vert J_{T_1} \right\vert  ]  \label{eq:grad_elbo_l1_full},                                                                         \\
  \nabla_{\beta_2} \elbo_{\bayes} = \E_{\epsilon \sim p(\epsilon)}[ & \nabla_{\theta} \left\{ \log p(Y \mid \varphi, \theta) + \log p(\theta \mid \varphi) \right\} \nabla_{\beta_2} \{ \theta \}                            \nonumber \\
                                                                    & + \nabla_{\beta_2} \log \left\vert J_{T_2} \right\vert  ]. \label{eq:grad_elbo_l2_full}
\end{align}
These gradients are used in \acrlong*{svi} \citep{Hoffman2013svi} to obtain the optimal variational parameters $\beta^*\in B^*$ for approximation of the Bayes posterior.

\subsection{Variational SMI}\label{subsec:vsmi}

In this section we define our variational approximation to the \acrshort*{smi} posterior.
For this, we expand the variational distribution in \cref{eq:q_modular_product} to include the auxiliary parameter $\tilde\theta$.
Again, we parametrise the variational posterior as a product,
\begin{equation}\label{eq:q_modular_smi_product}
  q_{\lambda}(\varphi,\theta, \tilde\theta) = q_{\lambda_1}(\varphi) q_{\lambda_2}(\theta \mid \varphi) q_{\lambda_3}(\tilde\theta \mid \varphi)
\end{equation}
where each factor has its own parameters, $\lambda_1\in \Lambda_1$ and $\lambda_2,\lambda_3\in \Lambda_2$ where $\Lambda_1$ and $\Lambda_2$ are defined above. Let $\lambda=(\lambda_1,\lambda_2,\lambda_3)$ with $\lambda\in \Lambda$ and $\Lambda=\Lambda_1\times \Lambda_2\times \Lambda_2$ so that $\Lambda=\Re^{L_1+2L_2}$. The parameters of $q_{\lambda_2}(\theta \mid \varphi)$ and $q_{\lambda_3}(\tilde\theta \mid \varphi)$ both match the variational Bayes parameterisation so we write $(\lambda_1,\lambda_2)\in B$ and $(\lambda_1,\lambda_3)\in B$.
Let
\[
  \mathcal{Q}=\{q_\lambda(\varphi,\theta,\tilde\theta); \lambda\in \Lambda\}
\]
denote the class of densities in our variational family.

\subsubsection{The variational-SMI approximation and its properties}

The purpose of \acrshort*{smi} is to control the flow of information from the $Y$-module into the posterior distribution for $\varphi$.
This leads us to define three basic properties that a useful variational approximation of the \acrshort*{smi} posterior must possess:
\begin{description}
  \item[(P1)] (expresses Cut) at $\eta=0$ the optimal variational posterior $q_{\lambda_1^*}(\varphi)$ is completely independent of the generative model for $Y$,
    with \[\lambda_1^*\in\{\lambda_1\in\Lambda_1: \kl{q_{\lambda_1}(\varphi)}{p(\varphi|Z)}=d^*_{\cut}\}\] and $d^*_{\cut}=\min_{\lambda_1\in\Lambda_1}\kl{q_{\lambda_1}(\varphi)}{p(\varphi|Z)}$.
  \item[(P2)] (expresses Bayes) at $\eta=1$ the marginal variational \acrshort*{smi} posterior $q_{\lambda^*_1,\lambda_2^*}(\varphi,\theta)$ is equal to the variational Bayes posterior, $q_{\beta^*_1,\beta^*_2}(\varphi,\theta)$ in \cref{eqn:variational_bayes_standard}, so $(\lambda_1^*,\lambda^*_2)\in B^*$;
  \item[(P3)] (approximates \acrshort*{smi}) for $\eta\in [0,1]$, if $p_{\smi,\eta}\in\mathcal{Q}$ then the variational approximation is equal to the target \acrshort*{smi} posterior, so $q_{\lambda^*}=p_{\smi,\eta}$.
\end{description}

Properties (P1-2) require $q_{\lambda^*}$ to interpolate a variational approximation to the Cut-posterior (removing all feedback from $Y$ into the variational approximation to the distribution of $\varphi$) and our original variational approximation to the Bayes posterior.
We will see that a standard variational approximation to the \acrshort*{smi} posterior based on the \acrshort*{kl} divergence between $q_\lambda$ and $p_{\smi,\eta}$ cannot satisfy these properties. We give a variational procedure based on the loss $\mathcal{L}^{( \smi, \eta)}(\lambda)$ in \cref{eq:vsmi_loss}, and show in \cref{prop:stop-gradient-loss} that if we take a variational approximation to $p_{\smi,\eta}$ minimising this loss then our variational approximation satisfies properties (P1-3).

\subsubsection{Defining the variational family}
As in \cref{subsec:vi_multi}, our notation it set up for flow-based approximation of the \acrshort*{smi} posterior, but captures other variational families such as \acrshort*{mfvi}.
We expand the base distribution and diffeomorphism to accommodate the auxiliary $\tilde\theta$.
The distributions $p(\epsilon_1)$ and $p(\epsilon_2 \mid \epsilon_1)$ and the transformations $T_1$ and $T_2$ are unchanged from \cref{subsec:vi_multi}. 

Let $\epsilon=(\epsilon_1, \epsilon_2, \epsilon_3)$ be the vector of random variables for our base distribution, $\epsilon\sim p(\epsilon)$, with $\epsilon_1$ and $\epsilon_2$ as before and $\epsilon_3\sim \epsilon_2$ so that $\epsilon_3\in \Re^{p_\theta}$.
Consider the extended diffeomorphism $T$ defined by the transformations,
\begin{align}
  \varphi_{(\lambda_1,\epsilon)}      & =T_1(\epsilon_1;\lambda_1)\nonumber                                                                                                                     \\
  \theta_{(\lambda_2,\epsilon)}       & =T_2(\epsilon_2; \lambda_2, \epsilon_1)\nonumber                                                                                                        \\
  \tilde\theta_{(\lambda_3,\epsilon)} & =T_2(\epsilon_3; \lambda_3, \epsilon_1)\nonumber                                                                                                        \\
  T(\epsilon; \lambda)                & = \left( T_1(\epsilon_1;\lambda_1) ,\; T_2(\epsilon_2;\lambda_2, \epsilon_1) ,\; T_2(\epsilon_3;\lambda_3, \epsilon_1) \right).\label{eq:smi_transform}
\end{align}
See \cref{eq:smi_transform_details} in \cref{sec:nf} for further details of these maps in a generic \acrshort*{nf} setting.

The Jacobian of the transformation, $J_{T}=\partial T/\partial\epsilon$, is block lower triangular as before, so its determinant factorises $\left\vert J_{T} \right\vert  = \left\vert J_{T_1}(\epsilon_1) \right\vert \, \left\vert J_{T_2}(\epsilon_1,\epsilon_2) \right\vert \, \left\vert J_{T_2}(\epsilon_1,\epsilon_3) \right\vert$ where we draw attention to the different arguments in the Jacobian factors involving $T_2$ but omit the $\lambda$-dependence.
We give more details of the transformation $T$ in \cref{subsec:smi_flow}.

Our variational family approximating the \acrshort*{smi} posterior $p_{\smi,\eta}$ has variational parameters $\lambda=(\lambda_1,\lambda_2,\lambda_3)$ and conditional independence structure
\begin{align}\label{eqn:q-lambda-smi-var}
  q_{\lambda}(\varphi, \theta, \tilde\theta) & = p(\epsilon_1, \epsilon_2, \epsilon_3) \left\vert J_{T} \right\vert ^{-1}
  \\
                                             & =q_{\lambda_1}(\varphi)q_{\lambda_2}(\theta\mid \varphi)q_{\lambda_3}(\tilde\theta\mid \varphi),\nonumber
\end{align}
where $q_{\lambda_1}$ and $q_{\lambda_2}$ are given in \cref{eqn:q_lambda1_from_p_jacob,eqn:q_lambda2_from_p_jacob} (replacing $\beta\to \lambda$) and
\begin{equation}
  q_{\lambda_3}(\tilde\theta\mid \varphi)=p(\epsilon_3\mid \epsilon_1) \left\vert J_{T_2} \right\vert ^{-1}.
\end{equation}


\subsubsection{The standard variational loss does not satisfy Properties (P1-2)}
A naive application of variational approximation to \acrshort*{smi} would minimise the \acrshort*{kl} divergence to the \acrshort*{smi} posterior at $\lambda^*=(\lambda^*_1,\lambda^*_2,\lambda^*_3)$ where
\begin{equation}\label{eqn:naive_VI_KL_basic}
  \lambda^* = \argmin_{\lambda\in \Lambda} \kl{ q_{\lambda}(\varphi, \theta, \tilde\theta) }{ p_{ \smi, \eta}(\varphi, \theta, \tilde\theta \mid Z, Y) },
\end{equation}
(ignoring non-uniqueness for brevity).
However, this presents two problems, one of principle and one of practice.

The principle of \acrshort*{smi} is to control the flow of information from the $Y$-module into the posterior distribution for $\varphi$.
This is lost in this setup.
The \acrshort*{kl}-divergence in \cref{eqn:naive_VI_KL_basic} is
\begin{align}\label{eqn:naive_VI_KL_expand}
  \kl{ q_{\lambda} }{ p_{ \smi, \eta}}
   & =\kl{q_{\lambda_1,\lambda_3}(\varphi,\tilde\theta) }{p_{\pow, \eta}(\varphi,\tilde\theta \mid Y,Z)}\nonumber
  \\
   & \qquad+\quad E_{\varphi\sim q_{\lambda_1}}[\kl{q_{\lambda_2}(\theta\mid \varphi)}{p(\theta \mid Y,\varphi)}].
\end{align}
The first term allows controlled feedback from the $Y$-module, as the $Y$-dependence in the power posterior $p_{\pow, \eta}$ is controlled by $\eta$ and vanishes entirely when $\eta=0$.
However, the second term leaks information from the $Y$-module to inform $q_{\lambda_1}(\varphi)$, even in the case $\eta=0$, therefore violating property (P1).

This variational approximation will not in general satisfy property (P2) either. In order for property (P2) to be satisfied at $\eta=1$ we must have $(\lambda^*_1,\lambda^*_2)=(\beta^*_1,\beta^*_2)$ for some $(\beta^*_1,\beta^*_2)\in B^*$ whenever $(\lambda^*_1,\lambda^*_2,\lambda^*_3)$ satisfy \cref{eqn:naive_VI_KL_basic} for some $\lambda^*_3\in \Lambda_2$, that is, the marginal variational \acrshort*{smi} distribution for $(\varphi,\theta)$ must coincide with one of the variational Bayes solutions. The optimal $\lambda^*$ minimise \cref{eqn:naive_VI_KL_expand}, so they maximise the \acrshort*{elbo},
\begin{align}\label{eq:elbo_naive_smi}
  \elbo_{\smi \text{naive}} = & \E_{(\varphi, \tilde\theta) \sim q_{\lambda_1,\lambda_3}(\varphi, \tilde\theta)}[ \log p_{\pow, \eta}(\varphi, \tilde\theta, Z, Y) - \log q_{\lambda_1,\lambda_3}(\varphi, \tilde\theta) ] \nonumber \\
                              & + \E_{(\varphi, \theta) \sim q_{\lambda_1,\lambda_2}(\varphi, \theta)}[ \log p(\varphi, \theta, Y) - \log q_{\lambda_2}(\theta \mid \varphi) ]                                                       \\
                              & - \E_{\varphi \sim q_{\lambda_1}(\varphi)}[ \log p(Y, \varphi) ]. \nonumber
\end{align}
Since the $\lambda^*$ parameters solve $\nabla_\lambda \elbo_{\smi \text{naive}}=0$ and the $\beta^*$ parameters solve $\nabla_\beta \elbo_{\bayes}=0$, a necessary condition for (P2) is that
$\left.\nabla_\beta \elbo_{\bayes}\right\vert_{\beta=(\lambda^*_1,\lambda^*_2)}=0$
(ie \cref{eq:grad_elbo_l1_full,eq:grad_elbo_l2_full} at $\beta=(\lambda^*_1,\lambda^*_2)$) at $\eta=1$.
However, using the reparameterisation trick, the $\elbo_{\smi \text{naive}}$-gradients can be written
\begin{align}
  \nabla_{\lambda_1} \elbo_{\smi \text{naive}} = \E_{\epsilon \sim p(\epsilon)}[ & \nabla_{\varphi} \left\{ \log p(Z \mid \varphi) + \eta \log p(Y \mid \varphi, \tilde\theta) + \log p(\varphi) \right\} \nabla_{\lambda_1} \{ \varphi \} \nonumber                         \\
                                                                                 & + \nabla_{\lambda_1} \log \left\vert J_{T_1}(\epsilon_1) \right\vert \nonumber                                                                                                            \\
                                                                                 & + \nabla_{\varphi} \left\{ \log p(Y \mid \varphi, \theta) + \log p(Y\mid \varphi) \right\} \nabla_{\lambda_1} \{ \varphi \} ]  \label{eq:grad_elbo_l1_naive_smi},                         \\
  \nabla_{\lambda_2} \elbo_{\smi \text{naive}} = \E_{\epsilon \sim p(\epsilon)}[ & \nabla_{\theta} \left\{ \log p(Y \mid \varphi, \theta) + \log p(\theta \mid \varphi) \right\} \nabla_{\lambda_2} \{ \theta \}                            \nonumber                        \\
                                                                                 & + \nabla_{\lambda_2} \log \left\vert J_{T_2}(\epsilon_1,\epsilon_2) \right\vert  ]. \label{eq:grad_elbo_l2_naive_smi}                                                                     \\
  \nabla_{\lambda_3} \elbo_{\smi \text{naive}} = \E_{\epsilon \sim p(\epsilon)}[ & \nabla_{\theta} \left\{ \eta \log p(Y \mid \varphi, \tilde\theta) + \log p(\tilde\theta \mid \varphi) \right\} \nabla_{\lambda_3} \{ \tilde\theta \}                            \nonumber \\
                                                                                 & + \nabla_{\lambda_3} \log \left\vert J_{T_2}(\epsilon_1,\epsilon_2) \right\vert  ]. \label{eq:grad_elbo_l3_naive_smi}
\end{align}
If these equations and $\left.\nabla_\beta \elbo_{\bayes}\right\vert_{\beta=(\lambda^*_1,\lambda^*_2)}=0$
all hold at at $\eta=1$ then
\[
  \E_{\epsilon \sim p(\epsilon)}\left[\nabla_{\varphi} \left\{ \log p(Y \mid \varphi, \tilde\theta) + \log p(Y\mid \varphi) \right\} \nabla_{\lambda_1} \{ \varphi \}\right]_{\lambda=\lambda^*}=0.
\]
Our variational framework has to satisfy (P2) for every target $p_{\smi,\eta}$ and every variational family $\mathcal{Q}$. However, if we target the loss in \cref{eqn:naive_VI_KL_basic} then $\lambda^*$ would have to satisfy an over-determined system of equations at $\eta=1$ and this will in general have no solutions.

We learn from this that the loss we seek for the variational \acrshort*{smi} approximation is not captured by the \acrshort*{kl}-divergence in \cref{eqn:naive_VI_KL_basic}.
However, there is a second practical problem with carrying out \acrlong*{svi} based on this naive variational loss.
In practice, in order to minimise \cref{eqn:naive_VI_KL_expand}, we maximise $\elbo_{\smi,\text{naive}}$ in \cref{eq:elbo_naive_smi}
using a Monte Carlo estimate of its gradients.
The last term in \cref{eq:elbo_naive_smi} involves the intractable $E_{\varphi\sim q_{\lambda_1}}[\log(p(Y, \varphi)]$, making the $\lambda_1$-variation unrealisable in practice.

\subsubsection{Loss for variational-SMI}

One way to characterise variational \acrshort*{smi} is by generalising the two-stage optimisation approach given by \cite{Yu2021variationalcut} for the Cut posterior. We will see that this approach satisfies properties (P1-3), and that the optimal variational parameters are given by minimising a customised variational loss.
Let
\begin{equation}\label{eqn:distance_to_set}
  d(q_{\lambda},\mathcal{F}_{\smi,\eta})=\min_{\tilde q\in \mathcal{F}_{\smi,\eta}}\kl{q_{\lambda}}{\tilde q}.
\end{equation}
define the divergence between the density $q_{\lambda}$ and the set of densities $\mathcal{F}_{\smi,\eta}$.

\begin{proposition}\label{prop:var_smi_distance}
  The divergence defined in \cref{eqn:distance_to_set} can be written
  \begin{equation}\label{eqn:var_smi_distance}
    d(q_{\lambda},\mathcal{F}_{\smi,\eta})=\kl{q_{\lambda_1, \lambda_3}(\varphi, \tilde\theta)}{p_{\pow, \eta}(\varphi,\tilde\theta \mid Y,Z)},
  \end{equation}
  and hence does not depend on $\lambda_2$.
\end{proposition}
\begin{proof}
  See \cref{proof:var_smi_distance}.
\end{proof}

We now define the optimal variational parameters.
These will minimise divergence from distributions in $\mathcal{F}_{\smi,\eta}$ and otherwise approximate \acrshort*{smi}.
First, exploiting \cref{prop:var_smi_distance}, $(\lambda_1,\lambda_3)$ minimise \cref{eqn:var_smi_distance}. Let
\begin{align}
  d^*_{\smi}        & =\min_{(\lambda_1,\lambda_3)\in B}\kl{q_{\lambda_1, \lambda_3}(\varphi, \tilde\theta)}{p_{\pow, \eta}(\varphi,\tilde\theta \mid Y,Z)}\nonumber \\
  \Lambda^*_{(1,3)} & = \{(\lambda_1,\lambda_3)\in B:
  \kl{q_{\lambda_1, \lambda_3}(\varphi, \tilde\theta)}{p_{\pow, \eta}(\varphi,\tilde\theta \mid Y,Z)}=d^*_{\smi}\}.
  \label{eqn:lambda13-star-defn}\end{align}
Secondly, $\lambda_2$ is chosen for best approximation of $p_{ \smi, \eta}$ at fixed $(\lambda^*_1,\lambda^*_3)\in \Lambda^*_{(1,3)}$. Let
\begin{align}
  D^*_{\smi}(\lambda^*_1)      & =\min_{\lambda_2\in \Lambda_2} E_{\varphi\sim q_{\lambda^*_1}}[\kl{q_{\lambda_2}(\theta\mid \varphi)}{p(\theta \mid Y,\varphi)}]\nonumber                                                    \\
  \Lambda^*_{(2)}(\lambda^*_1) & =\{\lambda_2\in \Lambda_2: E_{\varphi\sim q_{\lambda^*_1}}[\kl{q_{\lambda_2}(\theta\mid \varphi)}{p(\theta \mid Y,\varphi)}]  = D^*_{\smi}(\lambda^*_1)\}\label{eqn:lambda2-star-defn-equiv}
\end{align}
The following proposition shows that $\lambda^*_2\in \Lambda^*_{(2)}(\lambda^*_1)$ targets a good fit to $p_{ \smi, \eta}$.
\begin{proposition}\label{prop:var_smi_lambda2}
  The set $\Lambda^*_{(2)}(\lambda^*_1)$ defined in \cref{eqn:lambda2-star-defn-equiv} is equivalently
  \begin{align}
    \tilde D^*_{\smi}(\lambda^*_1) & =\min_{\lambda_2\in \Lambda_2} \kl{q_{(\lambda^*_1,\lambda_2,\lambda^*_3)}}{p_{ \smi, \eta}}\nonumber                                                       \\
    \Lambda^*_{(2)}(\lambda^*_1)   & =\{\lambda_2\in \Lambda_2:  \kl{q_{(\lambda^*_1,\lambda_2,\lambda^*_3)}}{p_{ \smi, \eta}} = \tilde D^*_{\smi}(\lambda^*_1)\}. \label{eqn:lambda2-star-defn}
  \end{align}
\end{proposition}
\begin{proof}
  Expand the \acrshort*{kl} divergence in \cref{eqn:lambda2-star-defn} using \cref{eqn:naive_VI_KL_expand} and substitute
  $(\lambda_1,\lambda_3)=(\lambda^*_1,\lambda^*_3)$.
  The first term does not depend on $\lambda_2$ and the second term gives \cref{eqn:lambda2-star-defn-equiv}.
\end{proof}

We now define variational \acrshort*{smi} and demonstrate (P1-3).
\begin{definition}\label{defn:var-smi}
  \emph{(Variational \acrshort*{smi})}
  A variational \acrshort*{smi} posterior density is a density $q_{\lambda^*}(\varphi,\theta,\tilde\theta)$ parameterised in  \cref{eq:q_modular_smi_product} with $\lambda^*\in\Lambda^*$ where
  \begin{equation}\label{eqn:smi-lambda-star}
    \Lambda^*=\bigcup_{(\lambda^*_1,\lambda^*_3)\in \Lambda^*_{(1,3)}}\left(\bigcup_{\lambda^*_2\in\Lambda^*_{(2)}(\lambda^*_1)}\{(\lambda^*_1,\lambda^*_2,\lambda^*_3)\}\right),
  \end{equation}
  and $\Lambda^*_{(1,3)}$ and $\Lambda^*_{(2)}(\lambda^*_1)$ are defined in \cref{eqn:lambda13-star-defn,eqn:lambda2-star-defn-equiv} respectively.
\end{definition}


\begin{remark}\label{remk:smi-defn-eta-depend}
  Our discussion in this section takes $\eta$ fixed. As we vary $\eta$ the target $p_{\smi,\eta}$ varies, so the set of optimal variational parameters $\Lambda^*$ depends on $\eta$. Below we write $\Lambda^*(\eta)$ when we need to emphasise this dependence.
\end{remark}

\begin{remark}\label{remk:var-smi-roots}
  The variational \acrshort*{smi} parameters $\lambda^*\in\Lambda^*$ are roots of the equations
  \begin{align}
    \nabla_{(\lambda_1,\lambda_3)}\kl{q_{\lambda_1,\lambda_3}}{p_{\pow, \eta}}                                       & =0 \label{eqn:lam-star-roots1a} \\
    \nabla_{\lambda_2} E_{\varphi\sim q_{\lambda_1}}\kl{q_{\lambda_2}(\theta\mid \varphi)}{p(\theta \mid Y,\varphi)} & =0\label{eqn:lam-star-roots1b}
  \end{align}
  with positive curvature.
  We have not substituted $\lambda_1=\lambda^*_1$ in \cref{eqn:lam-star-roots1b}. As a system, any $\lambda_1$ satisfying \cref{eqn:lam-star-roots1b} is required to be a root (with $\lambda^*_3$) of \cref{eqn:lam-star-roots1a} so the system imposes this condition. This will allow us to solve these equations as a single system using \acrshort*{sgd} on the loss $\mathcal{L}^{\smi,\eta}$ in \cref{prop:stop-gradient-loss} below, avoiding a two-stage procedure.
\end{remark}

\begin{remark}\label{remk:var-smi-roots-gradients}
  Consider the variational family defined in \cref{eqn:q-lambda-smi-var}.
  Using the reparametrisation trick and expanding terms, \cref{eqn:lam-star-roots1a,eqn:lam-star-roots1b} are
  \begin{align}
    0 = \E_{\epsilon \sim p(\epsilon)}[ & \nabla_{\varphi} \left\{ \log p(Z \mid \varphi) + \eta \log p(Y \mid \varphi, \tilde\theta) + \log p(\varphi) \right\} \nabla_{\lambda_1} \{ \varphi \} \nonumber
    \\
                                        & + \nabla_{\lambda_1} \log \left\vert J_{T_1}(\epsilon_1) \right\vert  ],  \label{eq:vsmi_roots_l1}                                                                   \\
    0 = \E_{\epsilon \sim p(\epsilon)}[ & \nabla_{\theta} \left\{ \log p(Y \mid \varphi, \theta) + \log p(\theta \mid \varphi) \right\} \nabla_{\lambda_2} \{ \theta \} \nonumber                              \\
                                        & + \nabla_{\lambda_2} \log \left\vert J_{T_2}(\epsilon_1,\epsilon_2) \right\vert  ] \label{eq:vsmi_roots_l2}                                                          \\
    0 = \E_{\epsilon \sim p(\epsilon)}[ & \nabla_{\tilde\theta} \left\{ \eta \log p(Y \mid \varphi, \tilde\theta) + \log p(\tilde\theta \mid \varphi) \right\} \nabla_{\lambda_3} \{ \tilde\theta \} \nonumber \\
                                        & + \nabla_{\lambda_3} \log \left\vert J_{T_2}(\epsilon_1,\epsilon_3) \right\vert  ] \label{eq:vsmi_roots_l3}
  \end{align}
  Notice that the extra terms in \cref{eq:grad_elbo_l1_naive_smi} are absent in \cref{eq:vsmi_roots_l1} so $\lambda^*$ will not be over-determined when we come to match variational Bayes at $\eta=1$.
\end{remark}

Consider now property (P1). If $\lambda^*\in\Lambda^*$ with $\lambda^*=(\lambda^*_1,\lambda^*_2,\lambda^*_3)$ are some generic fitted variational parameters, then $q_{\lambda^*_1}(\varphi)$ cannot depend in any way on $p(Y\mid \varphi,\theta)$ at $\eta=0$, as the power posterior in \cref{eqn:lambda13-star-defn} is
\[
  p_{\pow, \eta=0}(\varphi,\tilde\theta \mid Y,Z)=p(\varphi \mid Z)\,p(\tilde\theta\mid \varphi).
\]
The $Y$ observation observation model doesn't enter \cref{eqn:lambda13-star-defn} at $\eta=0$.
Under an additional assumption on the variational family, we can remove any $p(\tilde\theta \mid \varphi)$-dependence (so $q_{\lambda^*}$ is ``completely independent of the generative model'' at $\eta=0$).

\begin{proposition}\label{prop:var_smi_is_cut_at_eta0}
  Variational \acrshort*{smi} satisfies property (P1) at $\eta=0$: If the set
  \[
    \Lambda^*_{(3)}=\{\lambda_3\in \Lambda_2: q_{\lambda^*_3}(\tilde\theta\mid \varphi)=p(\tilde\theta\mid \varphi)\}
  \]
  is non-empty and we set
  \[
    \Lambda^*_{(1)} =   \{\lambda_1\in \Lambda_1:
    \kl{q_{\lambda_1}(\varphi)}{p(\varphi \mid Z)}=d^*_{\cut}\}
  \]
  with $d^*_{\cut}$ defined in (P1) then $\Lambda^*_{(1,3)}$ defined in \cref{eqn:lambda13-star-defn} satisfies
  \[
    \Lambda^*_{(1,3)} = \Lambda^*_{(1)} \times \Lambda^*_{(3)},
  \]
  so $q_{\lambda^*_1}$ does not depend in any way on $p(Y \mid \varphi,\theta)$ or $p(\theta \mid \varphi)$ at $\eta=0$.
\end{proposition}
\begin{proof}
  See \cref{proof:var_smi_is_cut_at_eta0}.
\end{proof}

The point here is that the auxiliary variable $\tilde\theta$ is present only through its prior in the power posterior at the Cut, $\eta=0$, but this factor is perfectly expressed by a corresponding factor in the variational approximation, and hence doesnt enter the $\lambda_1$-variation.
The condition that there is $\lambda^*_3\in \Lambda_2$ such that $q_{\lambda^*_3}(\tilde\theta\mid \varphi)=p(\tilde\theta\mid \varphi)$ is met by choosing $\epsilon_2,\epsilon_3\sim p(\cdot\mid \varphi)$, the prior distribution for $\theta$ and $\tilde\theta$.
We can then find $\lambda^*_3$ to give $T_2(\epsilon_3;\lambda^*_3,\epsilon_1)=(\epsilon_1,\epsilon_3)$ equal to the identity map (possible in a flow-parameterised map, but not in general in \acrshort*{mfvi}).
At this $\lambda_3$-value, $\tilde\theta=\epsilon_3$ and hence $q_{\lambda^*_3}(\tilde\theta\mid \varphi)=p(\tilde\theta\mid \varphi)$. If we have a cut prior as in \cref{sec:prior_feedback} then take $\epsilon_2,\epsilon_3\sim \tilde p(\cdot)$, the Cut prior.

The variational approximation to the Cut-posterior defined in \cref{prop:var_smi_is_cut_at_eta0} is similar to that given in \cite{Yu2021variationalcut}.
We focus on flow-based parameterisations of the variational density $q_\lambda$, but apart from this our methods coincide at $\eta=0$.

We consider now property (P2).
Taking $\eta=1$, the power posterior is the Bayes posterior, so \cref{eqn:variational_bayes_standard,eqn:lambda13-star-defn} are identical optimisation problems as the $\lambda$-dependence is the same.
However, this shows that $q_{\lambda^*_1,\lambda^*_3}(\varphi,\tilde\theta)$ is variational Bayes at $\eta=1$, and we have to check that $q_{\lambda^*_1,\lambda^*_2}(\varphi,\theta)$ is variational Bayes.

\begin{proposition} \label{prop:smi-show-p2}
  Variational \acrshort*{smi} satisfies property (P2). Let
  \[
    \Lambda^*_{(1)}=\bigcup_{(\lambda^*_1,\lambda^*_3)\in \Lambda^*_{(1,3)}} \{\lambda_1^*\}.
  \]
  The set of Bayes and \acrshort*{smi} variational posteriors for $\varphi,\theta$ are the same, that is,
  \[
    \bigcup_{\lambda_1^*\in\Lambda^*_{(1)}}\bigcup_{\lambda^*_2\in \Lambda_{(2)}^*(\lambda^*_1)} \{(\lambda^*_1,\lambda^*_2)\}=B^*,
  \]
  when $\eta=1$.
\end{proposition}
\begin{proof}
  See \cref{proof:smi-show-p2}.
\end{proof}

\begin{proposition} \label{prop:smi-show-p3}
  Variational \acrshort*{smi} satisfies property (P3).
  If $p_{ \smi, \eta}\in\mathcal{Q}$ then $q_{\lambda^*}=p_{ \smi, \eta}$ for $\lambda^*\in \Lambda^*$.
\end{proposition}
\begin{proof}\label{proof:smi-show-p3}
  This is usually immediate for standard variational methods but has to be checked here. If $p_{\smi, \eta}\in\mathcal{Q}$ then there exist $\lambda_1,\lambda_3$ such that  $q_{\lambda_1,\lambda_3}=p_{\pow, \eta}$ and $\lambda_2$ such that $q_{\lambda_2}(\theta\mid \varphi)=p(\theta \mid Y,\varphi)$ and since these choices minimise the \acrshort*{kl}-divergences in \cref{prop:var_smi_distance,prop:var_smi_lambda2} they are the optimal values, so $q_{\lambda^*}=p_{\smi, \eta}$.
\end{proof}

\subsubsection{The overall loss targeted by variational-SMI}


We have defined $\lambda^*$ in two steps, \cref{eqn:lambda13-star-defn,eqn:lambda2-star-defn-equiv} with two losses, $d(q_\lambda,\mathcal{F}_{\smi,\eta})$ and $\kl{q_{\lambda}}{p_{ \smi, \eta}}$.
We can bring this together into a single overall loss in two ways. The first is formal but useful for computation. The second is useful for understanding.

For computational purposes we define the loss $\mathcal{L}^{( \smi, \eta)}$ targeted by variational \acrshort*{smi} using the \texttt{stop\_gradient} operator $\cancel{\nabla}(\cdot)$ acting on $\varphi_{(\epsilon,\lambda_1)}$. 
The \texttt{stop\_gradient} operator protects the object it acts on from the gradient operator $\nabla_{\lambda}$.
Let
\begin{align} \label{eq:vsmi_loss}
  \mathcal{L}^{( \smi, \eta)}(\lambda) = \elbo_{\pow, \eta}(\lambda_1,\lambda_3) + \elbo_{\bayes \cancel{\nabla}(\varphi)}(\lambda_1,\lambda_2)
\end{align}
where
\begin{align}
  \elbo_{\pow, \eta}(\lambda_1,\lambda_3) = \E_{(\varphi,\tilde\theta)\sim q_{\lambda_1,\lambda_3}}[                & \log p_{\pow, \eta}(\varphi, \tilde\theta, Z, Y) - \log q_{\lambda_1,\lambda_3}(\varphi, \tilde\theta) ] \label{eq:elbo_modular_pow}                  \\
  \elbo_{\bayes \cancel{\nabla}(\varphi)}(\lambda_1,\lambda_2) = \E_{(\varphi,\theta)\sim q_{\lambda_1,\lambda_2}}[ & \log p( \cancel{\nabla}(\varphi), \theta, Z, Y) - \log q_{\lambda_1,\lambda_2}(\cancel{\nabla}(\varphi), \theta) ]. \label{eq:elbo_modular_stop_grad}
\end{align}
with the joint and powered joint distributions given as \cref{eq:joint_toymodel,eq:powjoint_toymodel}.
We are in effect defining the function and its derivative separately and so this loss is formal and cannot take the place of \cref{prop:smi-show-p3-extra} below in giving meaning to the variation.
However it is convenient for implementation, as the \texttt{stop\_gradient} operator is directly expressed in the automatic differentiation framework we use.

\begin{proposition}\label{prop:stop-gradient-loss}
  The set $\Lambda^*$ in \cref{defn:var-smi} is the set of solutions of $\nabla_{\lambda} \mathcal{L}^{( \smi, \eta)} = 0$ corresponding to minima.
\end{proposition}
\begin{proof}
  See \cref{proof:stop-gradient-loss}.
\end{proof}

An overall loss \emph{function} can be given as follows. Let $v\ge 0$ and
\begin{equation}\label{eqn:loss_weighted_var_smi}
  \mathcal{L}^{(v)}(\lambda)=d(q_\lambda,\mathcal{F}_{\smi,\eta})+
  v\cdot \kl{q_{\lambda}}{p_{ \smi, \eta}}
\end{equation}
denote a weighted loss which allows varying levels of priority to be put on proximity to $\mathcal{F}_{\smi,\eta}$ and approximation of $p_{ \smi, \eta}$.
\begin{proposition} \label{prop:smi-show-p3-extra}
  Let $\mathcal{L}^*(v)=\min_{\lambda\in\Lambda}\mathcal{L}^{(v)}(\lambda)$
  and
  \[
    \Lambda^*(v)=\{\lambda\in \Lambda: \mathcal{L}^{(v)}(\lambda)=\mathcal{L}^*(v)\}.
  \]
  Under regularity conditions on $\mathcal{F}_{smi,\eta}$ and $p_{ \smi, \eta}$ given in \cref{prop:use_IFT_show_loss_limit}, for every solution $\lambda^*\in \Lambda^*$ in \cref{defn:var-smi} and all sufficiently small $v\ge 0$ there exists a unique continuous function $\lambda^*(v)$ satisfying $\lambda^*(v)\in \Lambda^*(v)$ and \[\lim_{v\to 0}\lambda^*(v)=\lambda^*.\]
\end{proposition}
\begin{proof}
  See \cref{proof:smi-show-p3-extra}.
\end{proof}


The value of \cref{prop:smi-show-p3-extra} is that it allows us to interpret $q_{\lambda^*},\ \lambda^*\in\Lambda^*$ as minimising a proper loss function $\mathcal{L}^{(v)}$ at small $v$ (approximately). The minimum loss $\mathcal{L}^*(v)$ decreases as we expand the variational family $\mathcal{Q}$ and is zero when $p_{ \smi, \eta}\in\mathcal{Q}$, in which case $q_{\lambda^*}=p_{ \smi, \eta}$ for $\lambda^*\in \Lambda^*(v)$. In contrast, although $d^*_{\smi}$ decreases as $\mathcal{Q}$ expands, $D^*_{\smi}(\lambda_1^*)$ may increase, though must eventually go to zero when $\mathcal{Q}$ expands to include $p_{ \smi, \eta}$, by \cref{prop:smi-show-p3}. However, $\mathcal{L}^{(v)}$ is not a viable optimisation target at small $v$ because the second term in \cref{eqn:loss_weighted_var_smi} is intractable, as we saw in our discussion of \cref{eq:elbo_naive_smi}.

\subsubsection{Stochastic gradient descent for variational-SMI}
\Cref{alg:vsmi} gives our \acrlong*{sgd} method to target the \acrshort*{smi} posterior for a fixed value of the influence parameter, $\eta$.
The algorithm is based on the loss $\mathcal{L}^{( \smi, \eta)}$ in \cref{eq:vsmi_loss} and consists of a single training loop, using the \verb|stop_gradient| operator to avoid two-stage optimisation procedures. This is given for understanding. In \cref{sec:meta_posterior} and \cref{alg:v_meta_smi} we will train a ``meta-posterior'' approximating the whole family of \acrshort*{smi}-posteriors as a function of $\eta$.
 
\begin{algorithm}[tb]
  \caption{Variational Posterior approximation for $p_{ \smi, \eta}$} \label{alg:vsmi}
  \begin{algorithmic}
    \STATE \textbf{Input:} $\mathcal{D}$: Data. $p(\varphi,\theta,\mathcal{D})$: Multi-modular probabilistic model. $q_\lambda=(p(\epsilon), T, \lambda)$: variational family. A value of $\eta \in [0,1]$: Influence parameter(s) for suspect module(s) \\[0.1in]
    \STATE \textbf{Output:} Variational approximation $q_{\hat\lambda}(\varphi,\theta,\tilde\theta)$ of the $\eta$-\acrshort*{smi} posterior. \\[0.1in]

    \STATE Initialise variational parameters $\lambda$
    \WHILE{\acrshort*{sgd} not converged}
    \STATE (Optional) Sample a random minibatch of data $\mathcal{D}^{(b)} \sim \mathcal{D}$.
    \FOR{$s = 1,\ldots,S$}
    \STATE Sample the base distribution, $\epsilon_s \sim p(\epsilon)$.
    \STATE Transform the sampled values $(\varphi_s, \theta_s, \tilde\theta_s) \leftarrow T(\epsilon_s; \lambda)$ as in \cref{eq:smi_transform}.
    \ENDFOR

    \STATE Compute the Monte Carlo estimate of the loss $\mathcal{L}^{( \smi, \eta)}$ in \cref{eq:vsmi_loss} and its gradients.
    \begin{equation}
      \hat{\mathcal{L}}^{( \smi, \eta)} = \widehat{\elbo}_{\pow, \eta} + \widehat{\elbo}_{\cancel{\nabla}(\varphi)}
    \end{equation}
    where
    \begin{align}
      \widehat{\elbo}_{\pow, \eta}               & = - \frac{1}{S} \sum_{s=1}^{S} \left[ \log p_{\pow, \eta}(\varphi_{s}, \tilde\theta_{s}, \mathcal{D}^{(b)}) - \log q(\varphi_{s}, \tilde\theta_{s}) \right]           \\
      \widehat{\elbo}_{\cancel{\nabla}(\varphi)} & = - \frac{1}{S} \sum_{s=1}^{S} \left[ \log p(\cancel{\nabla}(\varphi_{s}), \theta_{s}, \mathcal{D}^{(b)}) - \log q( \cancel{\nabla}(\varphi_{s}), \theta_{s}) \right]
    \end{align}

    \STATE Update $\lambda$ using the estimated gradient vector $\nabla_{\lambda}\hat{\mathcal{L}}^{( \smi, \eta)}$
    \STATE Check convergence of $q_{\lambda}(\varphi_s, \theta_s, \tilde\theta_s)$
    \ENDWHILE

    \RETURN $\hat\lambda=\lambda$

  \end{algorithmic}
\end{algorithm}

\subsection{Selecting the SMI posterior} \label{subsec:best_smi}

We now give a utility for selection of the influence parameter $\eta$.
This will depend on the goals of inference. Recall that $Y\in \mathcal{Y}^m$ and $Z\in \mathcal{Z}^n$.
In the following we take a predictive loss based on the \acrshort*{smi}-predictive distribution for independent new data $(y',z')\in \mathcal{Y}\times\mathcal{Z}$,
\begin{equation*}
  p_{\smi,\eta}(y',z' \mid Y,Z)=\int p(y',z' \mid \varphi,\theta)p_{\smi,\eta}(\varphi,\theta \mid Y,Z)\,d\varphi d\theta,
\end{equation*}
and a utility which is equivalent to the negative of the \acrshort*{kl} divergence to the true generative model for the new data $p^*(y',z')$.
This utility is the \acrfull*{elpd} \citep{Vehtari2016},
\begin{equation*}
  U(\eta)=\int p^*(y',z')\log(p_{\smi,\eta}(y',z' \mid Y,Z))\,dy' dz'.
\end{equation*}
In our variational setting, these quantities are replaced by estimates based on our variational approximation $q_{\lambda^*}$ to $p_{\smi,\eta}$.
The variational parameters are $\lambda^*(\eta)\in \Lambda^*(\eta)$, per \cref{remk:smi-defn-eta-depend}.
We define the variational \acrshort*{smi} posterior predictive distribution
\begin{equation}\label{eq:post_predictive_var_smi}
  q_{\eta}(y',z')=\int p(y',z' \mid \varphi,\theta)q_{\lambda^*(\eta)}(\varphi,\theta)\,d\varphi d\theta,
\end{equation}
with corresponding utility
\begin{equation}\label{eq:ELPD_var_smi_def}
  u(\eta)=\int p^*(y',z')\log(q_{\eta}(y',z'))\,dy' dz'.
\end{equation}
We estimate $u$ in \cref{eq:ELPD_var_smi_def} using the WAIC \citep{Watanabe2012}, following \cite{Vehtari2016}. See \cref{subsec:best_smi_vmp} for further details.
When the variational \acrshort*{smi} posterior predictive distribution can be calculated in closed form, $u(\eta)$ may be estimated using leave one out cross validation.
This is asymptotically equivalent to the WAIC, but will in general be too computationally demanding to compute.
In order to complete the inference, we select the optimal influence parameter \begin{equation}\label{eqn:eta-star-defn}
  \eta^*=\arg\min_{\eta} u(\eta),
\end{equation}
and return the final selected variational \acrshort*{smi} posterior,
$q_{\lambda^*(\eta^*)}(\varphi,\theta)$, for further analysis.

A number of other procedures have been given for selecting the influence in a pure power-posterior setting. \cite{wu-martin-20,wu-martin-21} introduce a new method and summarise and compare a selection of methods, reflecting different priorities in the inference and corresponding utilities.
If our goal is parameter estimation, then a utility that directly targets parameter estimates, rather than predictive distributions, will be preferred.
In recent work, \cite{chakraborty22_smiLF} select $\eta$ as the most Bayes-like (i.e., the largest) value that does not show goodness-of-fit violation with the Cut.
\cite{Carmona2022spatial} use a utility tailored to their inference objectives.
They have data in which $\theta$ is a high dimensional vector, and some of the components of $\theta$ are directly measured.
They use the posterior mean square error for prediction of known $\theta$-components in a LOOCV framework to select $\eta$, linking $\eta$-selection to success in parameter estimation.

\section{The Variational Meta-Posterior}\label{sec:meta_posterior}

In order to select a posterior from the family of variational \acrshort*{smi} posteriors we need the fitted variational parameters $\lambda^*(\eta)$ as a function of $\eta$ in order to estimate the selection criterion $u(\eta)$ in \cref{eq:ELPD_var_smi_def} as a function of $\eta$ and select an optimal $\eta$-value in \cref{eqn:eta-star-defn} and the variational \acrshort*{smi} posterior $q_{\lambda^*(\eta^*)}$.

Up to this point $\eta\in [0,1]$, has been a scalar.
When $\eta$ is scalar, we can fit the variational posterior independently at a lattice of $\eta$-values, estimate the \acrshort*{elpd} at each value, smooth the estimated \acrshort*{elpd} values over $\eta\in [0,1]$ and select the $\eta$-value maximising this function.
However, when we analyse multi-modular models with multiple misspecified modules, the dimension of $\eta$ grows with the number of bad modules and so independent fitting is both inefficient and computationally prohibitive.
In this section we give two parameterisations of the \acrfull*{vmp}, $q_{\lambda(\alpha,\eta)}$ and $q_{\alpha,\eta}$. In the former, based on a ``\acrshort*{vmp}-map'', the parameters $\lambda$ of the \acrshort*{nf} are themselves parameterised as functions of $\eta$ with parameters $\alpha$. In the latter, based on a ``\acrshort*{vmp}-flow'', $\eta$ is treated as an additional input to the \acrshort*{nf} alongside $\epsilon$, with its own additional flow parameters $\mu$, and $\alpha=(\lambda,\mu)$. 

\subsection{Motivation and definition}

The \acrshort*{smi}-posterior varies continuously with $\eta$.
Expanding the \acrshort*{kl} divergence at $\eta+\delta$,
\[
  \kl{p_{\smi,\eta}}{p_{\smi,\eta+\delta}}=-\delta E_{\pow, \eta}\left(\log(p(Y \mid \varphi,\tilde\theta))\right)+\log\left(E_{\pow, \eta}(p(Y \mid \varphi,\tilde\theta)^{\delta})\right),
\]
and this is continuous and has continuous derivatives in $\delta$ if the integrals exist. This motivates flow- and map- parameterisations of the variational densities $q_{\lambda(\alpha,\eta)}$ and $q_{\alpha,\eta}$ which are continuous in the same sense.

\subsubsection{The VMP-map}
Continuity holds in a stronger sense. Under regularity conditions, a continuous sequence of solutions $\lambda^*(\eta)\in \Lambda^*(\eta)$ passes through any point $\lambda^*\in \Lambda^*(\eta^*)$.
Applying the Implicit Function Theorem (as in \cref{prop:use_IFT_show_loss_limit}), to the $\eta$-dependence of the roots of the functions on the LHS of \cref{eqn:lam-star-roots1a,eqn:lam-star-roots1b} we can show that, for every $\lambda^*\in\Lambda^*(\eta^*)$, there is a unique continuous function $\lambda^*(\eta)$ satisfying $\lambda^*(\eta)\in\Lambda^*(\eta)$ for $\eta$ in an open neighborhood of $\eta^*$ and satisfying $\lambda^*(\eta^*)=\lambda^*$. The regularity conditions require the functions on the LHS of \cref{eqn:lam-star-roots1a,eqn:lam-star-roots1b} to be continuously differentiable in $\lambda$ and $\eta$, and the Jacobians of those functions (in $(\lambda_1,\lambda_3)$ and $\lambda_2$) to be invertible at $\lambda=\lambda^*$.

This motivates a low dimensional reparameterisation of $\lambda$ which approximates $\lambda^*(\eta)$. Let $\lambda(\alpha,\eta)=f_\alpha(\eta)$, where $\alpha\in A$ is a vector of real parameters and let
\begin{equation}\label{eq:f-alpha-vmp-map-defn}
  f_\alpha: H \rightarrow \Lambda
\end{equation}
be a continuously differentiable mapping parameterised by $\alpha$. We refer to $f_\alpha$ as the \textbf{\acrshort*{vmp}-map} and
define a \acrlong*{vmp} as a family of distributions 
\[\mathcal{Q}_{H,\alpha} = \{ q_{f_\alpha(\eta)} \;,\; \eta \in H \},\quad \alpha\in A.\] 

There is a question of how the $\alpha$ parameters should contribute to the different components of $\lambda$. We take $\alpha=(\alpha_1,\alpha_2,\alpha_3)$ and
\[
  f_\alpha(\eta)=(f^{(1)}_{\alpha_1}(\eta),f^{(2)}_{\alpha_2}(\eta),f^{(3)}_{\alpha_3}(\eta))
\]
where $f^{(1)}_{\alpha_1}: H\to \Lambda_1$
and $f^{(2)}_{\alpha_2},f^{(3)}_{\alpha_3}: H\to \Lambda_2$ and set $\lambda_k(\alpha_k,\eta)=f^{(k)}_{\alpha_k}(\eta),\ k=1,2,3$. This gives $f_\alpha: H\to \Lambda$ as before, but breaks up the dependence as
\begin{equation}\label{eq:vmp-map-transform-variables-alpha123}
    (\varphi_{(\lambda_1(\alpha_1,\eta),\epsilon)},\theta_{(\lambda_2(\alpha_2,\eta),\epsilon)},\tilde\theta_{(\lambda_3(\alpha_3,\eta),\epsilon)}) \sim q_{f_\alpha(\eta)}.
\end{equation}
Changing $\alpha_1$ to improve the fit to the $\varphi$ distribution at one $\eta$ does not affect the $\theta$ distribution at another $\eta$-value, though it will affect the $\varphi$ distribution there. 

A very expressive \acrshort*{vmp}-map may be undesirable due to a bias-variance trade off in the estimation of $\lambda^*$. The estimates $\hat\lambda$ of $\lambda^*$ output by \cref{alg:vsmi} are estimated independently over $\eta$ and will not in general lie in $\Lambda^*(\eta)$. Reparameterising with $f_\alpha$ and estimating $\hat\alpha$ for best fit across $\eta\in H$ smooths the output $\hat\lambda(\eta)=f_{\hat\alpha}(\eta),\ \eta\in H$ at the price of some potential bias. Properties (P1-3) hold only approximately on both outputs.

\subsubsection{The VMP-flow}

When we parameterise the \acrshort*{vmp} with a \acrshort*{nf}, we model the $\eta$-dependence of the variational densities $q$. We can parameterise the function $\lambda^*(\eta)$ using a \acrshort*{vmp}-map. Alternatively we can add an $\eta$-input to the maps $T_1,T_2$ (technically an extra conditioner, like $\epsilon_1$). The flow architecture is expanded with extra nodes and weight parameters $\mu=(\mu_1,\mu_2,\mu_3)$ with $\mu\in M$ say. The map $T$ is continuous in its inputs so $q$ will be continuous in $\eta$. Let $\alpha=(\alpha_1,\alpha_2,\alpha_3)$ with $\alpha_k=(\lambda_k,\mu_k),\ k=1,2,3$ and $\alpha\in A$ where now $A=\Lambda\times M$. The transformations $T_1,T_2$ with input $\eta$ are given in terms of the conditioners and transformers of the \acrshort*{nf} in \cref{eq:meta_new_map_phi} to \cref{eq:meta_new_map_theta_tilde} in \cref{sec:nf}. Formally,
\begin{align}
  \varphi_{(\alpha_1,\eta,\epsilon)}      & =T_1(\epsilon_1;\alpha_1,\eta)\nonumber                                                                                                                     \\
  \theta_{(\alpha_2,\eta,\epsilon)}       & =T_2(\epsilon_2; \alpha_2, (\eta,\epsilon_1))\nonumber                                                                                                        \\
  \tilde\theta_{(\alpha_3,\eta,\epsilon)} & =T_2(\epsilon_3; \alpha_3, (\eta,\epsilon_1))\nonumber                                                                                                        \\
  T(\epsilon;\alpha,\eta)                & = \left( T_1(\epsilon_1;\alpha_1,\eta))  ,\; T_2(\epsilon_2; \alpha_2, (\eta,\epsilon_1))  ,\; T_2(\epsilon_3; \alpha_3, (\eta,\epsilon_1))  \right).\label{eq:smi_transform_new_vmp}
\end{align}
We call this extended flow mapping $T$ with input $\eta$ a {\bf \acrshort*{vmp}-flow}. In terms of the new map, the variational densities are
\begin{align}\label{eqn:q-lambda-smi-var-meta-new}
  q_{\alpha,\eta}(\varphi, \theta, \tilde\theta) =q_{\alpha_1,\eta}(\varphi)q_{\alpha_2,\eta}(\theta\mid \varphi)q_{\alpha_3,\eta}(\tilde\theta\mid \varphi),
\end{align}
simply replacing $\lambda_k\leftarrow\alpha_k,\ k-1,2,3$ and making the $\eta$-dependence explicit as it is a flow input. This gives a second
\acrlong*{vmp} as the family of distributions 
\[\mathcal{Q}_{H,\alpha} = \{ q_{\alpha,\eta} \;,\; \eta \in H \}, \quad\alpha\in A.\] 


\subsection{Learning the Variational Meta-Posterior for SMI}\label{subsec:learning_vmp_map}

The \acrlong*{vmp} for \acrshort*{smi} is characterised by a pair $(\mathcal{P}_{\smi}, \mathcal{Q}_H)$, where
$
\mathcal{P}_{\smi}=\{p_{\smi,\eta}: \eta\in H\}
$
is the family of \acrshort*{smi} posteriors indexed by $\eta$ which we want to approximate and
$ \mathcal{Q}_{H}$ is the family of all available \acrlongpl*{vmp}, which can be written $\mathcal{Q}_{H} = \cup_{\alpha\in A}  \mathcal{Q}_{H,\alpha}$
for both \acrshort*{vmp}-map and \acrshort*{vmp}-flow based \acrshortpl*{vmp}.

In this section we give losses for estimation of $\alpha^*$ in the \acrshort*{vmp}-map and \acrshort*{vmp}-flow. Let $\eta_{1:R}=(\eta_r)_{r=1,...,R}$ be a given lattice of $\eta$-values and let
\[\rho(\eta)=\frac{1}{R}\sum_{r=1}^R \delta_{\eta_r}(\eta),\ \eta\in H.\]
The values in $\eta_{1:R}$ would ideally be concentrated around $\eta^*$. As this isn't known in advance, concentrating them near the points $\eta_{1:R}\in \{0,1\}^R$ is a useful rule as $\lambda^*(\eta)$ often varies rapidly with $\eta$ near Cut and Bayes. Adaptive sequential estimation and maximisation of $u(\eta)$ may be of interest in future work.

We take a meta-\acrshort*{smi} loss weighted across $\eta\in \eta_{1:R}$. For the \acrshort*{vmp}-map this is
\begin{align}
  \mathcal{L}^{(\msmi-map)}(\alpha) & =\E_{\eta\sim\rho}\left(\mathcal{L}^{(\smi,\eta)}(f_{\alpha}(\eta)\right)\nonumber                       \\
                                & =\frac{1}{R}\sum_{r=1}^R \mathcal{L}^{(\smi,\eta_r)}(f_{\alpha}(\eta_r)) \label{eq:meta_smi_target_loss}
\end{align}
where $\mathcal{L}^{(\smi,\eta)}(\lambda)$ is defined in \cref{eq:vsmi_loss} and we have taken $\lambda=f_{\alpha}(\eta)$ in order to enforce the parameterisation at each $\eta\in \eta_{1:R}$. For the \acrshort*{vmp}-flow the loss is
\begin{align}\label{eq:vsmi_loss-meta-new-main}
  \mathcal{L}^{(\msmi-flow)}(\alpha)  =\E_{\eta\sim\rho}\left(\mathcal{L}^{(\msmi-flow,\eta)}(\alpha)\right) 
\end{align}
where $\mathcal{L}^{(\msmi-flow,\eta)}(\alpha)$ is obtained by substituting $\varphi_{(\alpha_1,\eta,\epsilon)}$ etc into $\mathcal{L}^{(\smi,\eta)}$ in \cref{eq:vsmi_loss} and is defined in detail in \cref{eq:vsmi_loss-meta-new} in \cref{sec:nf}.

In the \acrshort*{vmp}-flow the optimal variational parameters $\lambda^*$ don't depend on $\eta$ and this seems to give relatively more rapid and stable convergence in \acrshort*{sgd} targeting $\mathcal{L}^{(\msmi-flow)}(\alpha)$ compared to \acrshort*{sgd} targeting in $\mathcal{L}^{(\msmi-map)}(\alpha)$ (given in \cref{alg:v_meta_smi}). It is a relatively ``lightweight'' parameterisation, as the dimension of $\alpha$ in the \acrshort*{vmp}-flow is quite a bit smaller than that of $\mu$ in the \acrshort*{vmp}-map.

In order to estimate $\alpha^*=\arg\min_{\alpha\in A} \mathcal{L}^{(\msmi)}(\alpha)$ (dropping the -map and -flow distinction, and ignoring non-uniqueness for brevity) and fit the \acrshort*{vmp}, we apply \acrshort*{sgd}, simply replacing $\mathcal{L}^{(\smi,\eta)}(\lambda)$ \cref{alg:vsmi} with $\mathcal{L}^{(\msmi)}(\alpha)$ in \cref{alg:v_meta_smi} in \cref{sec:sgd-for-meta-losses}, and updating $\alpha$ with the gradient $\nabla_\alpha \mathcal{L}^{(\msmi)}$ instead of updating $\lambda$ with the gradient $\nabla_{\lambda} \mathcal{L}^{(\smi,\eta)}$. We implement this using \acrfull*{svi}. We take a continuous density $\rho$ and sample a new batch $\eta_{1:R}$ of $\eta$-values at each pass of the \acrshort*{sgd} algorithm. We approximate the family  $\mathcal{P}_{\smi}$ in a single end-to-end optimisation, propagating the loss function gradients through the \acrshort*{vmp}-map or -flow using automatic differentiation. See \cref{alg:v_meta_smi} in \cref{sec:sgd-for-meta-losses}.

When $\dim(H)=C$ ($C$ influence parameters for $C$ cuts) we have vectors $\eta_r=(\eta_{1,r},...,\eta_{C,r}),\ r=1,...,R$. We defined $\rho$ for resampling purposes as $\eta_{c,r}\sim \mbox{beta}(a,1),\ c=1,...,C$ with $a<1$ independently for each component of $\eta$, for example $a=0.2$. This concentrates sampled $\eta$ at the Cut boundaries of $H$ as noted above.

\subsection{Maximising the ELPD using the Variational Meta-Posterior}\label{subsec:best_smi_vmp}

The \acrshort*{vmp} allows us to produce posterior samples for any given $\eta \in H$ efficiently. This helps us find the best influence parameter $\eta^*$ (see \cref{subsec:best_smi}), as we can estimate the utility function $u(\eta)$ accurately in fractions of a second and compare it across the family of \acrshort*{smi} posteriors. In settings with a single cut, maximising $u(\eta)$ may be as simple as linear search, but in the case of many potential cuts, with a higher-dimensional $H$ space, we require more elaborate search strategies.

In our case the utility $u(\eta)$ is the \acrshort*{elpd} and we estimate this (its negative) using the WAIC estimator given in \cite{Vehtari2016} and minimise the WAIC over $\eta\in H$ with \acrshort*{sgd}. Denote by $\psi_{\hat\alpha,\eta,\epsilon}=(\varphi_{\hat\alpha,\eta,\epsilon},\theta_{\hat\alpha,\eta,\epsilon},\tilde\theta_{\hat\alpha,\eta,\epsilon})$ a full sample parameter vector $\psi_{\alpha,\eta,\epsilon}\sim q_{\hat\alpha,\eta}$ from a fitted \acrshort*{vmp}-flow evaluated at $\eta$. Denote by $\psi_{\alpha,\eta,\epsilon_{1:J}}$ a set of $J$ iid samples from $q_{\alpha,\eta}$ and let $\mathcal{D}=(Y,Z)$ denote the data. The WAIC is a function $\widehat {-\elpd}(\psi_{\alpha,\eta,\epsilon_{1:J}},D)$ of the samples and data \citep{Vehtari2016}. In order to implement \acrshort*{sgd} we need function evaluations and derivatives of $\widehat {-\elpd}(\eta)$ wrt $\eta$ (keeping only the $\eta$ dependence). Function evaluations are very fast. In our JAX/TensorFlow setup \citep{deepmind2020jax,Dillon2017tfp} we get $\eta$-derivatives using automatic differentiation through the functions in the \acrshort*{elpd} and all the way into $\psi_{\alpha,\eta,\epsilon_{j}},\ j=1,...,J$.
This can be seen in operation in our online code.

The main difficulty (in our example in \cref{subsec:exp_rnd_eff} where $\eta\in [0,1]^{N}$) is that the \acrshort{elpd} is clearly non-convex (from our plots), and quite flat when the $\eta$ components are all close to one. We therefore initialise \acrshort*{sgd} using an (informed) greedy backward search. This uses backwards selection over cuts starting from Bayes, cutting the module which gives the greatest reduction in $\widehat {-\elpd}$ and stopping when no decrease is possible. 

\section{Experiments} \label{sec:experiments}

Our experiments illustrate the following points: Variational-\acrshort*{smi} with a \acrshort*{nf} and with or without a \acrshort*{vmp} accurately approximates \acrshort*{smi}-posteriors at all $\eta$ in the examples we consider; the \acrshort*{vmp}-framework allows us to select an influence parameter \emph{vector} $\eta^*\in H$, where $H=[0,1]^C$ and $C$ is the number of cuts, at values of $C$ which are completely out of reach for one-$\eta$-at-a-time \acrshort*{mcmc} or variational-\acrshort*{smi}.
\acrshort*{mcmc} is fine if we want to check the \acrshort*{vmp} at a handful of $\eta$ values. 

We use two examples which have become default test cases \citep[e.g.][]{Plummer2015cut,Jacob2017together,Carmona2020smi,Liu2020sacut,Nicholls2022smi}. In the first epidemiological example taken from \cite{Plummer2015cut}, we show that \acrshort*{vmp} agrees with variational-\acrshort*{smi} and nested-\acrshort*{mcmc} (which serves as ground truth) across a range of $\eta$-values and in particular at $\eta=0,1$, Cut and Bayes.  An expressive variational family is needed, so while \acrshortpl*{nf} are effective, \acrshort*{mfvi} fails.  In the second random effects example taken from \cite{Jacob2017together}, we illustrate variational-\acrshort*{smi} with multiple cuts,
and compare different methods for estimating the utility $u(\eta)$ which we take as the \acrshort*{elpd} throughout.
In a companion paper, \cite{Carmona2022spatial}, we give an analysis of a spatial model where \acrshort*{mcmc} at even one $\eta$-value is infeasible. This (third) extended example illustrates careful choice of utility for $\eta$-selection, as well as being of independent interest in the application domain.

Our variational family has \acrfull*{nsf} transformers \citep{Durkan2019neural} with \acrshort*{mlp} conditioners in eight coupling layers \citep{Dinh2016realnvp}, with a \acrshort*{mfvi} analysis for comparison. See \cref{sec:nf} for this terminology. We found this arrangement gave an expressive transformation $T=(T_1,T_2,T_2)$ that was easily trained and worked for both examples. Code to replicate all results in this section is available as an open-source repository \footnote{\url{https://github.com/chriscarmona/modularbayes}}.

Our implementation is based on DeepMind JAX Ecosystem \citep{deepmind2020jax} and TensorFlow Probability \citep{Dillon2017tfp}. Experiments were carried out using a single Cloud TPU machine type v3-8.
Qualitative runtimes to approximate a single \acrshort*{smi} posterior for the Random Effects example using our favoured \acrshort*{nf} were in the range of $10$ minutes and sampling 10000 iid samples takes less than a second. This total time is similar to the time to obtain one correlated sample of size $10,000$ at one $\eta$-value using nested \acrshort*{mcmc}. Training the \acrshort*{vmp} required between $0.5$ and $2$ hours. However, this training time is compensated by a significant reduction in the search for $\eta^*$, as we can generate samples from any $q_{\hat\alpha,\eta}$ and estimate $u(\eta)$ (the WAIC) in a fraction of a second. Optimisation using greedy initialisation and \acrshort*{sgd} requires thousands of WAIC-estimates and took about 5 minutes using the \acrshort*{vmp}, whereas each of these estimates would take 10 minutes using \acrshort*{mcmc}. Further, we cannot get gradients by automatic differentiation in nested \acrshort*{mcmc}.

\subsection{Epidemiological Model} \label{subsec:exp_epidemiology}

We revisit the well-known epidemiological model for the correlation between \acrfull*{hpv} prevalence and cervical cancer incidence \citep[see][for details]{Maucort-Boulch2008, Plummer2015cut}.
In this modular model a small ``expensive'' prospective trial controlling sample selection from the target population gives straightforward statistical modelling. A second much larger retrospective data set contains information about population parameters, but was gathered with little control over sample selection bias. This sort of data synthesis appears frequently. For example, the simplest Covid prevalence model in \cite{Nicholson2021covid}, which brings together sample survey data and walk-in testing results, belongs to this class.

The data consist of four variables observed from $n=m=13$ groups of women from $n$ different countries.
The model has two modules, a Binomial distribution for the number $Z_i,\ i=1,...,n$ of women infected with HPV in a sample of size $N_i$ from the $i$'th group and a Poisson distribution for the number of cancer cases $Y_i$ during $T_i$ women-years of followup. That is,
\begin{gather*}
  Z_i \sim Binomial(N_i, \varphi_i ). \nonumber \\
  Y_i \sim Poisson( \mu_i ) \\
  \mu_i=T_i \exp( \theta_1+\theta_2 \varphi_i ),\quad i=1,\ldots,13.\nonumber
\end{gather*}
Following previous authors, the parameter spaces are $\theta\in [0,\infty)^2$ and $\varphi\in [0,\infty)^n$ and the priors are truncated independent normal priors with variance 1000.


Our variational approximation takes an $L=17$-dimensional independent standard Normal $p(\epsilon)=N(\epsilon; 0,\mathbb{I}_{17})$ as our \emph{base distribution} ($L_1=13$ elements for $\varphi$, $L_2=2$ for $\theta$, and so also $L_2=2$ for $\tilde\theta$).
The \acrshort*{nf}-conditioner in $T_2$ (\cref{sec:nf}) takes $\epsilon_{1:13}$ as an input, allowing correlation between $\varphi$ and $\theta$ and between $\varphi$ and $\tilde\theta$, and conditional independence between $\theta$ and $\tilde\theta$ given $\varphi$. 

Samples from a \acrshort*{vmp} $q_{f_{\hat\alpha}(\eta)}$ fitted using \acrshort*{vmp}-\acrshort*{nsf}-\acrshort*{mlp} and a \acrshort*{vmp}-map are shown in  \cref{fig:epidemiology_vmp} (at $\eta\in \{0.001,0.1,1\}$, corresponding to the Cut, Bayes, and a value of $\eta$ ``halfway'' between the two \footnote{for illustration we take 0.1 instead of 0.5 because posteriors with $\eta	\gtrsim 0.2 $ are very similar}). ``Ground truth'' \acrshort*{smi} distributions obtained using nested-\acrshort*{mcmc} \citep{Plummer2015cut,Carmona2020smi} are shown in \cref{fig:epidemiology_mcmc} for comparison. Samples from the \acrshort*{vmp} $q_{\hat\alpha,\eta}$ fitted using a \acrshort*{vmp}-flow and samples from variational-\acrshort*{smi} distributions $q_{\hat\lambda(\eta)}$ (estimating $\lambda^*(\eta)$ separately at each $\eta$ without a \acrshort*{vmp}) are essentially identical to the variational $q_{f_{\hat\alpha}}(\eta)$ posteriors and are omitted. The good agreement here to \acrshort*{mcmc} shows both that the training losses $\mathcal{L}^{(\smi,\eta)}$ and $\mathcal{L}^{(\msmi)}$ we wrote down in \cref{subsec:vi_multi} and \cref{sec:meta_posterior} are doing their job and enforcing a good fit to $p_{\smi,\eta}$ over all $\eta\in [0,1]$, and at the same time interpolating variational approximations with good inferential properties to the Cut (no $Y$-module feedback) and Bayes (full feedback) posteriors.

In \cref{sec:experiments_extra} we include a comparison with \acrshort*{mfvi} (see \cref{fig:epidemiology_mfvi}). This demonstrates its failure, under-dispersed relative to the target $p_{\smi,\eta}$, and demonstrates the advantages of using an expressive flow-family.
In this example we omit the final stage of an \acrshort*{smi} analysis, that is, we do not estimate the \acrshort*{elpd} and select $\eta^*$ and $p_{\smi,\eta^*}$. As the variational posteriors match nested \acrshort*{mcmc}, this part of the analysis is the same as that given in \cite{Carmona2020smi} (though faster, as sampling our flow is \emph{much} faster than nested-\acrshort*{mcmc}).

\begin{figure}[!htb]
  \centering
  \includegraphics[width=0.48\textwidth]{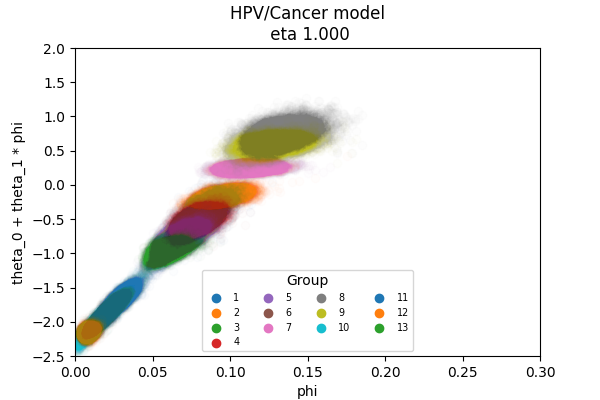}
  \includegraphics[width=0.37\textwidth]{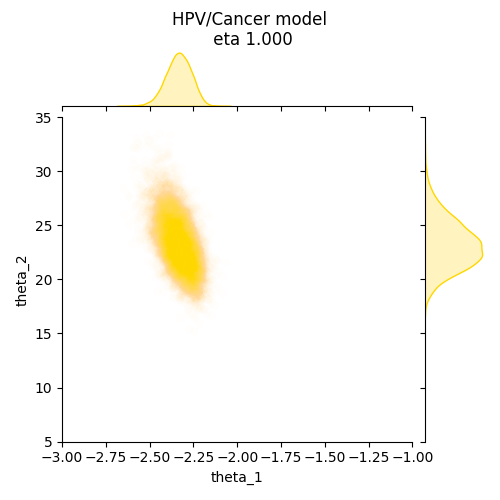}
  \includegraphics[width=0.48\textwidth]{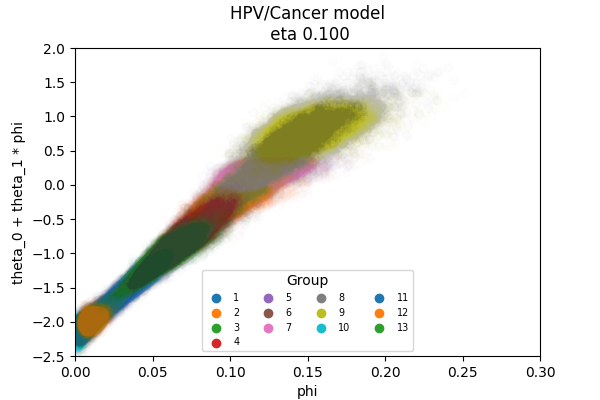}
  \includegraphics[width=0.37\textwidth]{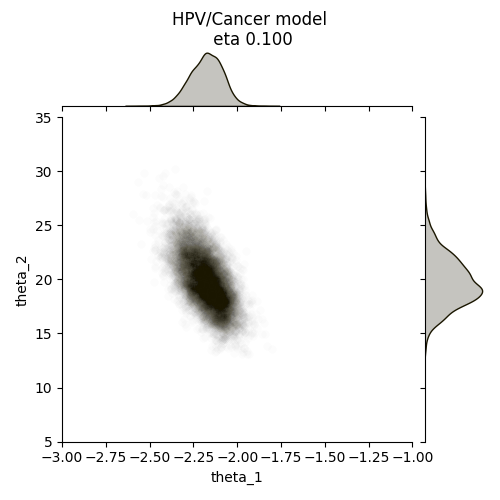}
  \includegraphics[width=0.48\textwidth]{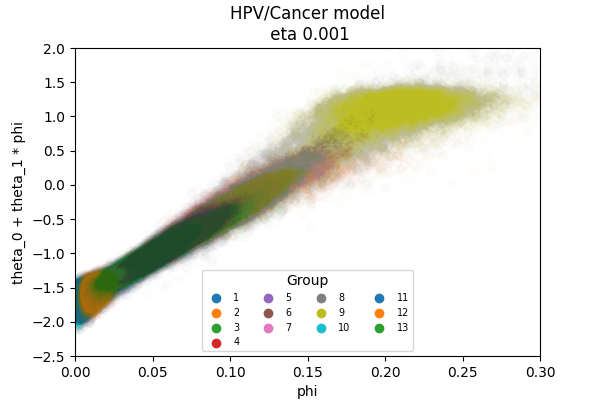}
  \includegraphics[width=0.37\textwidth]{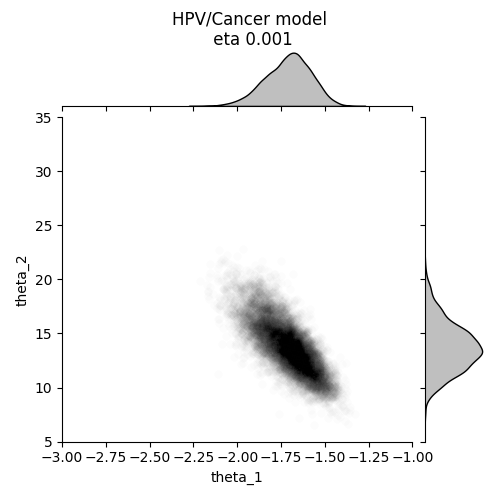}

  \caption[Epidemiology model VMP]{
    Samples from the \acrlong*{vmp} in \cref{subsec:exp_epidemiology}, obtained using \acrshort*{vmp}-\acrshort*{nsf}-\acrshort*{mlp}.
    Rows correspond to influence levels $\eta=(0.001,0.1,1)$.
    The left column shows the relation between HPV prevalence ($\varphi$) and cervical cancer incidence $\mu_i=\theta_1+\theta_2\varphi_i$ colored by group $i=1,...,n$ for the $n=13$ groups in the data.
    The right column shows how the joint distribution of slope ($\theta_1$) and intercept ($\theta_2$) varies with $\eta$.
    All samples are produced using a single \acrshort*{vmp}-map.
  }
  \label{fig:epidemiology_vmp}
\end{figure}

\subsection{Random effects model with Multiple Cuts} \label{subsec:exp_rnd_eff}

The example in this section illustrates the \acrshort*{vmp} with multiple cuts, demonstrating its convenience in settings with more than two modules. The modules are all potentially misspecified, which in this model gives thirty cuts.
The \emph{influence parameter vector} $\bmo{\eta}=(\eta_1,\ldots,\eta_{N})$, $\eta\in H$ regulates the influence of each module.
We take our random effects model and synthetic data setup from \cite{Liu2009modularization} and \cite{Jacob2017together}.

Denote by $Y_i=(y_{i,1}, \ldots, y_{i,n_i})$ the data in group $i=1,...,N$.  We take $N=30$ and $n_i=5$ below. The hierarchical Gaussian model with random effects $\beta=(\beta_i)_{i=1,...,N}$ and variances $\sigma=(\sigma_i)_{i=1,...,N}$ is specified as follows:
\begin{align*}
  Y_{i,j} \sim Normal( \beta_{i}, \sigma_{i}^2 ), & \quad i=1,\ldots,N \; ; \; j=1,\ldots,n_i \\
  \beta_{i} \mid \tau \sim Normal( 0, \tau^2 ),   &
\end{align*}
with $\sigma$ and $\tau$ priors
\begin{align*}
  p(\sigma_{i}^2)            & \propto \sigma_{i}^{-2}                                                          \\
  p(\tau \mid \bmo{\sigma} ) & \propto \frac{1}{\tau^2 + \frac{1}{N} \sum_{i=1}^{N} \frac{\sigma_{i}^2}{n_i} }.
\end{align*}
A graphical representation of the model and its cuts is displayed in \cref{fig:rnd_eff_model}.
\begin{figure}[!htb]
  \centering
  \def\svgwidth{0.4\textwidth}
  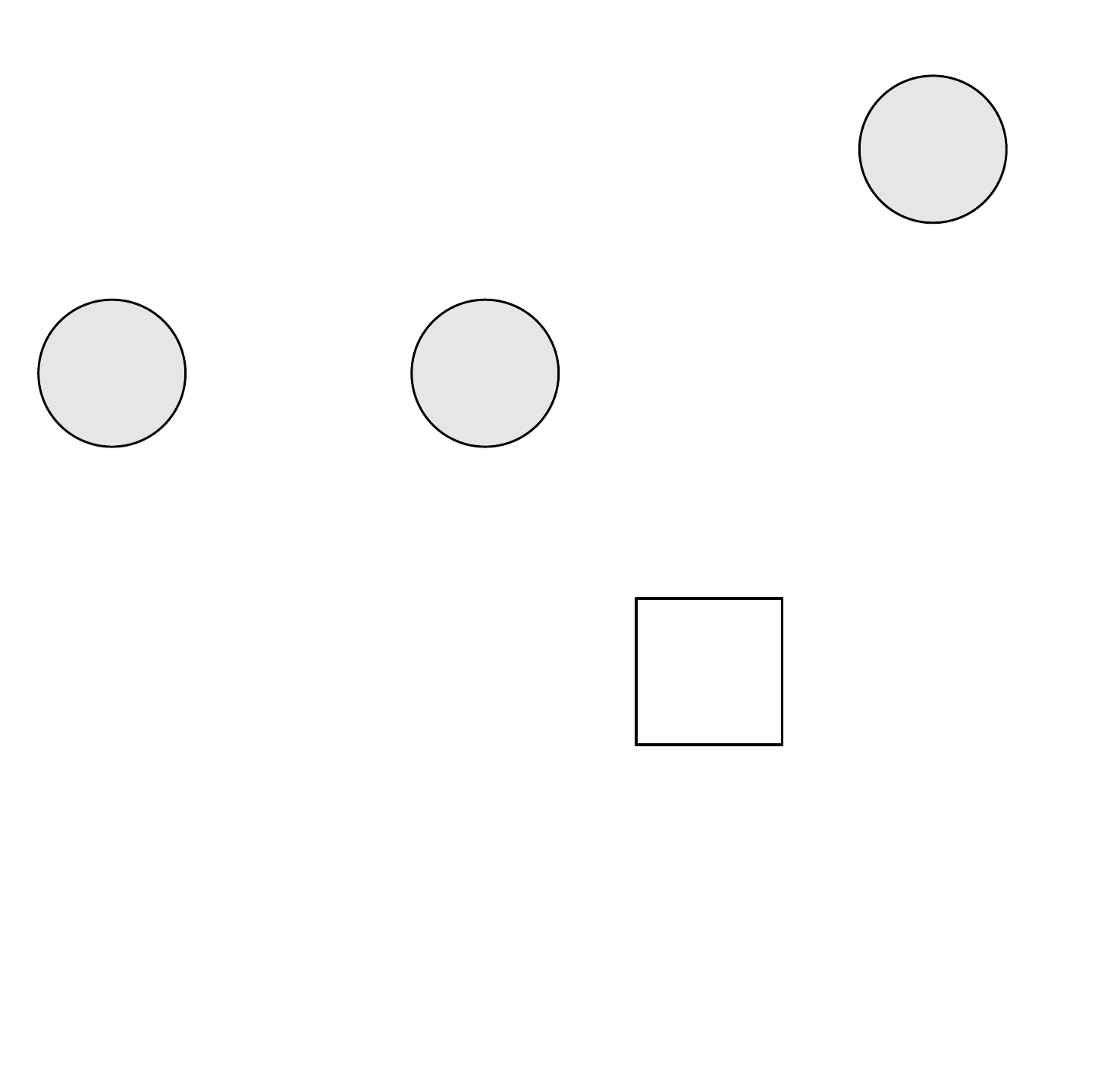
  \caption{Graphical representation of the model for random effects.}
  \label{fig:rnd_eff_model}
\end{figure}
The Bayes posterior is
\[
  p(\bmo{\sigma},\bmo{\beta},\tau \mid Y) \propto p(Y \mid \bmo{\sigma},\bmo{\beta})p(\bmo{\beta} \mid \tau)p(\tau \mid \bmo{\sigma}) p(\bmo{\sigma}).
\]
Since this is a study on synthetic data, we follow \cite{Liu2009modularization} and take the number of groups $N$ to be large and the number of replications $n_i$ per group to be small as this gives a strong distorting effect under the misspecified model.
We simulate data from the true observational model, $p^*(Y)=p(Y\mid \bmo{\beta^*},\bmo{\sigma^*})$, with two very large random effects, $\beta^*_1=10$ and $\beta^*_2=5$ and zero for the rest, $\beta^*_i=0$ for $i=3,\ldots,N$. We take a unit scale for all groups, $\sigma^*_i=1$ for $i=1,\ldots,N$.
The data are shown in \cref{fig:rnd_eff_data}.
\begin{figure}[!htb]
  \centering
  \includegraphics[width=\textwidth]{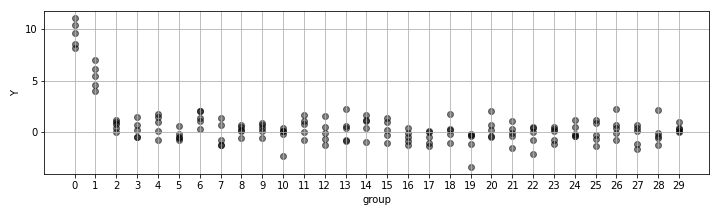}
  \caption[Random effect data]{Simulated data for the random effects model of \cref{subsec:exp_rnd_eff}. With $\beta_1^*=10$, $\beta_2^*=5$, $\beta^*_i=0$ for $i=3,\ldots,N$ and $\sigma^*_i=1$ for $i=1,\ldots,N$.}
  \label{fig:rnd_eff_data}
\end{figure}

Our choice of how to divide the model into modules depends on how we plan to cut.
\cite{Liu2009modularization} have two modules and one cut: Module 1 is the observation module $(Y_{i},\beta_i,\sigma_i)_{i=1,...,N}$, while Module 2 is the prior module $(\beta_i,\sigma_i,\tau)_{i=1,...,N}$. Those authors discuss how the Bayes posterior is distorted when the underlying true random effect $\beta_1$ in a single group of observations is significantly different from the rest of the groups, so the $\beta|\tau$ prior is misspecified. They take a Cut posterior $p_{\cut}(\sigma,\tilde\beta|Y)$, replacing the prior $p(\beta|\tau)$ with an improper imputation prior $\tilde p(\tilde\beta)\propto 1$. This eliminates \emph{prior feedback} as in \cref{subsec:mod_prioir_feed_paper} and \cref{sec:prior_feedback} using an implicit imputation prior. This prior is conjugate so the random effects $\tilde\beta$ are integrated out to give, at the first stage, $p_{\cut}(\sigma \mid Y)=\prod_i p(\sigma_i|Y_i)$. At the second stage the Cut distribution of $\sigma|Y$ is fed into the posterior $p(\beta,\tau|Y,\sigma)$, using the original prior $p(\beta \mid \tau)$ and conditioning on the imputed $\varphi$.

In contrast to this all (Bayes) or nothing (Cut) approach, \acrshort*{smi} \emph{reduces} the influence of \emph{some} of the groups that may be causing contamination of the posterior.
Our modules are for $i=1,...,N$ the separate generative models for $(Y_i,\beta_i,\sigma_i)$ and $(\beta_i,\sigma_i,\tau)$ so we have in effect $2N=60$ modules, and we work with the joint distributions involving $\beta$. The marginals would still be available, but the higher dimensional joint distributions have interesting shapes and present more of a challenge.
Our target \acrshort*{smi} posterior is
\begin{equation}\label{eqn:rnd_eff_smi_posterior}
  p_{smi, \bmo{\eta}}(\bmo{\sigma}, \bmo{\beta}, \tau, \bmo{\tilde\beta}, \tilde\tau \mid Y) = p_{pow, \bmo{\eta}}(\bmo{\sigma}, \bmo{\tilde\beta}, \tilde\tau \mid Y) \; p(\bmo{\beta}, \tau \mid \bmo{\sigma}, Y),
\end{equation}
where
\begin{align}
  p_{pow, \bmo{\eta}}(\bmo{\sigma}, \bmo{\tilde\beta}, \tilde\tau \mid Y) & \propto p(Y \mid \bmo{\tilde\beta}, \bmo{\sigma} ) \; p(\tilde\tau \mid \bmo{\sigma}) \; \prod_{i=1}^{N} \tilde{p}_{\eta_i}(\tilde\beta_i \mid \tilde\tau), \\
  p(\bmo{\beta}, \tau \mid \bmo{\sigma}, Y)                               & \propto p(Y \mid \bmo{\beta}, \bmo{\sigma} ) \; p(\tau \mid \bmo{\sigma}) \; \prod_{i=1}^{N} p(\beta_i \mid \tau)
\end{align}
Performing \acrlong*{smi} using \cref{eqn:rnd_eff_smi_posterior} entails a 30-dimensional influence parameter $\eta \in H$ with $H= [0,1]^{30}$.
The \acrshort*{smi} imputation is
\[
  (\sigma,\tilde\beta,\tilde\tau)\sim p_{pow, \bmo{\eta}}(\bmo{\sigma}, \bmo{\tilde\beta}, \tilde\tau \mid Y),
\]
and the analysis is
\[
  (\bmo{\beta}, \tau \mid \bmo{\sigma})\sim p(\bmo{\beta}, \tau \mid \bmo{\sigma}, Y).
\]
The modulated priors $\tilde{p}_{\eta_i}$ used at the imputation stage are chosen to interpolate between the same cut prior $\tilde p(\tilde\theta)\propto 1$ and analysis priors as before. Following the discussion in \cref{sec:prior_feedback}, we define the modulated imputation prior as a normal density
\[
  \tilde{p}_{\eta_i}(\tilde\beta_i \mid \tilde\tau) = \mathcal{N}(\tilde\beta_i ; 0, \tilde\tau / \eta_i ),
\]
and $\tilde{p}_{\eta}(\tilde\beta \mid \tilde\tau)=\prod_i \tilde{p}_{\eta_i}(\tilde\beta_i \mid \tilde\tau)$. This parameterisation gives the Cut prior $\tilde{p}_{\eta}(\tilde\beta \mid \tilde\tau)\to \tilde{p}(\tilde\beta)$ as $\eta\to 0$ and the Bayes prior $\tilde{p}_{\eta}(\tilde\beta \mid \tilde\tau)\to p(\tilde\beta|\tilde\tau)$ as $\eta\to 1$.

Our goal is to optimally modulate the feedback from each group into the shared distribution of the $\beta$'s. Accurate estimates of the utility $u(\eta)$ in \cref{eq:ELPD_var_smi_def} are needed in order to locate the maximum-utility influence-vector $\eta^*$ in \cref{eqn:eta-star-defn}. We could
approximate the modular posterior, using either nested-\acrshort*{mcmc} or variational \acrshort*{smi}, at a lattice of $\eta$-values in $H=[0,1]^N$, but this quickly becomes impractical with increasing $N$. Instead we approximate the candidate \acrshort*{smi} posteriors at \emph{all} $\eta$ in a single function by learning $q_{f_{\hat\alpha}(\eta)}$ and $q_{\hat\alpha,\eta}$, the flow- and map-based \acrlongpl*{vmp}. We found we could do this fairly accurately with the essentially same \acrshort*{vmp}-\acrshort*{nsf}-\acrshort*{mlp} meta-variational setup we used in \cref{subsec:exp_epidemiology}.
The inputs to the flow are $\epsilon_1\sim N(0,\mathbb{I}_N)$ (these express $\sigma$), $\epsilon_2\sim N(0,\mathbb{I}_{N+1})$ (expressing $\beta, \tau$) and $\epsilon_3\sim \epsilon_2$ (expressing $\tilde\beta, \tilde\tau$). Runtimes for \cref{alg:v_meta_smi} are manageable as simulation of the \acrshort*{vmp},
\begin{equation}\label{eq:hm_sim_vsmi}
  (\sigma,\beta,\tau,\tilde\beta,\tilde\tau)\sim q_{f_\alpha(\eta)} 
\end{equation}
at given $\alpha,\eta$ is fast: set $(\lambda_1,\lambda_2,\lambda_3)=f_{\alpha}(\eta)$ and then compute $\sigma=T_1(\lambda_1,\epsilon), (\beta,\tau)=T_2(\lambda_2,\epsilon)$ and $(\tilde\beta,\tilde\tau)=T_2(\lambda_3,\epsilon)$ using the deterministic flow mapping in \cref{eq:smi_transform}. Simulation of $q_{\hat\alpha,\eta}$ is slightly faster.

Our results at $\eta=\bm{0}_N,\bm{1}_N$ are qualitatively consistent with the Cut and Bayes analyses in \citet[Sec.~4.4]{Jacob2017together}. We took two misspecified effects rather than one, but in other respects the setup is the same. Samples from the exact $p_{\smi,\eta}$ posterior produced via \acrshort*{mcmc} are shown in \cref{fig:rnd_eff_mcmc} for a selection of variable pairs.
Comparing these with the corresponding distributions given by \acrshort*{vmp}-\acrshort*{nsf}-\acrshort*{mlp} in \cref{fig:rnd_eff_vmp} using a \acrshort*{vmp}-map, we see good agreement across each of the three rows/$\eta$-configurations, despite the highly irregular contour shapes. We emphasise that \emph{all} samples are produced from a \emph{single} \acrshort*{vmp} $q_{f_{\hat\alpha}(\eta)}$, and we just plug in different $\eta$-values to get different rows in \cref{fig:rnd_eff_vmp}. The \acrshort*{vmp}-flow density $q_{\hat\alpha,\eta}$ converged much more rapidly to agreement with the ground truth than did the \acrshort*{vmp}-map density $q_{f_{\hat\alpha}(\eta)}$, but $\mathcal{Q}_{H,\hat\alpha}$ approximates $\mathcal{P}_{\smi}$ accurately in both cases. We omit the corresponding \acrshort*{vmp}-flow plots as they are essentially identical. 

Selected components of the \acrshort*{vmp}-map $f_{\hat\alpha}(\eta)$ are shown in \cref{fig:rnd_eff_vmp_map}. The surfaces in the top row show a complex structure across the two axes: varying influence parameters in the misspecified groups has a significant impact on the variational posteriors.
However, the surface in the bottom row of plots is almost constant with $\eta_3$: cutting one of these ``good'' groups with labels $3,...,30$ doesn't have a strong impact on the variational posterior.

Producing nested-\acrshort*{mcmc} plots for comparison with the corresponding \acrshortpl*{vmp} at a few $\eta$-values is undemanding. However, this is where the contribution from \acrshort*{mcmc} ends. Accurate estimation of the the utility $u(\eta)=\elpd(\eta)$ over $\eta\in H$ becomes prohibitively expensive using two-stage nested-\acrshort*{mcmc} methods.
The posterior predictive $p(y,z|Y,Z)$ in \cref{eq:ELPD_var_smi_def} must in general be estimated using samples from the \acrshort*{vmp}. Although the density $q_{\hat\alpha,\eta}$ is available in closed form and can be sampled independently, it is nevertheless a complicated function.  However, as noted above, the simulation in \cref{eq:hm_sim_vsmi} is fast. In \cref{fig:rnd_eff_elpd} we plot (negative) $\elpd(\eta)$-surfaces using the operational WAIC-estimator. We check this estimate using direct simulation of synthetic data $y'\sim p^*(\cdot)$. In the top row, reducing the feedback from the two misspecified modules improves predictive performance (the \acrshort*{elpd} is larger at smaller $\eta_1,\eta_2$-values).
In the bottom row, where we vary $(\eta_1,\eta_3)\in[0,1]$, the rates for one misspecified and one well-specified group, we see that the \acrshort*{elpd} surface is relatively flat for $\eta_3$, the influence a well-specified group, though trending up with increasing $\eta_3$. 

\Cref{fig:rnd_eff_greedy_search} illustrates the initialisation stage for the minimisation of the WAIC. The second \acrshort*{sgd}-stage on the WAIC target terminates quickly from this initialisation. The estimated optimal values $\hat\eta^*$ (rounded the first decimal place) are 
\[\hat\eta^*=(0,0,1,1,1,1,1,1,1,1,1,1,1,1,1,0.9,1,1,0.9,1,1,1,0.8,1,1,1,1,1,1,1).\]
This result shows the method is working as information from the first two modules is cut while the rest are Bayes or close to Bayes. This is as expected as we have synthetic data with modules 1 and 2 misspecified.
The resulting $\hat\eta^*$ gives a $p_{\smi,\hat\eta^*}$ which is hard to distinguish from
the $p_{\smi,\eta}$ posterior at $\eta=(0,0,1,...,1)$ (shown
in \cref{fig:rnd_eff_model}) as the \acrshort*{smi}-posterior is insensitive to changes to $\eta_3,...,\eta_N$ close to Bayes values.

\begin{figure}[!htb]
  \centering
  \includegraphics[width=0.24\textwidth]{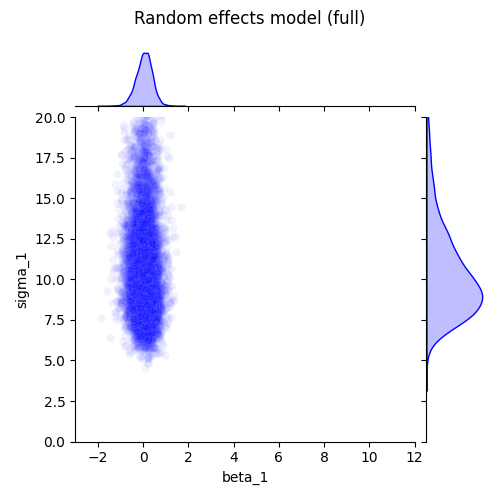}
  \includegraphics[width=0.24\textwidth]{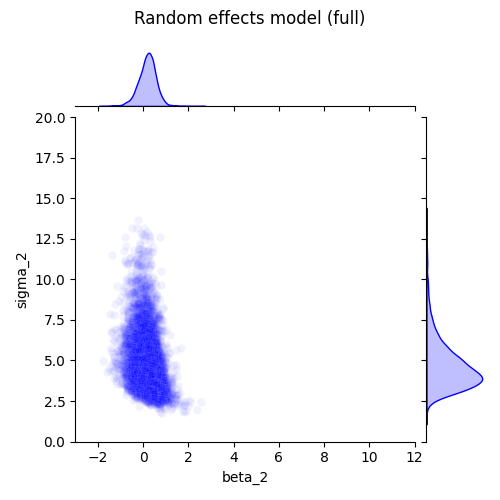}
  \includegraphics[width=0.24\textwidth]{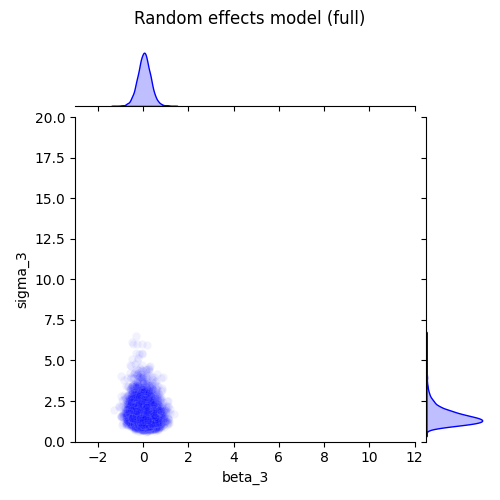}
  \includegraphics[width=0.24\textwidth]{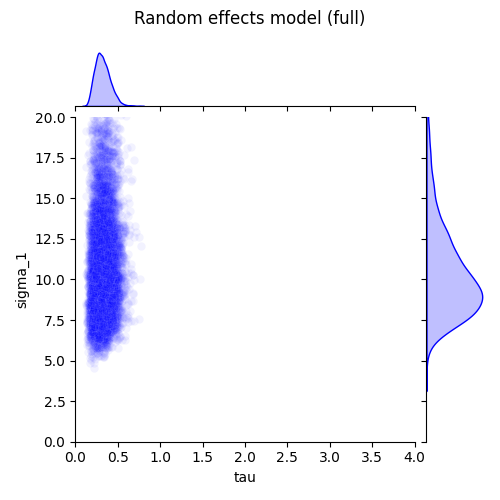}

  \includegraphics[width=0.24\textwidth]{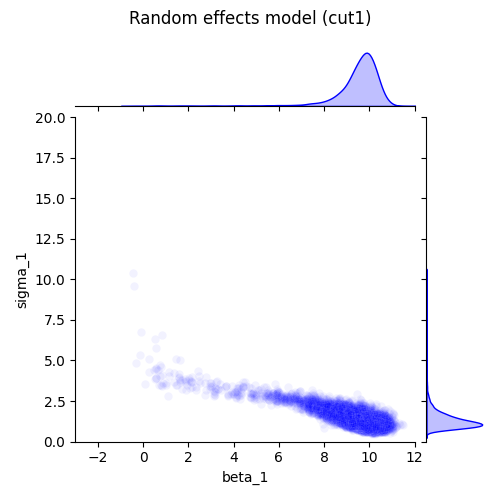}
  \includegraphics[width=0.24\textwidth]{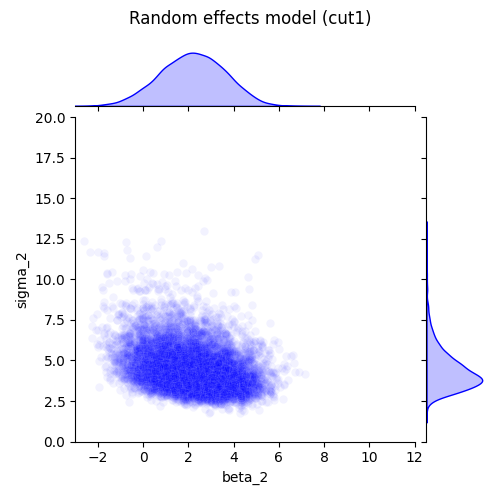}
  \includegraphics[width=0.24\textwidth]{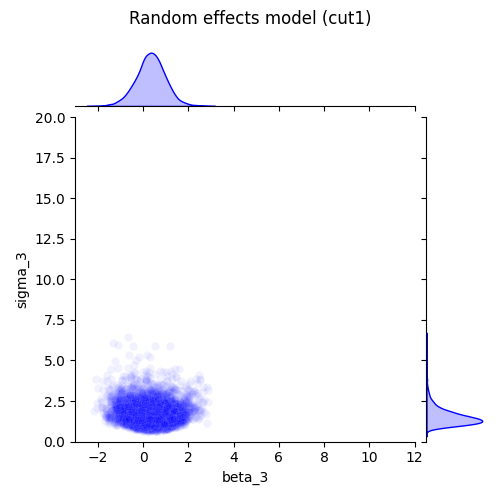}
  \includegraphics[width=0.24\textwidth]{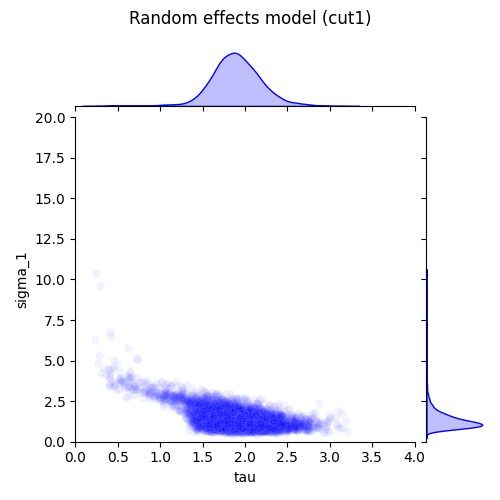}

  \includegraphics[width=0.24\textwidth]{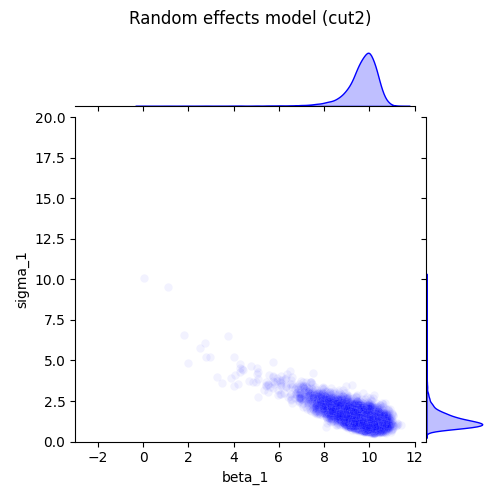}
  \includegraphics[width=0.24\textwidth]{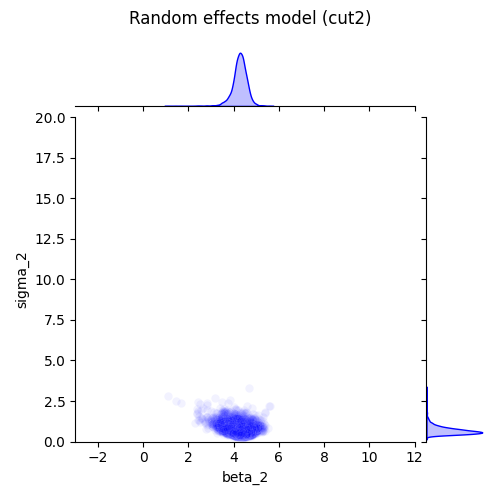}
  \includegraphics[width=0.24\textwidth]{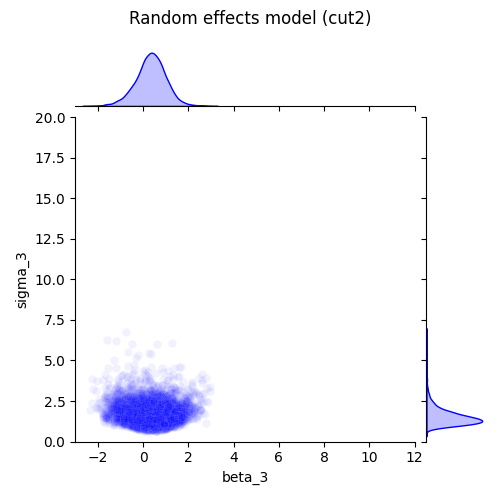}
  \includegraphics[width=0.24\textwidth]{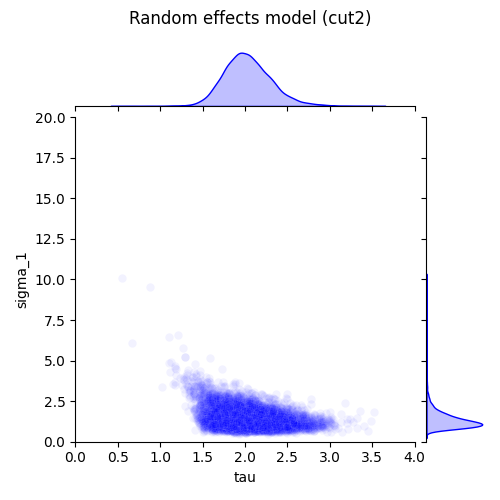}

  \caption[Random Effects model VMP]{
    Samples from the \acrlong*{vmp} of the Random Effects model, obtained from a \acrshort*{vmp}-\acrshort*{nsf}-\acrshort*{mlp} architecture. Each graph shows the joint distribution of a selected pair of parameters. Rows correspond to three modular \emph{feedback} configurations between groups: (Top row) Bayes, $\eta_1=...=\eta_{30}=1$; (Middle) One Cut module, $\eta_1=0$, $\eta_2=...=\eta_{30}=1$; (Bottom) Two Cut Modules, $\eta_1=\eta_2=0$, $\eta_3=...=\eta_{30}=1$.
    Compare \acrshort*{mcmc} in \cref{fig:rnd_eff_mcmc}.
  }
  \label{fig:rnd_eff_vmp}
\end{figure}

\begin{figure}[!htb]
  \centering
  \includegraphics[width=0.98\textwidth]{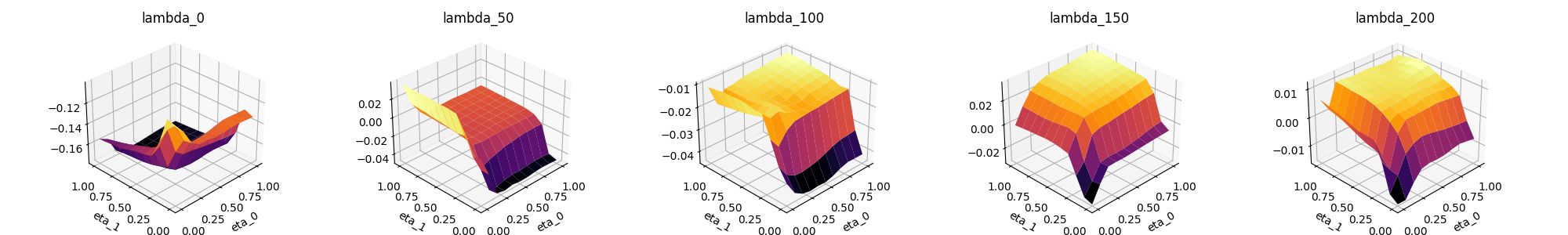}
  \includegraphics[width=0.98\textwidth,trim={0 0 0 1cm},clip]{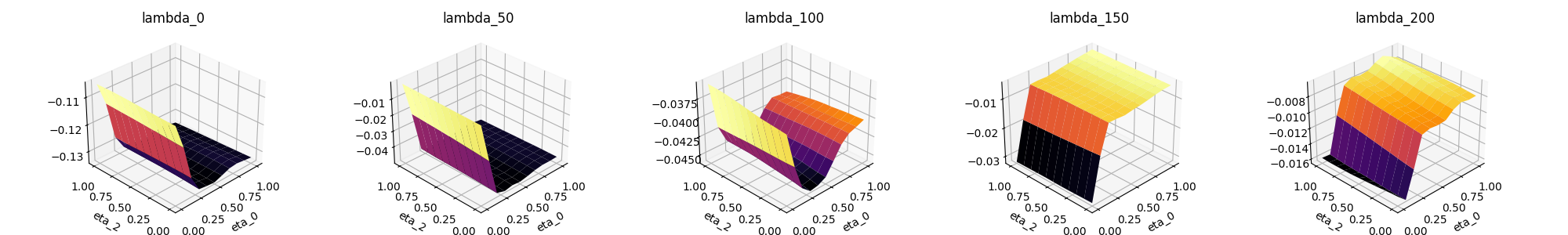}

  \caption[Random Effects model VMP map]{
  Trained \acrshort*{vmp}-map $\lambda^*(\eta)=f_{\alpha^*}(\eta)$ for selected components of the variational parameter vector $\lambda^*_i(\eta),\ i=0,50,100,150,200$ in a \acrlong*{nsf}.
  Here $\eta = (\eta_{1},\eta_{2},\ldots,\eta_{30})$, and we vary two selected $\eta$-components at a time, while keeping the rest constant.
  Top: vary $(\eta_1,\eta_2)\in[0,1]$ (associated with the misspecified groups) and set $\eta_3=\ldots=\eta_{30}=1$.
  Bottom: vary $(\eta_1,\eta_3)\in[0,1]$ (one misspecified group, one well-specified) and set $\eta_2=0$, $\eta_4=\ldots=\eta_{30}=1$.
  }
  \label{fig:rnd_eff_vmp_map}
\end{figure}

\begin{figure}[!htb]
\includegraphics[width=0.98\textwidth,trim={0 0 0 1.5cm},clip]{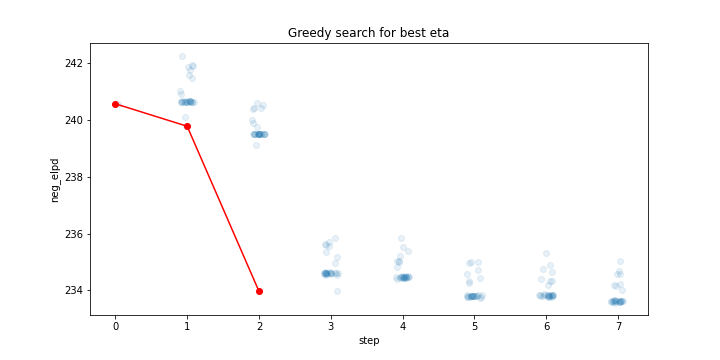}
\caption[Random Effects model VMP map]{
  Greedy search algorithm to initialize \acrshort*{sgd} $\eta^*$-estimation for the Random Effects model. The red line is the trajectory of the negative \acrshort*{elpd} obtained by applying the initialisation algorithm in \cref{subsec:best_smi_vmp}. The blue dots represent the $-\elpd$ of the ``candidate'' models at each step for the greedy search, produced by cutting one (additional) module. The search stops after two iterations, cutting the first and second modules, as expected. Each full step takes a fraction of second to compute.
  }
\label{fig:rnd_eff_greedy_search}
\end{figure}

\begin{figure}[!htb]
  \centering
  \includegraphics[width=0.75\textwidth,trim={10cm 0 0 0},clip]{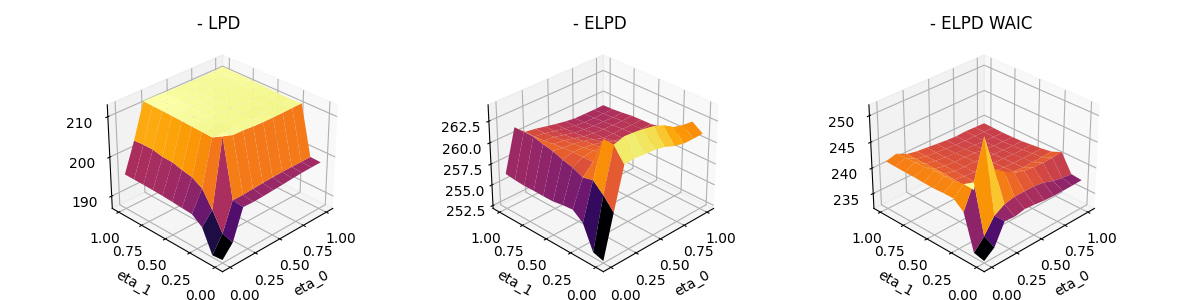}
  \includegraphics[width=0.75\textwidth,trim={10cm 0 0 1cm},clip]{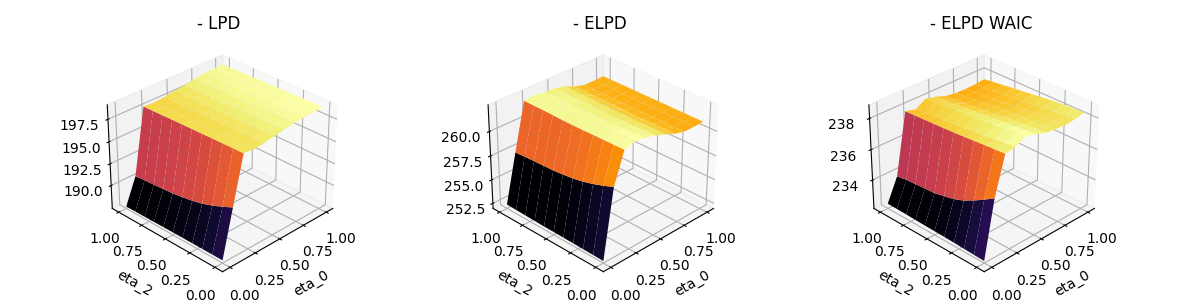}

  \caption[Random Effects model ELPD]{
    Predictive scores comparing performance of the \acrshort*{smi} posterior for different values of $\eta$ in the Random Effects model.
    We show two approximations to the negative ELPD. 
    (Left column) \acrshort*{elpd} estimated using draws from the true generative model.
    (Right columns) \acrshort*{elpd} estimated using the WAIC. These match the left column which acts as a check. 
    (Top row) varying the influence parameters associated with the two misspecified groups, $(\eta_1,\eta_2)\in[0,1]$, fixing $\eta_2=\ldots=\eta_{30}=1$. (Bottom row) varying $(\eta_1,\eta_3)\in[0,1]$, the rates for one misspecified and one well-specified group, fixing $\eta_2=0$ and $\eta_4=...=\eta_{30}=1$.
  }
  \label{fig:rnd_eff_elpd}
\end{figure}

\section{Conclusions}

We have given variational families and loss functions for approximating Cut- and \acrshort*{smi}-posterior distributions. Much of the presentation is agnostic to the details of the variational family. However, we focus on parameterisation based on \acrshortpl*{nf} as this overcomes many well known weaknesses of variational inference. We saw no sign of underdispersion relative to an \acrshort*{mcmc}-baseline in our examples. In contrast \acrshort*{mfvi} approximated the $\eta$-dependent mean of the $p_{\smi,\eta}$ target well but was significantly underdispersed.

The loss function we use (really two loss functions) in \cref{defn:var-smi} is not the standard \acrshort*{kl} divergence between the variational density and target, as that loss does not allow control of information flow between modules. Our loss removes dependence on cut modules when we impute parameters in well-specified modules at $\eta=0$ (the Cut-posterior). Just as $p_{\smi,\eta}$ interpolates between the Cut and Bayes posterior distributions, so our variational approximation $q_{\lambda^*(\eta)}$ exactly interpolates between a variational approximation to the Cut due to \cite{Yu2021variationalcut} and standard variational Bayes. Although the optimized loss need not decrease as we enlarge the variational family, it goes to zero, and we recover the exact target, as the family expands to include the target.

In variational \acrshort*{smi} our goal is to approximate distributions in the family $\mathcal{P}_{\smi}=\{p_{\smi,\eta},\ \eta\in H\}$. We gave a \acrlong*{vmp} $\mathcal{Q}_{H,\alpha}=\{q_{\alpha,\eta},\ \eta\in H\}$ which fits all the distributions in $\mathcal{P}_{\smi}$ at the same time, by taking $\eta$ as a conditioned quantity in the \acrshort*{nf}. We called the modified \acrshort*{nf} the \acrshort*{vmp}-flow. The \acrshort*{sgd} in \cref{alg:v_meta_smi} finds $\alpha^*$ which fits $\mathcal{Q}_{H,\alpha^*}$ to $\mathcal{P}_{\smi}$ in a single joint optimisation for efficient end-to-end training. We gave two parameterisations of the \acrshort*{vmp}. We favor the \acrshort*{vmp}-flow. Our second parameterisation of the \acrshort*{vmp}, $\mathcal{Q}_{H,\alpha}=\{q_{\lambda(\alpha,\eta)},\ \eta\in H\}$, trains a \acrshort*{vmp}-map $f_{\alpha}(\eta)$ to output the optimal variational parameters $\lambda^*(\eta)$ as functions of $\eta$. The two approaches gave similar variational approximations, but training the \acrshort*{vmp}-flow was faster, as \acrshort*{sgd} converged more steadily and rapidly with less tuning of optimisation hyper-parameters.

One advantage of using the \acrshort*{vmp} is that is allows us to modulate feedback from multiple modules at the same time. When we apply a Cut-posterior we have to pre-identify misspecified modules in order to give the locations of cuts. We gave an analysis in which we cut every data-module separately and estimated the associated influence parameters. This allows us to discover rather than pre-identify the cut-modules. 
Accurate estimation of the \acrshort*{elpd} over $\eta\in H$ calls for \acrshort*{smi} or variational-\acrshort*{smi} posterior samples at 
$C$-exponentially many $\eta$ values
for $C$ cuts, and this is prohibitively expensive using one-$\eta$-at-a-time methods such as \cref{alg:vsmi} or nested-\acrshort*{mcmc} methods. 

The main weaknesses of our methods are first, a certain amount of experimentation was needed to find flow architectures that worked well for our targets. However, having found an architecture (with coupling-layer \acrshort*{mlp}-conditioners and rational spline transformers) that worked well, it worked well for all targets. Further exploration is needed to see if this holds more generally. Tuning of the initialisation and learning rate in \cref{alg:vsmi} and \cref{alg:v_meta_smi} were needed also. Another weakness is that the final step of the analysis, after the \acrshort*{vmp} is trained and we have only to estimate and maximise the \acrshort*{elpd} and select the optimal influence parameters $\eta$, is not straightforward, at least for high dimensional $\eta$. The estimated \acrshort*{elpd} is non-convex in $\eta$. Working with the \acrshort*{vmp} makes this step as easy as possible, as sampling the \acrshort*{vmp} is very fast.



The \acrshort*{vmp} framework may be useful outside \acrlong*{smi}. Approximating a complete family of models indexed by a set of continuous hyperparameters has potential applications in parameter selection for hyperpriors in standard Variational Bayesian inference. The marginal density  $q_{\alpha^*,\eta}(\varphi,\tilde\theta),\ \eta\in H$ gives variational approximations to the power posterior distribution $p_{\pow,\eta}(\varphi,\tilde\theta\mid Y,Z)$ at every $\eta$ at the same time, and this may be all we need if the power posterior rather than \acrshort*{smi} is the target.


\section*{Acknowledgements}
We thank Dennis Prangle, David Nott and Kamélia Daudel for insighful discussions on Variational Methods and Modular Inference.

Research supported with Cloud TPUs from Google's TPU Research Cloud (TRC)

\bibliographystyle{ba}
\bibliography{references}

\clearpage
\newpage

\appendix
\section{Modulating prior feedback}\label{sec:prior_feedback}

The material in this section does not address the main point of this paper, which is to present variational inference for \acrshort*{smi}.
However, the issue we cover here does not seem to have been addressed explicitly in the literature to date, and we need some supporting theory for one of our main examples, in \cref{subsec:exp_rnd_eff}.

In the standard Cut-model setup given above, a likelihood factor which is present in the analysis stage has been removed in the imputation stage.
For example, in \cref{eq:cut_posterior}, the likelihood $p(Y \mid \varphi,\theta)$ is absent in $p(\varphi \mid Z)$ and present in $p(\theta \mid Y,\varphi)$.
In some applications of Cut-posteriors and \acrshort*{smi}, going back to \cite{Liu2009modularization} and \cite{Jacob2017together}, and including the archaeological applications in \cite{Styring2017extensification}, \cite{Carmona2020smi} and \cite{Styring2022urban}, the misspecified module has no associated data. In these examples the feedback within a prior module is cut.
If a Cut is applied to a \textit{prior density} and we simply remove the prior factor at imputation then all that remains in the imputation posterior distribution is the base measure.
We have effectively replaced the imputation prior density with a constant, and this may be inappropriate in some settings.
However, we are free to choose the imputation prior and we should use this freedom.
In this section we discuss Cut-priors as a special case of Cut-posterior inference and point out that the choice of imputation prior must be justified.

In order to see how this works, we take a simply structured example, as we did in \cref{fig:toy_multimodular_model}.
Consider a generative model of the form
\begin{align*}
  \varphi & \sim p(\cdot)                                  \\
  \theta  & \sim p(\cdot \mid \varphi)                     \\
  Y_i     & \sim p(\cdot \mid \varphi,\theta),\ i=1,...,n.
\end{align*}
This is the model represented by the leftmost graph in \cref{fig:cut-prior}.
\begin{figure}
  \centering
  \includegraphics[height=1.5in]{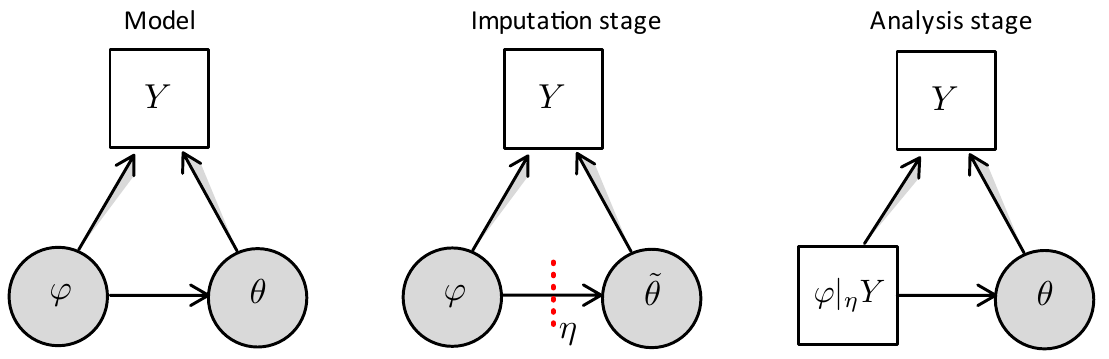}
  \caption{Graphical representation of Cut and \acrshort*{smi}-posterior inference with a Cut prior. Notation as \protect\cref{fig:toy_multimodular_model_2stg}. (Left) Graphical model with prior $p(\theta \mid \varphi)$ indicated by arrow $\varphi\to\theta$.  (Mid) First (imputation) stage, auxiliary variable $\tilde\theta$ introduced for $\varphi$ imputation. Prior $\varphi|\tilde\theta$ modulated as indicated by the red dashed line. (Right) Second (analysis) stage giving inference for $\theta$, now conditioning on the Cut distribution of $\varphi|Y$.}
  \label{fig:cut-prior}
\end{figure}
The posterior is
\begin{equation}\label{eq:bayes_in_cut_prior}
  p(\varphi,\theta \mid Y)\propto p(Y \mid \varphi,\theta)p(\varphi)p(\theta \mid \varphi).
\end{equation}
Suppose we want to cut feedback from $\theta$ into $\varphi$.
In the Cut-posterior for imputation of $\varphi$, the generative model for the data is
\begin{align*}
  \varphi      & \sim p(\cdot)                                         \\
  \tilde\theta & \sim \tilde p(\cdot)                                  \\
  Y_i          & \sim p(\cdot \mid \varphi,\tilde \theta),\ i=1,...,n.
\end{align*}
This model is represented by the middle graph in \cref{fig:cut-prior}.
The new parameter $\tilde\theta$ is an auxiliary variable with Cut prior $\tilde p(\cdot)$.
Some measure for $\tilde\theta$ is needed in the imputation of $\varphi$ as $\tilde\theta$ is unknown and present in $p(Y \mid \varphi,\tilde \theta)$.
Work to date takes volume measure in the parameter space of $\theta$ without generally remarking that a choice had to be made (for example Equation~17 in \cite{Jacob2017together}). Exceptions include \cite{Moss2022}, who take a carefully constructed Cut prior $\tilde p$, and \cite{Styring2017extensification} and \cite{Yu2021variationalcut}, who remark on the choice.
The issue doesn't arise when we remove a likelihood factor, as any parameters enter the likelihood as conditioned variables.
In the analysis stage the generative model is
\begin{align*}
  \theta & \sim p(\cdot \mid \varphi)                     \\
  Y_i    & \sim p(\cdot \mid \varphi,\theta),\ i=1,...,n.
\end{align*}
This model is represented by the rightmost graph in \cref{fig:cut-prior}.
The Cut-posterior is
\begin{equation}\label{eq:cut_posterior_CUTPRIOR}
  p_{\cut}(\varphi,\theta,\tilde\theta \mid Y)=p_{\cut}(\varphi,\tilde\theta \mid Y)p(\theta \mid Y,\varphi),
\end{equation}
where
\begin{equation}\label{eq:CUTPRIOR_cut_imputation_post}
  p_{\cut}(\varphi,\tilde\theta \mid Y)\propto p(\varphi)\tilde p(\tilde\theta)p(Y \mid \varphi,\tilde\theta),
\end{equation}
and
\begin{equation}\label{eq:cut_prior_ThetaAnalysis_post}
  p(\theta \mid Y,\varphi)=p(\theta \mid \varphi)\frac{p(Y \mid \varphi,\theta)}{p(Y \mid \varphi)}.
\end{equation}
The cut prior $\tilde p(\tilde\theta)$ used for imputation in \cref{eq:CUTPRIOR_cut_imputation_post} is a modelling choice, like the Bayes prior $p(\theta \mid \varphi)$ which appears in the posterior.
Typically $p(\theta \mid \varphi)$ is a Subjective Bayes prior expressing the relations linking $\theta$ and $\varphi$ available from physical considerations, but is misspecified, and $\tilde p(\tilde\theta)$ will typically be a non-informative Objective Bayes prior.

The loss function $l_{\cut}(Y; \varphi,\theta,\tilde\theta, \pi_0)$ for which
\[
  p_{\cut}(\varphi,\theta,\tilde\theta \mid Y)\propto \exp(-l_{\cut}(Y; \varphi,\theta,\tilde\theta, \pi_0))p(\varphi)\tilde p(\tilde\theta)p(\theta \mid \varphi)
\]
is a Gibbs posterior is
\begin{equation}\label{eq:HM_full_cut_loss}
  l_{\cut}(Y; \varphi,\theta,\tilde\theta, \pi_0)=-\log(p(Y \mid \varphi,\tilde\theta))-\log(p(Y \mid \varphi,\theta))+\log(p(Y \mid \varphi)).
\end{equation}
and $\pi_0(\varphi,\theta,\tilde\theta)=p(\varphi)\tilde p(\tilde\theta)p(\theta \mid \varphi)$ is the prior which we must specify as part of the loss as it appears in $p(Y \mid \varphi)$ and changes as belief is updated.

This Cut-posterior belief update which cuts feedback in a prior is order coherent in the sense of \cite{Bissiri2016}.
This is known \citep{Carmona2020smi} for Cut-posteriors with the ``standard'' setup of \cref{sec:mod_bayes}.
However, the Cut-posterior with imputation and analysis priors is qualitatively different.

\begin{proposition}\label{prop:cut_prior_is_OK}
  The Cut posterior in \cref{eq:cut_posterior_CUTPRIOR} with Cut-prior feedback is an order coherent belief update.
\end{proposition}
\begin{proof}\label{proof:cut_prior_is_OK}
  If we split the data $Y=(y^{(1)},Y^{(2)})$ then the imputation prior
  $\pi_0(\varphi,\tilde\theta)=p(\varphi)\tilde p(\tilde\theta)$ is updated to $\pi_1(\varphi,\tilde\theta)=p_{\cut}(\varphi,\tilde\theta \mid Y^{(1)})$ and the analysis prior $\pi_0(\theta \mid \varphi)=p(\theta \mid \varphi)$ is updated to $\pi_1(\theta \mid \varphi)=p(\theta \mid \varphi,Y^{(1)})$.
  The update is order coherent if
  \begin{align*}
    \exp(-l_{\cut}(Y; \varphi,\theta,\tilde\theta, \pi_0))\pi_0(\varphi,\theta,\tilde\theta) = \exp(-l_{\cut}(Y^{(2)}; \varphi,\theta,\tilde\theta, \pi_1))\pi_1(\varphi,\theta,\tilde\theta).
  \end{align*}
  Expanding the RHS using the updated prior $\pi_1$ and applying \cref{eq:HM_full_cut_loss},
  \begin{align*}
    RHS & = \exp(-l_{\cut}(Y^{(2)}; \varphi,\theta,\tilde\theta, \pi_1))p_{\cut}(\varphi,\tilde\theta,\theta \mid Y^{(1)})                                                                                                                                 \\
        & \propto \exp(-l_{\cut}(Y^{(2)}; \varphi,\theta,\tilde\theta, \pi_1))\ \times\ p(\varphi)\tilde p(\tilde\theta)p(Y^{(1)} \mid \varphi,\tilde\theta)\ \times\ p(\theta \mid \varphi)\frac{p(Y^{(1)} \mid \varphi,\theta)}{p(Y^{(1)} \mid \varphi)} \\
        & = \exp(-l_{\cut}(Y^{(2)}; \varphi,\theta,\tilde\theta, \pi_1)-l_{\cut}(Y^{(1)}; \varphi,\theta,\tilde\theta, \pi_0))p(\varphi)\tilde p(\tilde\theta)p(\theta \mid \varphi)                                                                       \\
        & =LHS
  \end{align*}
  where the last step holds so long as
  \[
    l_{\cut}(Y^{(2)}; \varphi,\theta,\tilde\theta, \pi_1)+l_{\cut}(Y^{(1)}; \varphi,\theta,\tilde\theta, \pi_0)=l_{\cut}(Y; \varphi,\theta,\tilde\theta, \pi_0).
  \]
  This is the property \cite{Nicholls2022smi} call prequential additivity. It is easilly verified here.
  By independence $p(Y \mid \varphi,\theta)=p(Y^{(2)} \mid \varphi,\theta)p(Y^{(1)} \mid \varphi,\theta)$ so,
  \begin{align*}
    l_{\cut}(Y^{(2)}; \varphi,\theta,\tilde\theta, \pi_1)+l_{\cut}(Y^{(1)}; \varphi,\theta,\tilde\theta, \pi_0) = &
    -\log(p(Y \mid \varphi,\tilde\theta))-\log(p(Y \mid \varphi,\theta))                                                                                                                   \\
                                                                                                                  & \ +\log(p(Y^{(2)} \mid \varphi,Y^{(1)}))+\log(p(Y^{(1)} \mid \varphi))
  \end{align*}
  with
  \begin{align*}
    \log(p(Y^{(2)} \mid \varphi,Y^{(1)}))+\log(p(Y^{(1)} \mid \varphi)) & = \log\left(\frac{p(Y^{(2)},\varphi,Y^{(1)})p(\varphi)}{p(\varphi,Y^{(1)})p(\varphi)}\right)+\log(p(Y^{(1)} \mid \varphi)) \\
                                                                        & =\log(p(Y \mid \varphi)).
  \end{align*}
\end{proof}

When we do \acrshort*{smi} with modulated priors we interpolate between the Bayes posterior in \cref{eq:bayes_in_cut_prior} and Cut posterior in \cref{eq:cut_posterior_CUTPRIOR}. Let $p_\eta(\tilde\theta \mid \varphi)$ be a family of probability densities indexed by $\eta$ and satisfying
$p_{\eta=0}(\tilde\theta \mid \varphi)=\tilde p(\tilde\theta)$ (modulated prior equals Cut prior) and $p_{\eta=1}(\tilde\theta \mid \varphi)=p(\tilde\theta \mid \varphi)$ (modulated prior equals Bayes prior).
We take
\begin{equation}\label{eq:smi_full_cutprior}
  p_{\smi,\eta}(\varphi,\theta,\tilde\theta \mid Y)=p_{\pow, \eta}(\varphi,\tilde\theta \mid Y)p(\theta \mid Y,\varphi)
\end{equation}
where in this setting with a Cut prior
\[
  p_{\pow, \eta}(\varphi,\tilde\theta \mid Y) \propto p(Y \mid \varphi,\tilde\theta)p(\varphi) p_\eta(\tilde\theta \mid \varphi),
\]
and $p(\theta \mid Y,\varphi)$ is given in \cref{eq:cut_prior_ThetaAnalysis_post}, so the second stage analysis is the Cut posterior analysis.
This ensures $p_{\smi,0}=p_{\cut}$ and $p_{\smi,1}$ gives the Bayes posterior. Taking a normalised family $p_\eta(\tilde\theta \mid \varphi),\ \eta\in [0,1]$ of interpolating priors ensures that the marginal prior for $p(\varphi)$ in the imputation doesn't depend on $\eta$. An un-normalised family such as $p_\eta(\varphi,\tilde\theta)\propto \tilde p(\tilde\theta)^{1-\eta} p(\tilde\theta \mid \varphi)^\eta$ has all the desired interpolating properties, but the marginal $p(\varphi)$ in the imputation stage will then depend on $\eta$.
In some settings (for example when working with normal priors with fixed variance) the two prior parameterisations may be equivalent as $\eta$ scales the variance.

\acrshort*{smi} with modulated prior feedback is an order coherent belief update. We now give \cref{prop:smi_cut_prior_is_OK}, stated in \cref{sec:mod_bayes}.

\noindent{\bf{\Cref{prop:smi_cut_prior_is_OK}}.}
{\it
  The \acrshort*{smi} posterior in \cref{eq:smi_posterior_CUTPRIOR} with cut prior feedback is an order coherent belief update.
}
\begin{proof} \label{proof:smi_cut_prior_is_OK}
  Replace $\tilde p(\tilde\theta)\to p_\eta(\tilde\theta \mid \varphi)$ in the proof of \cref{prop:cut_prior_is_OK}.
\end{proof}
In contrast to the tempered likelihood in \cref{eq:powjoint_toymodel} used in Cut-likelihood \acrshort*{smi} in \cref{eqn:smi_01}, where the tempered likelihood is not normalised over $Y$, the modulated prior $p_\eta(\tilde\theta \mid \varphi)$ is a normalised distribution over $\tilde\theta$. However, \acrshort*{smi} in \cref{eq:smi_full_cutprior} with a modulated prior is not simply Bayesian inference with a revised prior, as the priors in the imputation and analysis stages are not the same (unless $\eta=1$).

%

\newpage
\section{Variational families} \label{sec:variational_families_details}

\subsection{Choice of transformations for SMI} \label{subsec:smi_flow}

The transformations $T_1$ and $T_2$ introduced in \cref{eq:bayes_transform,eq:smi_transform} determine our flow-based variational approximation to the \acrshort*{smi}-posterior. We make a few standard assumptions about these transformations:
\begin{itemize}
  \item They are diffemorphisms, meaning that they must be differentiable, invertible and their inverse must be differentiable (needed as we wish to be able to evaluate, differentiate and sample the density $q_\lambda$),
  \item They map the domain of the base distribution onto the domain of the corresponding model parameter (e.g. if $\varphi$ represents a probability, we need $\varphi=T(\epsilon_1)\in[0,1]$),
  \item They support fast sampling from the flow-based density, and fast computation of the density at the sampled values. This allows efficient computation of the forward transformation and the determinant of its Jacobian matrix.
\end{itemize}
Additionally, there are two attributes of $T_1$ and $T_2$ that are important for an accurate approximation, but not strict requirements, namely:
\begin{itemize}
  \item They are expressive transformations, ideally \acrlongpl*{nf} with the property of being a universal approximator,
  \item $T_2$ is defined as a conditional transformation, so that correlation in the posterior of $\varphi$ and $\theta$ can be represented.
\end{itemize}

These requirements are met by \acrfullpl*{nf}. There are now many ways to define expressive flow-based transformations for $T_1$ and $T_2$ \citep[see][for reviews]{Kobyzev2020normalizing,Papamakarios2021normalizing}.
For our experiments in \cref{sec:experiments}, we defined these transformations by the composition of multiple \emph{coupling} layers
based on the \acrfull*{nsf} \citep{Durkan2019neural}, followed by a last layer that maps to the domain of the model parameters.
In order to allow posterior correlation between $\varphi$ and $\theta$, we defined $T_2$ as a \emph{conditional} transformation, so that the elementwise transformation from $\epsilon_2$ to $\theta$ (and from $\epsilon_3$ to $\tilde\theta$) also depends on $\epsilon_1$.









\subsection{Normalizing Flows.} \label{sec:nf}

\subsubsection{Conditioners and transformers in a general flow}
We now define the maps $T_1: \Re^{p_\phi}\to \Re^{p_\phi}$ and $T_2: \Re^{p_\theta+p_\phi}\to \Re^{p_\theta}$ in the normalising flows in terms of their transformers and conditioners. The material in the section is based on \cite{Papamakarios2021normalizing} and further detail may be found there.
Let $x=(x_1,...,x_p)$ and $v=(v_1,...,v_q)$ be generic real vectors and let \[x'_i=\tau(x_i,h_i)\] transform $x_i$ using parameter vectors $h_i\in \Omega_h,\ i=1,...,p$. Here \[h_i=c_i(x_{<i};v,w_i)\] is the output of a conditioner parameterised by a vector $w_i$ with $c_i$ taking as input the conditioning variables $x_{<i}=(x_1,...,x_{i-1})$ and $v$. Note that the conditioner argument $x_{<i}$ in $c_i$ changes dimension as we step through the components $i=1,...,p$ but all the component updates condition on a common set of shared variables $v$. In our setting the conditioners $c_i: \Re^{i-1+q}\to \Omega_h$ are \acrshortpl*{mlp} and $w_i$ are the weights in the $i$'th net. The transformer is a strictly monotonic function. For example, if it is affine then $h_i=(m_i,s_i)$ and $\tau(x_i,h_i)=m_i+s_ix_i$. We experimented with a range of transformers and settled on \emph{rational-quadratic spline transformers} \citep{Durkan2019neural}.
 
One pass over the flow composes these maps 
\[x_i\leftarrow \tau(x_i,c_i(x_{<i};v,w_i)) \quad \mbox{for $i=1,...,p$}\] to update all components $x\to x'$.
Denote by $g: \Re^{p+q}\to \Re^p$, $x'=g(x;v,w)$ a map formed in this way, with $w$ the set of all parameters present in the composition of maps. We make multiple passes over the components, with independent sets of parameters $w$, permuting the indices in order to get an expressive flow. The auto-regressive dependence of $x'_i$ on $x_{<i}$ and the monotone transformer ensure the map is invertible, \[x=g^{-1}(x';v,w)\] given the values of the shared conditioning variables $v$ and parameters $w$, with a tractable lower triangular Jacobian.

In a variant of this setup we used a composition of \emph{coupling layer conditioners} \citep{Dinh2016realnvp} to define $g$. These have the advantage that both density evaluation and sampling are fast operations, important for fitting the flow to the target, and then sampling the related flow to estimate the utility. In this case we set $x'_i=x_i,\ i=1,...,d$ with $d=\lfloor p/2\rfloor$ and
\[x'_i=\tau(x_i,h_i),\quad\mbox{with}\quad h_i=c_i(x_{\le d};v,w_i)\qquad \mbox{ for $i=d+1,...,p$.}\]
The Jacobian determinant from one application of the map is just $\prod_{i>d} |\partial \tau(x_i,h_i)/\partial x_i|$. Again we compose the maps over different permutations of the indices of $x$ (eight times) so that each entry appears in the conditioner and as output from the transformer to get an overall map $g(x;v,w)$.

\subsubsection{Normalizing flows for variational SMI}
In variational-\acrshort*{smi} in \cref{subsec:vsmi}, the \acrshort{nf} has the form 
\begin{align}
\varphi_{(\lambda_1,\epsilon)}&=T_1(\epsilon_1;\lambda_1)\nonumber\\
&=g(\epsilon_1;\emptyset,\lambda_1)\nonumber\\
\intertext{
as $\lambda_1$ parameterises the conditioners of $\varphi$ with no additional conditioners,}
\theta_{(\lambda_2,\epsilon)}&=T_2(\epsilon_2;\lambda_2,\epsilon_1)\nonumber\\
&=g(\epsilon_2;\epsilon_1,\lambda_2)\nonumber\\
\intertext{as $\theta$ is conditioned on $\epsilon_1$ and hence $\varphi$ and} 
\tilde\theta_{(\lambda_3,\epsilon)}&=T_2(\epsilon_3;\lambda_3,\epsilon_1)\nonumber\\
&=g(\epsilon_3;\epsilon_1,\lambda_3),\nonumber
\intertext{similarly $\tilde\theta$, and together}
T(\epsilon; \lambda)                & = \left( T_1(\epsilon_1;\lambda_1) ,\; T_2(\epsilon_2;\lambda_2, \epsilon_1) ,\; T_2(\epsilon_3;\lambda_3, \epsilon_1) \right). \label{eq:smi_transform_details}
\end{align}
These equations define the maps in \cref{eq:smi_transform}. When we want to train this to approximate $p_{\smi,\eta}$ we simply retrain at each $\eta$ where $q_{\lambda^*(\eta)}$ is needed.

\subsubsection{Parameterising the VMP-map}
In the \acrshort*{vmp}, we consider two ways to parameterise the flow defining the meta-posterior.
In the first, in \cref{eq:vmp-map-transform-variables-alpha123} in \cref{subsec:learning_vmp_map}, we define a function (a \acrshort*{mlp}) \[f_\alpha(\eta)=(f^{(1)}_{\alpha_1}(\eta),f^{(2)}_{\alpha_2})(\eta),f^{(3)}_{\alpha_3}(\eta))\] in which the output components correspond to $(\lambda_1(\eta),\lambda_2(\eta),\lambda_3(\eta))$ at any $\eta\in [0,1]^C$. The corresponding mappings in terms of the \acrshortpl*{nf} and their inputs are
\begin{align}
\varphi_{(\lambda_1(\alpha_1,\eta),\epsilon)}&=T_1(\epsilon_1;\lambda_1(\eta))\nonumber\\
&=g(\epsilon_1;\emptyset,f^{(1)}_{\alpha_1}(\eta))\nonumber\\
\theta_{(\lambda_2(\alpha_2,\eta),\epsilon)}&=T_2(\epsilon_2;\lambda_2(\eta),\epsilon_1)\nonumber\\
&=g(\epsilon_2;\epsilon_1,f^{(2)}_{\alpha_2}(\eta))\nonumber\\
\tilde\theta_{(\lambda_3(\alpha_3,\eta),\epsilon)}&=T_2(\epsilon_3;\lambda_3(\eta),\epsilon_1)\nonumber\\
&=g(\epsilon_3;\epsilon_1,f^{(3)}_{\alpha_3}(\eta)).\nonumber\\
T(\epsilon; \lambda(\eta))                & = \left( T_1(\epsilon_1;\lambda_1(\eta)) ,\; T_2(\epsilon_2;\lambda_2(\eta), \epsilon_1) ,\; T_2(\epsilon_3;\lambda_3(\eta), \epsilon_1) \right). \label{eq:vmp-map-transform-variables-alpha123_details}
\end{align}
This keeps the conditioner-transformer relations unchanged from \cref{eq:smi_transform_details}, and simply injects the right parameters $\lambda$ into the flow to express $p_{\smi,\eta}$ at any particular $\eta$.

A perfectly trained universal \acrshort*{vmp}-map $f_\alpha(\eta)$ would inject an optimal set $\lambda^*(\eta)$ of parameters into the \acrlong*{nf} at every $\eta$. 
We assume $f_{\alpha}(\eta)$ is a continuous function of $\eta$, motivated by the discussion above. We require $f_{\alpha}(\eta)$ to be differentiable almost everywhere in $\alpha$ for the purpose of optimisation. These are minimal assumptions. 
In order to extend the theory from \cref{sec:vi_modular} to this setting, suppose we have $\lambda^*(\eta)\in \Lambda^*$ where $\Lambda^*(\eta)$ is given in \cref{defn:var-smi} (at $\eta$). Suppose $f_\alpha(\eta)$ is continuously differentiable in $\alpha$ at each $\eta$ and there exists $\alpha^*$ such that $f_{\alpha^*}(\eta)=\lambda^*(\eta),\ \eta\in H$. In this case substituting $\lambda=f_{\alpha}(\eta)$ into $\mathcal{L}^{\smi,\eta}$ in \cref{prop:stop-gradient-loss} and minimising over $\alpha$ will give the same $\lambda^*$ with the same properties (P1-3).

Referring to the loss in \cref{eq:vsmi_loss}, let
\[
  A^*=\{\alpha\in A: \mathcal{L}^{(\msmi-map)}(\alpha)=\min_{a\in A} \mathcal{L}^{(\msmi-map)}(a)\}
\]
be the set of optimal $\alpha$-values.
If for some $\alpha^*\in A$, the \acrshort*{vmp}-map $f_{\alpha^*}(\eta)$ expresses the function $\lambda^*(\eta)$ perfectly at $\eta\in\eta_{1:R}$, that is if $f_{\alpha^*}(\eta_r)\in \Lambda^*(\eta_r),\ r=1,...,R$ then
\[
  A^*=\{\alpha\in A: f_{\alpha}(\eta_r)\in \Lambda^*(\eta_r),\ r=1,...,R\}
\]
since these solutions in $\alpha$ and no others minimise $\mathcal{L}^{(\smi,\eta)}$ at every $\eta$ in the sum in \cref{eq:meta_smi_target_loss}.
In this case we recover $\lambda^*\in \Lambda^*$ and (P1-3) at each $\eta\in \eta_{1:R}$.

For our examples in section 5, we found three Multi-Layer Perceptrons (MLPs) in
parallel, trained with input $\eta$ and output $f^{(k)}_{\alpha_k}(\eta),\ k=1,2,3$ expressed the non-linear
relationships holding between variational parameters, and scaled well with increasing
numbers of cuts (dimension of input $\eta \in H$).
We report experiments with  \acrshort*{vmp}-maps parameterised with \acrfullpl*{gp} \citep{Rasmussen2005gp} and Cubic Splines \citep{Hastie2001esl}. These may be convenient when the dimension of $H$ is small.

We found in training $f_{\alpha}$ that $\alpha$ should be initialised to output a constant function $f_{\alpha_0}(\eta)\approx\lambda_0$ independent of $\eta$. Here $\lambda_0$ are variational parameters obtained in a ``pre-training stage'' using \cref{alg:vsmi} for a fixed central value of $\eta$.
This strategy significantly reduces training time and improves convergence.

\subsubsection{Parameterising the VMP-flow}
The other way we parameterise the meta-posterior is to modify the conditioner. Consider the coupling layer setup. In order to express the $\eta$-dependence (in say the coupling layer conditioner), $x'_i=x_i,\ i=1,...,d$ and
\[
x'_i=\tau(x_i,h_i), \quad h_i=c_i(x_{\le d};v,w_i)+c'(\eta;w'_i),\quad \mbox{for $i=d+1,...,p$}.
\]
Here $\eta$ enters via a second additive conditioner $c'$ (which is not indexed by $i$ as it always takes the same conditioning variable $\eta$ but has parameters $w'_i$ which vary across variables $x_i,\ i>d$). Like $v$, $\eta$ is a common conditioner in every transform. When we compose this map 
to form the overall map, $g(x;(v,\eta),(w,w'))$ say, \emph{each} application of the map has its own sets of parameters $w_i$ and $w'_i$ \emph{for each} $i=d+1,...,p$. When we parameterise the meta-posterior we write $\alpha_k=(\lambda_k,\mu_k),\ k=1,2,3$ and
\begin{align}
    \varphi_{(\alpha_1,\eta,\epsilon)}&=T_1(\epsilon_1;\alpha_1,\eta)\nonumber\\
    &=g(\epsilon_1;(\emptyset,\eta),(\lambda_1,\mu_1)),\label{eq:meta_new_map_phi}\\
    \theta_{(\alpha_2,\eta,\epsilon)}&=T_2(\epsilon_2;\alpha_2, (\eta,\epsilon_1))\nonumber\\
    &=g(\epsilon_2;(\epsilon_1,\eta),(\lambda_2,\mu_2)),\label{eq:meta_new_map_theta}\\
    \tilde\theta_{(\alpha_3,\eta,\epsilon)}&=T_2(\epsilon_3;\alpha_3,(\eta,\epsilon_1))\nonumber\\
    &=g(\epsilon_3;(\epsilon_1,\eta),(\lambda_3,\mu_3)),\label{eq:meta_new_map_theta_tilde}\\
    T(\epsilon;\alpha,\eta) & = \left( T_1(\epsilon_1;\alpha_1,\eta) ,\; T_2(\epsilon_2;\alpha_2, (\eta,\epsilon_1)) ,\; T_2(\epsilon_3;\alpha_3,(\eta,\epsilon_1)) \right).
\end{align}
so that $\lambda=(\lambda_1,\lambda_2,\lambda_3)$ contain the $w$-parameters of the ``old'' conditioner $c_i$ in each of the flows for $\varphi,\theta$ and $\tilde\theta$ respectively and $\mu=(\mu_1,\mu_2,\mu_3)$ are the corresponding $w'$-parameters of the new $\eta$-dependent conditioner $c'$. These are collectively  the \acrshort{vmp} parameters $\alpha=(\alpha_1,\alpha_2,\alpha_3)$.
We take $\mu\in M$ and $\alpha\in A$ with $A=\Lambda\times M$. The influence parameter $\eta$ enters through the conditioner so it is associated with $\epsilon_1$ as a conditioner for $\theta$ and $\tilde\theta$ and as a conditioner on its own for $\varphi$.

The family of variational densities $q_{\alpha,\eta}(\varphi,\theta,\tilde\theta)$ defined by the \acrshort*{vmp}-flow are given in \cref{eq:smi_transform_new_vmp}. The optimal parameters $\alpha^*=(\lambda^*,\mu^*)$, $\alpha^*\in A$, minimise a loss $\mathcal{L}^{(\msmi-flow)}(\alpha)$ which is closely related to $\mathcal{L}^{(\msmi-map)}(\alpha)$ in \cref{eq:meta_smi_target_loss}. In more detail,
\begin{align} \label{eq:vsmi_loss-meta-new}
  \mathcal{L}^{( \msmi-flow, \eta)}(\alpha) = \elbo_{\pow, \eta}(\alpha_1,\alpha_3) + \elbo_{\bayes \cancel{\nabla}(\varphi)}(\alpha_1,\alpha_2)
\end{align}
where
\begin{align}
  \elbo_{\pow, \eta}(\alpha_1,\alpha_3) = \E_{(\varphi,\tilde\theta)\sim q_{\alpha_1,\alpha_3,\eta}}[                & \log p_{\pow, \eta}(\varphi, \tilde\theta, Z, Y) - \log q_{\alpha_1,\alpha_3,\eta}(\varphi, \tilde\theta) ] \label{eq:elbo_modular_pow_meta_new}                  \\
  \elbo_{\bayes \cancel{\nabla}(\varphi)}(\alpha_1,\alpha_2,\eta) = \E_{(\varphi,\theta)\sim q_{\alpha_1,\alpha_2,\eta}}[ & \log p( \cancel{\nabla}(\varphi), \theta, Z, Y) - \log q_{\alpha_1,\alpha_2,\eta}(\cancel{\nabla}(\varphi), \theta) ]. \label{eq:elbo_modular_stop_grad_meta_new}
\end{align}
The loss in \cref{eq:meta_smi_target_loss} given for the map becomes for the flow,
\[
 \mathcal{L}^{(\msmi-flow)}(\alpha) =\E_{\eta\sim\rho}\left(\mathcal{L}^{(\msmi-flow,\eta)}(\alpha)\right).
\]
In order to apply the reparameterisation trick, the expressions for $\varphi_{(\alpha_1,\eta,\epsilon)}, \theta_{(\alpha_2,\eta,\epsilon)}$ and $\tilde\theta_{(\alpha_3,\eta,\epsilon)}$ in \cref{eq:meta_new_map_phi}-\cref{eq:meta_new_map_theta_tilde} are substituted into \cref{eq:elbo_modular_pow_meta_new} and \cref{eq:elbo_modular_stop_grad_meta_new} so that gradients in $(\alpha_1,\alpha_3)$ and $\alpha_2$ can be taken inside the expectation which is now over $\epsilon\sim p(\cdot)$.

\newpage
\section{Optimisation of the VMP loss}\label{sec:sgd-for-meta-losses}

Here we give the algorithm we use to estimate the the parameters $\alpha$ of the \acrshort*{vmp}-map and -flow. \Cref{alg:v_meta_smi} gives \acrshort*{sgd} for the \acrshort*{vmp}-map loss $\mathcal{L}^{(\msmi-map)}(\alpha)$ (notation $\mathcal{L}^{(\msmi)}(\alpha)$ as the algorithm for optimisation of $\mathcal{L}^{(\msmi-flow)}(\alpha)$ is similar).

\begin{algorithm}[tbh]
  \caption{Variational Meta-Posterior approximation for $\mathcal{P}_{\smi}=\{p_{\smi,\eta},\ \eta\in H\}$} \label{alg:v_meta_smi}
  \begin{algorithmic}
    \STATE \textbf{Input:} $\mathcal{D}$: Data. $p(\varphi,\theta,\mathcal{D})$: Multi-modular probabilistic model. $q_\lambda=(p(\epsilon), T, \lambda)$: Variational family. $f_{\alpha}(\eta)$: \acrshort*{vmp}-map. $\rho$: $\eta$-weighting distribution over $H$. \\[0.1in]
    \STATE \textbf{Output:} \acrshort{vmp}-map $f_{\hat\alpha}(\eta)$ giving the Variational approximation for the $\mathcal{P}_{\smi}$ family.\\[0.1in]

    \STATE Initialise mapping parameters $\alpha$
    \WHILE{\acrshort*{sgd} not converged}
    \STATE Sample $\mathcal{D}^{(b)} \sim \mathcal{D}$ (random minibatch of data).
    \STATE Sample $R$ values of $\eta_{r}\sim \rho,\ r=1,...,R$.
    \FOR{$r = 1,\ldots,R$}
    \STATE Obtain variational parameters $\lambda_r = f_{\alpha}(\eta_r)$.
    \FOR{$s = 1,\ldots,S$}
    \STATE Sample the base distribution, $\epsilon_{r,s} \sim p(\cdot)$.
    \STATE Transform the sampled values $(\varphi_{r,s}, \theta_{r,s}, \tilde\theta_{r,s}) \leftarrow T_{\lambda_r}(\epsilon_{r,s})$ as in \cref{eq:smi_transform}.
    \ENDFOR
    \ENDFOR
    \STATE Compute the Monte Carlo estimate of the loss $\mathcal{L}^{( \msmi)}$ in \cref{eq:meta_smi_target_loss} and its gradients.
    \begin{equation}
      \widehat{\mathcal{L}}^{(\msmi)} = \widehat{\elbo}_{\pow}^{(\msmi)} + \widehat{\elbo}_{\cancel{\nabla}(\varphi)}^{(\msmi)}
    \end{equation}
    where
    \begin{align}
      \widehat{\elbo}_{\pow}^{(\msmi)}                     & = - \frac{1}{RS} \sum_{r=1}^{R}\sum_{s=1}^{S} \left[ \log p_{\pow, \eta_r}(\varphi_{r,s}, \tilde\theta_{r,s}, \mathcal{D}^{(b)}) - \log q(\varphi_{r,s}, \tilde\theta_{r,s}) \right]         \\
      \widehat{\elbo}_{\cancel{\nabla}(\varphi)}^{(\msmi)} & = - \frac{1}{RS} \sum_{r=1}^{R}\sum_{s=1}^{S} \left[ \log p(\cancel{\nabla}(\varphi_{r,s}), \theta_{r,s}, \mathcal{D}^{(b)}) - \log q( \cancel{\nabla}(\varphi_{r,s}), \theta_{r,s}) \right]
    \end{align}
    \STATE Update $\alpha$ using the estimated gradient vector $\nabla_{\alpha}\widehat{\mathcal{L}}^{(\msmi)}$
    \STATE Check convergence of $q_{f_{\alpha}(\eta)}(\varphi,\theta)$ for multiple $\eta \in H$
    \ENDWHILE

    \RETURN $\hat\alpha=\alpha$ 

  \end{algorithmic}
\end{algorithm}

\section{Detailed derivation of Variational Modular posteriors.}

\subsection{Proofs for variational SMI properties} \label{sec:vsmi_detail}

\noindent{\bf{\Cref{prop:var_smi_distance}.}}
{\it
  The divergence defined in \cref{eqn:distance_to_set} can be written
  \[
    d(q_{\lambda},\mathcal{F}_{\smi, \eta})=\kl{q_{\lambda_1, \lambda_3}(\varphi, \tilde\theta)}{p_{\pow, \eta}(\varphi,\tilde\theta \mid Y,Z)},
  \]
  and hence does not depend on $\lambda_2$.
}
\begin{proof}\label{proof:var_smi_distance}
  Since $\tilde q\in \mathcal{F}_{\smi, \eta}$ it can be written
  $\tilde q(\varphi,\theta,\tilde\theta)=p_{\pow, \eta}(\varphi,\tilde\theta \mid Y,Z)\tilde q(\theta\mid \varphi,\tilde\theta)$
  where $\tilde q(\theta\mid \varphi,\tilde\theta)$ is an arbitrary conditional density. The divergence is
  \begin{align*}
    \min_{\tilde q\in \mathcal{F}_{\smi, \eta}}\kl{q_{\lambda}}{\tilde q} & =\min_{\tilde q(\theta\mid \varphi,\tilde\theta)}\kl{q_{\lambda_1,\lambda_3}(\varphi,\tilde\theta)q_{\lambda_2}(\theta\mid \varphi)}{p_{\pow, \eta}(\varphi,\tilde\theta \mid Y,Z)\tilde q(\theta\mid \varphi,\tilde\theta)}
    \\
                                                                          & =\kl{q_{\lambda_1, \lambda_3}(\varphi, \tilde\theta)}{p_{\pow, \eta}(\varphi,\tilde\theta \mid Y,Z)}
    \\
                                                                          & \qquad +\quad \min_{\tilde q(\theta\mid \varphi,\tilde\theta)} E_{\varphi\sim q_{\lambda_1}}[\kl{q_{\lambda_2}(\theta\mid \varphi)}{\tilde q(\theta\mid \varphi,\tilde\theta)}]
    \\
                                                                          & =\kl{q_{\lambda_1, \lambda_3}(\varphi, \tilde\theta)}{p_{\pow, \eta}(\varphi,\tilde\theta \mid Y,Z)}
  \end{align*}
  as the variation over $q(\theta\mid \varphi,\tilde\theta)$ is over \emph{all} conditional densities with optimum
  $\tilde q(\theta\mid \varphi,\tilde\theta)=q_{\lambda_2}(\theta\mid \varphi)$ since the expectation is non-negative and zero in that case.
\end{proof}


\noindent{\bf{\Cref{prop:var_smi_is_cut_at_eta0}.}}
{\it
Variational \acrshort*{smi} satisfies property (P1) at $\eta=0$: If the set
\[
  \Lambda^*_{(3)}=\{\lambda_3\in \Lambda_2: q_{\lambda^*_3}(\tilde\theta\mid \varphi)=p(\tilde\theta\mid \varphi)\}
\]
is non-empty and we set
\[
  \Lambda^*_{(1)} =   \{\lambda_1\in \Lambda_1:
  \kl{q_{\lambda_1}(\varphi)}{p(\varphi \mid Z)}=d^*_{\cut}\}
\]
with $d^*_{\cut}$ defined in (P1) then $\Lambda^*_{(1,3)}$ defined in \cref{eqn:lambda13-star-defn} satisfies
\[
  \Lambda^*_{(1,3)} = \Lambda^*_{(1)} \times \Lambda^*_{(3)},
\]
so $q_{\lambda^*_1}$ does not depend in any way on $p(Y \mid \varphi,\theta)$ or $p(\theta \mid \varphi)$ at $\eta=0$.
}
\begin{proof}\label{proof:var_smi_is_cut_at_eta0}
  At $\eta=0$, $p_{\pow, \eta=0}(\varphi,\tilde\theta \mid Y,Z)=p(\varphi \mid Z)\,p(\tilde\theta\mid \varphi)$ in \cref{prop:var_smi_distance}. By \cref{eqn:lambda13-star-defn},
  \begin{align*}
    d^*_{\smi} & =\min_{(\lambda_1,\lambda_3)\in B}\kl{q_{\lambda_1,\lambda_3}}{p(\varphi \mid Z)\,p(\theta\mid \varphi)}
    \\
               & =\min_{(\lambda_1,\lambda_3)\in B} \left(\kl{q_{\lambda_1}(\varphi)}{p(\varphi \mid Z)}+E_{\varphi\sim q_{\lambda_1}}[\kl{q_{\lambda_3}(\tilde\theta\mid \varphi)}{p(\tilde\theta\mid \varphi)}]\right) \\
               & =\min_{\lambda_1\in \Lambda_1} \kl{q_{\lambda_1}(\varphi)}{p(\varphi \mid Z)}                                                                                                                           \\
               & =d^*_{cut}
  \end{align*}
  as the expectation is zero (and hence minimised for every argument) when $\lambda^*_3\in \Lambda_2$ satisfies $q_{\lambda^*_3}(\tilde\theta\mid \varphi)=p(\tilde\theta\mid \varphi)$. Such a $\lambda^*_3$ exists because $\Lambda^*_{(3)}$ is non-empty by the assumption in the proposition. It follows that $\lambda$ minimises $d(q_\lambda,\mathcal{F}_{\smi, \eta})$ at $\eta=0$ if and only if $(\lambda^*_1,\lambda_3)\in\Lambda^*_{(1)}\times \Lambda^*_{(3)}$.
\end{proof}

\noindent{\bf{\Cref{prop:smi-show-p2}.}}
{\it
  Variational \acrshort*{smi} satisfies property (P2). Let
  \[
    \Lambda^*_{(1)}=\bigcup_{(\lambda^*_1,\lambda^*_3)\in \Lambda^*_{(1,3)}} \{\lambda_1^*\}.
  \]
  The set of Bayes and \acrshort*{smi} variational posteriors for $\varphi,\theta$ are the same, that is,
  \[
    \bigcup_{\lambda_1^*\in\Lambda^*_{(1)}}\bigcup_{\lambda^*_2\in \Lambda_{(2)}^*(\lambda^*_1)} \{(\lambda^*_1,\lambda^*_2)\}=B^*,
  \]
  when $\eta=1$.
}
\begin{proof}\label{proof:smi-show-p2}
  When $\eta=1$ the power posterior for $\varphi,\tilde\theta$ is the Bayes posterior,
  \[
    p_{\pow, \eta=1}(\varphi,\tilde\theta \mid Y,Z)=p(\varphi \mid Y,Z)\,p(\tilde\theta \mid Y,\varphi)
  \]
  so \cref{eqn:lambda13-star-defn} is the same as variational Bayes as determined by \cref{eqn:variational_bayes_standard}. Since $q_{\lambda_3}(\tilde\theta \mid \varphi)$ and $q_{\beta_2}(\theta \mid \varphi)$ have the same parameterisation, we have $B^*=\Lambda^*_{(1,3)}$. For any fixed $\lambda^*_1\in \Lambda^*_{(1)}$ let
  \[
    \Lambda^*_{(3)}(\lambda^*_1)=\{\lambda_3\in \Lambda_2: \kl{q_{\lambda^*_1, \lambda_3}(\varphi, \tilde\theta)}{p_{\pow, \eta}(\varphi,\tilde\theta \mid Y,Z)}=d^*_{\smi}\}
  \]
  so that
  \begin{equation}
    \Lambda^*_{(1,3)}=\bigcup_{\lambda_1\in \Lambda^*_{(1)}}\bigcup_{\lambda_3\in \Lambda^*_{(3)}(\lambda_1)} \{(\lambda_1,\lambda_3)\}\label{eqn:proof-p2-lstar13}
  \end{equation}
  Fixing $\lambda^*_1\in \Lambda^*_{(1)}$, and taking $\eta=1$,
  \begin{align*}
    \kl{q_{\lambda^*_1, \lambda_3}(\varphi, \tilde\theta)}{p_{\pow, \eta}(\varphi,\tilde\theta \mid Y,Z)} & =\kl{q_{\lambda^*_1}}{p(\varphi \mid Y,Z)} \\ &\qquad +\quad E_{\varphi\sim q_{\lambda^*_1}}[\kl{q_{\lambda_3}(\tilde\theta\mid \varphi)}{p(\tilde\theta \mid Y,\varphi)}]
  \end{align*}
  and the target level is
  \begin{align*}
    d^*_{\smi} & =\kl{q_{\lambda^*_1}}{p(\varphi \mid Y,Z)}+ \min_{\lambda_3\in \Lambda_2} E_{\varphi\sim q_{\lambda^*_1}}[\kl{q_{\lambda_3}(\tilde\theta\mid \varphi)}{p(\tilde\theta \mid Y,\varphi)}] \\
               & =\kl{q_{\lambda^*_1}}{p(\varphi \mid Y,Z)}+D^*_{\smi}(\lambda^*_1),
  \end{align*}
  so cancelling the common $\kl{q_{\lambda^*_1}}{p(\varphi \mid Y,Z)}$ term,
  \begin{align}
    \Lambda^*_{(3)}(\lambda^*_1) & = \{\lambda_3\in \Lambda_2: E_{\varphi\sim q_{\lambda^*_1}}[\kl{q_{\lambda_3}(\tilde\theta\mid \varphi)}{p(\tilde\theta \mid Y,\varphi)}]=D^*_{\smi}(\lambda^*_1)\}\nonumber \\
                                 & =\Lambda^*_{(2)}(\lambda^*_1)\label{eqn:proof-p2-l2},
  \end{align}
  from \cref{eqn:lambda2-star-defn-equiv}, so at $\eta=1$,
  \[
    \bigcup_{\lambda_1\in \Lambda^*_{(1)}}\bigcup_{\lambda_2\in \Lambda^*_{(2)}(\lambda_1)} \{(\lambda_1,\lambda_2)\}=\bigcup_{\lambda_1\in \Lambda^*_{(1)}}\bigcup_{\lambda_3\in \Lambda^*_{(3)}(\lambda_1)} \{(\lambda_1,\lambda_3)\}
  \]
  by \cref{eqn:proof-p2-l2}, so the LHS is equal $\Lambda^*_{(1,3)}$
  by \cref{eqn:proof-p2-lstar13} and we saw that $\Lambda^*_{(1,3)}=B^*$, giving the set relation claimed in the proposition.
  We conclude from this that at $\eta=1$, $q_{\lambda^*_1,\lambda^*_2}(\varphi,\theta)$ is a marginal variational \acrshort*{smi} posterior if and only if  it is also a variational Bayes posterior.
\end{proof}


\noindent{\bf{\Cref{prop:stop-gradient-loss}.}}
{\it
  The set $\Lambda^*$ in \cref{defn:var-smi} is the set of solutions of $\nabla_{\lambda} \mathcal{L}^{( \smi, \eta)} = 0$ corresponding to minima.
}
\begin{proof}\label{proof:stop-gradient-loss}
  Assuming the flow-based construction of the variational family $q_{\lambda}$ defined in \cref{eqn:q-lambda-smi-var}, the (stopped) gradients of $\mathcal{L}^{( \smi, \eta)}$ are
  \begin{align}
    \nabla_{\lambda_1} \mathcal{L}^{( \smi, \eta)} = \E_{\epsilon \sim p(\epsilon)}[ & \nabla_{\varphi} \left\{ \log p(Z \mid \varphi) + \eta \log p(Y \mid \varphi, \tilde\theta) + \log p(\varphi) \right\} \nabla_{\lambda_1} \{ \varphi \} \nonumber
    \\
                                                                                     & + \nabla_{\lambda_1} \log \left\vert J_{T_1}(\epsilon_1) \right\vert  ] \label{eq:grad_smi_loss_l1}
    \\
    \nabla_{\lambda_2} \mathcal{L}^{( \smi, \eta)} = \E_{\epsilon \sim p(\epsilon)}[ & \nabla_{\theta} \left\{ \log p(Y \mid \varphi, \theta) + \log p(\theta \mid \varphi) \right\} \nabla_{\lambda_2} \{ \theta \} \nonumber
    \\
                                                                                     & + \nabla_{\lambda_2} \log \left\vert J_{T_2}(\epsilon_1,\epsilon_2) \right\vert  ] \label{eq:grad_smi_loss_l2}
    \\
    \nabla_{\lambda_3} \mathcal{L}^{( \smi, \eta)} = \E_{\epsilon \sim p(\epsilon)}[ & \nabla_{\tilde\theta} \left\{ \eta \log p(Y \mid \varphi, \tilde\theta) + \log p(\tilde\theta \mid \varphi) \right\} \nabla_{\lambda_3} \{ \tilde\theta \} \nonumber
    \\
                                                                                     & + \nabla_{\lambda_3} \log \left\vert J_{T_2}(\epsilon_1,\epsilon_3) \right\vert  ]. \label{eq:grad_smi_loss_l3}
  \end{align}

  The system of equations determined by $\nabla_{\lambda} \mathcal{L}^{( \smi, \eta)} = 0$ is the same system defined in \cref{remk:var-smi-roots-gradients} for $\Lambda^*$, namely, \cref{eq:vsmi_roots_l1,eq:vsmi_roots_l2,eq:vsmi_roots_l3}.
\end{proof}

\noindent{\bf{\Cref{prop:smi-show-p3-extra}.}}
{\it
  Let $\mathcal{L}^*(v)=\min_{\lambda\in\Lambda}\mathcal{L}^{(v)}(\lambda)$
  and
  \[
    \Lambda^*(v)=\{\lambda\in \Lambda: \mathcal{L}^{(v)}(\lambda)=\mathcal{L}^*(v)\}.
  \]
  Under regularity conditions on $\mathcal{F}_{smi,\eta}$ and $p_{ \smi, \eta}$ given in \cref{prop:use_IFT_show_loss_limit}, for every solution $\lambda^*\in \Lambda^*$ in \cref{defn:var-smi} and all sufficiently small $v\ge 0$ there exists a unique continuous function $\lambda^*(v)$ satisfying $\lambda^*(v)\in \Lambda^*(v)$ and \[\lim_{v\to 0}\lambda^*(v)=\lambda^*.\]
}
\begin{proof}\label{proof:smi-show-p3-extra}
  Take the definition of $\mathcal{L}^{(v)}(\lambda)$ in \cref{eqn:loss_weighted_var_smi}, use \cref{prop:var_smi_distance} to replace $d$ and \cref{eqn:naive_VI_KL_expand} to expand $\kl{q_{\lambda}}{p_{\smi, \eta}}$. This gives
  \begin{align*}
    \mathcal{L}^{(v)}(\lambda)
     & =(1+v)\kl{q_{\lambda_1, \lambda_3}(\varphi, \tilde\theta)}{p_{\pow, \eta}(\varphi,\tilde\theta \mid Y,Z)}
    \\
     & \qquad+\quad v\cdot E_{\varphi\sim q_{\lambda_1}}[\kl{q_{\lambda_2}(\theta\mid \varphi)}{p(\theta \mid Y,\varphi)}]
  \end{align*}
  Minima $\lambda^{(v)}\in\Lambda^*(v)$ are stationary points of $\mathcal{L}^{(v)}$ so they solve $\nabla_{\lambda}\mathcal{L}^{(v)}(\lambda)=0$ with positive curvature. Using the reparameterisation trick these equations are, for $v>0$,
  \begin{align}
    0 = \E_{\epsilon \sim p(\epsilon)}[ & (1+v)\nabla_{\varphi} \left\{ \log p(Z \mid \varphi) + \eta \log p(Y \mid \varphi, \tilde\theta) + \log p(\varphi) \right\} \nabla_{\lambda_1} \{ \varphi \} \nonumber
    \\
                                        & + v\nabla_{\varphi} \left\{ \log p(Y \mid \varphi, \theta) - \log p(Y\mid \varphi) \right\} \nabla_{\lambda_1} \{ \varphi \} \nonumber                                 \\
                                        & + \nabla_{\lambda_1} \log \left\vert J_{T_1}(\epsilon_1) \right\vert  ] \label{eq:root_loss_l1_smi_general}                                                            \\
    0 = \E_{\epsilon \sim p(\epsilon)}[ & \nabla_{\theta} \left\{ \log p(Y \mid \varphi, \theta) + \log p(\theta \mid \varphi) \right\} \nabla_{\lambda_2} \{ \theta \} \nonumber                                \\
                                        & + \nabla_{\lambda_2} \log \left\vert J_{T_2}(\epsilon_1,\epsilon_2) \right\vert  ] \label{eq:root_loss_l2_smi_general}                                                 \\
    0 = \E_{\epsilon \sim p(\epsilon)}[ & \nabla_{\tilde\theta} \left\{ \eta \log p(Y \mid \varphi, \tilde\theta) + \log p(\tilde\theta \mid \varphi) \right\} \nabla_{\lambda_3} \{ \tilde\theta \} \nonumber   \\
                                        & + \nabla_{\lambda_3} \log \left\vert J_{T_2}(\epsilon_1,\epsilon_3) \right\vert  ] \label{eq:root_loss_l3_smi_general}
  \end{align}
  The system of equations \cref{eq:root_loss_l1_smi_general,eq:root_loss_l2_smi_general,eq:root_loss_l3_smi_general} converges to the system of equations \cref{eqn:lam-star-roots1a,eqn:lam-star-roots1b} as $v\to 0$. Under regularity conditions set out in \cref{prop:use_IFT_show_loss_limit} below, the solutions converge, in the sense that every point in $\Lambda^*$ is the limit as $v\to 0$ of some continuous sequence of solutions to $\nabla_\lambda\mathcal{L}^{(v)}=0$.
\end{proof}

We now state the regularity conditions and show solutions converge. Recall that $\Lambda_1=\Re^{L_1}$ and $\Lambda_2=\Re^{L_2}$ with $L=L_1+2L_2$ so that $\dim(\Lambda)=L$. The system of equations \cref{eq:root_loss_l1_smi_general,eq:root_loss_l2_smi_general,eq:root_loss_l3_smi_general} has the form $F(\lambda,v)=\mathbf{0}_L$ where
\[
  F(\lambda,v)=f(\lambda)+v\cdot g(\lambda)
\]
with $f,g:\Re^L\to \Re^L$ and $F:\Re^{L+1}\to \Re^L$. Equations \cref{eq:vsmi_roots_l1,eq:vsmi_roots_l2,eq:vsmi_roots_l3} are equivalent to $F(\lambda,0)=\mathbf{0}_L$ which is just $f(\lambda)=\mathbf{0}_L$. Now \cref{eq:vsmi_roots_l1,eq:vsmi_roots_l3} are the reparameterisation of \cref{eqn:lam-star-roots1a} and \cref{eq:vsmi_roots_l2} is the reparameterisation of \cref{eqn:lam-star-roots1b} so we identify
$f=(f_1,f_2,f_3)$ with
\begin{align*}
  f_1(\lambda_1,\lambda_3) & =\nabla_{\lambda_1}\kl{q_{\lambda_1,\lambda_3}}{p_{\pow, \eta}}                                                    \\
  f_2(\lambda_1,\lambda_2) & =\nabla_{\lambda_2} E_{\varphi\sim q_{\lambda_1}}\kl{q_{\lambda_2}(\theta\mid \varphi)}{p(\theta \mid Y,\varphi)}, \\
  f_3(\lambda_1,\lambda_3) & =\nabla_{\lambda_3}\kl{q_{\lambda_1,\lambda_3}}{p_{\pow, \eta}}
\end{align*}
where $f_1: \Re^{L_1+L_2}\to \Re^{L_1}$, $f_2: \Re^{L_1+L_2}\to \Re^{L_2}$ and $f_3: \Re^{L_1+L_2}\to \Re^{L_2}$.

\begin{proposition}\label{prop:use_IFT_show_loss_limit}
  Assume $\Lambda^*$ is not empty and let $\lambda^*\in \Lambda^*$ be given. Assume $f$ and $g$ are continuously differentiable in $\lambda$ at $\lambda=\lambda^*$ and the Hessians of $\kl{q_{\lambda_1,\lambda_3}}{p_{\pow, \eta}}$ (in $\lambda_1,\lambda_3$) and $E_{\varphi\sim q_{\lambda_1}}\kl{q_{\lambda_2}(\theta\mid \varphi)}{p(\theta \mid Y,\varphi)}$ (in $\lambda_2$) are invertible at $\lambda=\lambda^*$. For every $\lambda^*\in\Lambda^*$ there is $\delta>0$ and a unique continuous function $\lambda^*(v)$ satisfying $F(\lambda^*(v),v)=0$ for $0\le |v|\le \delta$ and $\lambda^*(0)=\lambda^*$ at $v=0$.
\end{proposition}
\begin{proof}
  If the Jacobian $\partial F/\partial \lambda$ is invertible at $\lambda=\lambda^*$ then by assumptions of the proposition and the Implicit Function Theorem, there exists $\delta>0$ and a unique continuous function $\lambda^*: \Re\to\Re^L$ satisfying $F(\lambda^*(v),v)=0$ for $0\le |v|\le \delta$. It follows that every point in $\Lambda^*$ is the limit as $v\to 0$ of some continuous sequence of solutions to $\nabla_\lambda\mathcal{L}^{(v)}=0$.

  The Jacobian $\partial F/\partial \lambda=\partial f/\partial \lambda$ at $v=0$ since the additive term in $g$ does not contribute at $v=0$.
  It follows that the Jacobian $\partial F/\partial \lambda$ is invertible at $\lambda=\lambda^*$ if $\left\vert\partial f/\partial \lambda\right\vert_{\lambda=\lambda^*}\ne 0$. The Jacobian matrix $\partial f/\partial \lambda$ has a block structure,
  \begin{align*}
    \frac{\partial f}{\partial \lambda} & =\left[\begin{array}{ccc}
                                                     \partial f_1/\partial\lambda_1 & \mathbf{0}_{L_1\times L_2}     & \partial f_1/\partial\lambda_3 \\
                                                     \partial f_2/\partial\lambda_1 & \partial f_2/\partial\lambda_2 & \mathbf{0}_{L_2\times L_2}     \\
                                                     \partial f_3/\partial\lambda_1 & \mathbf{0}_{L_1\times L_2}     & \partial f_3/\partial\lambda_3 \\
                                                   \end{array}\right] \\[0.1in]
                                        & =\left[\begin{array}{ccc}
                                                     L_1 \times L_1 & L_1\times L_2  & L_1 \times L_2 \\
                                                     L_2 \times L_1 & L_2 \times L_2 & L_2\times L_2  \\
                                                     L_2 \times L_1 & L_2\times L_2  & L_2 \times L_2 \\
                                                   \end{array}\right]
  \end{align*}
  with block dimensions in the second line. It follows that the determinant is
  \[
    \left\vert\frac{\partial f}{\partial \lambda}\right\vert=\left\vert\begin{array}{cc}
      \partial f_1/\partial\lambda_1 & \partial f_1/\partial\lambda_3 \\
      \partial f_3/\partial\lambda_1 & \partial f_3/\partial\lambda_3 \\
    \end{array}\right\vert\ \times\ \left\vert\partial f_2/\partial\lambda_2 \right\vert.\\
  \]
  This is just the product of the determinants of the Hessians of $\kl{q_{\lambda_1,\lambda_3}}{p_{\pow, \eta}}$ (in $\lambda_1,\lambda_3$) and $E_{\varphi\sim q_{\lambda_1}}\kl{q_{\lambda_2}(\theta\mid \varphi)}{p(\theta \mid Y,\varphi)}$ (in $\lambda_2$), so $\partial F/\partial\lambda$ is invertible if these Hessians are invertible at $\lambda=\lambda^*$.
\end{proof}

\section{Further details of Experiments.} \label{sec:experiments_extra}

\subsection{Epidemiological Data} \label{sec:epidemiology_extra}

\acrshort*{mcmc} samples from the posterior distribution of the epidemiological model in \cref{subsec:exp_epidemiology} are shown in \cref{fig:epidemiology_mcmc}.
Samples from the \acrshort*{mfvi} approximation fitted using variational \acrshort*{smi} separately at each $\eta$ are shown \cref{fig:epidemiology_mfvi}.

\begin{figure}[ht!]
  \centering
  \includegraphics[width=0.48\textwidth]{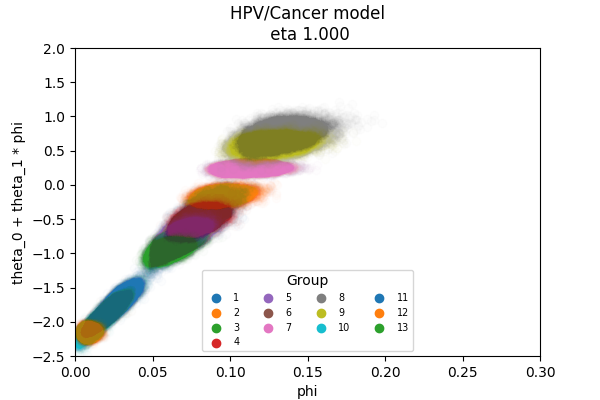}
  \includegraphics[width=0.37\textwidth]{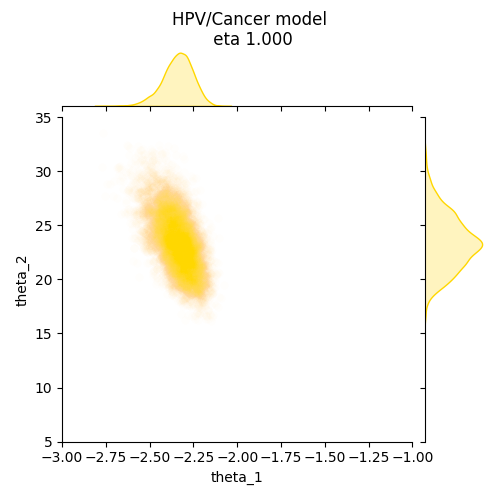}
  \includegraphics[width=0.48\textwidth]{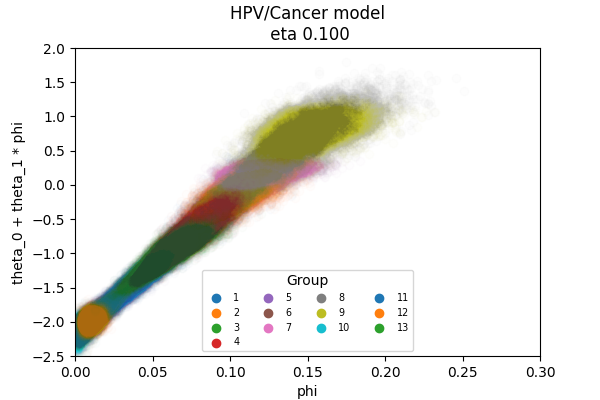}
  \includegraphics[width=0.37\textwidth]{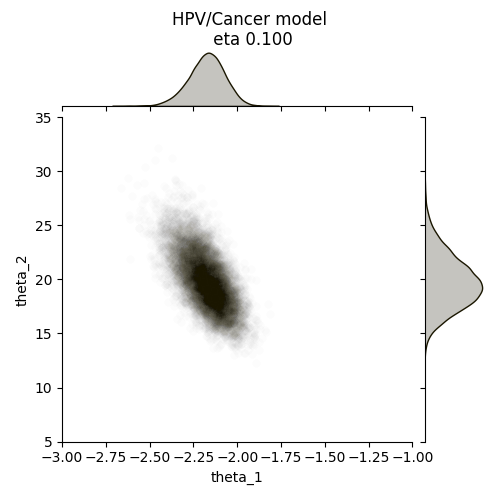}
  \includegraphics[width=0.48\textwidth]{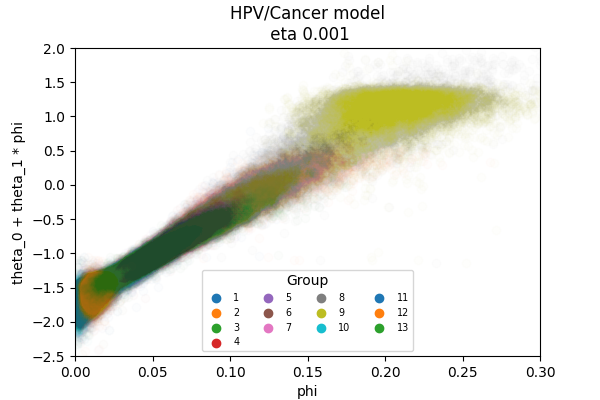}
  \includegraphics[width=0.37\textwidth]{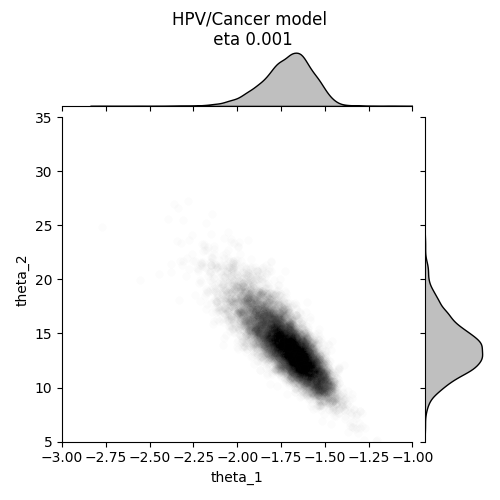}
  \caption[Epidemiology model MCMC]{
    Samples from the posterior distribution of the epidemiological model, obtained via \acrshort*{mcmc}.
    Rows correspond to three rates of \emph{feedback} from the Poisson module, $\eta=(0.001,0.1,1)$.
    In the left column, we plot the relation between HPV prevalence ($\phi$) and cervical cancer incidence ($\mu$) for the 13 groups in the data.
    In the right column, the joint distribution of slope ($\theta_1$) and intercept ($\theta_2$) of such relation.
  }
  \label{fig:epidemiology_mcmc}
\end{figure}

\begin{figure}[!ht]
  \centering
  \includegraphics[width=0.48\textwidth]{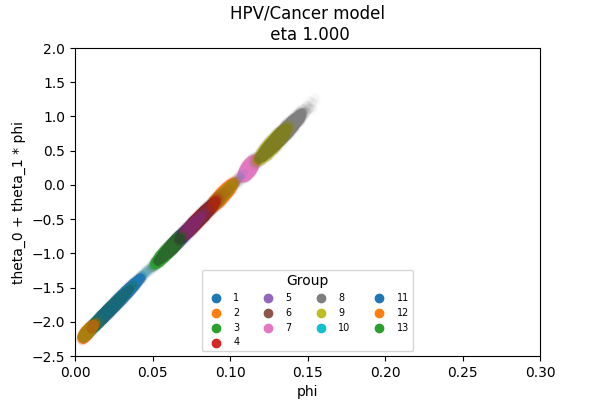}
  \includegraphics[width=0.37\textwidth]{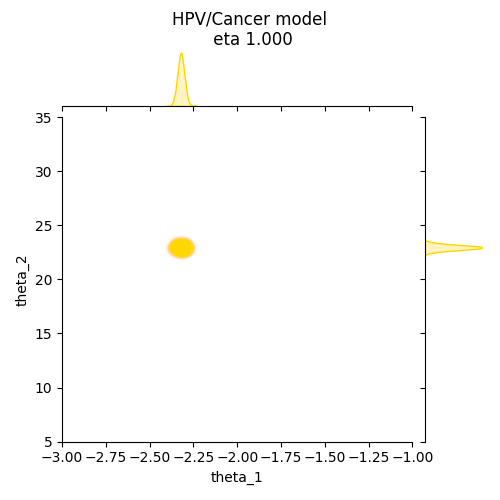}
  \includegraphics[width=0.48\textwidth]{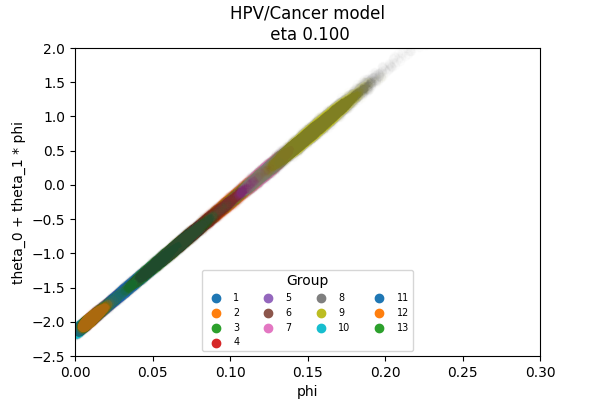}
  \includegraphics[width=0.37\textwidth]{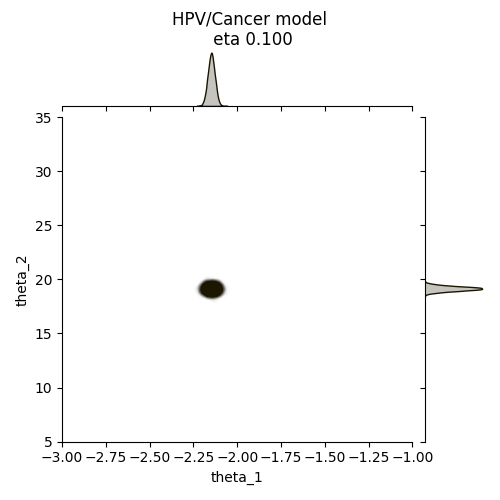}
  \includegraphics[width=0.48\textwidth]{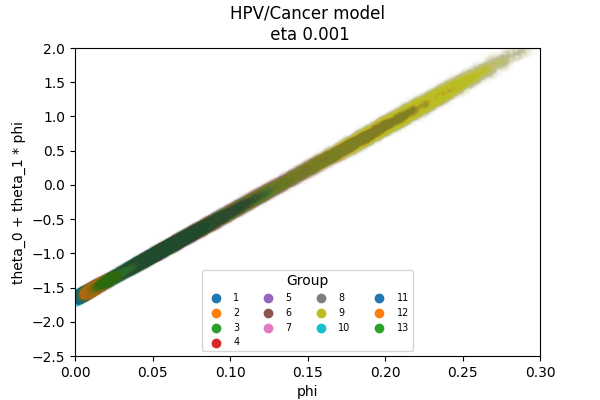}
  \includegraphics[width=0.37\textwidth]{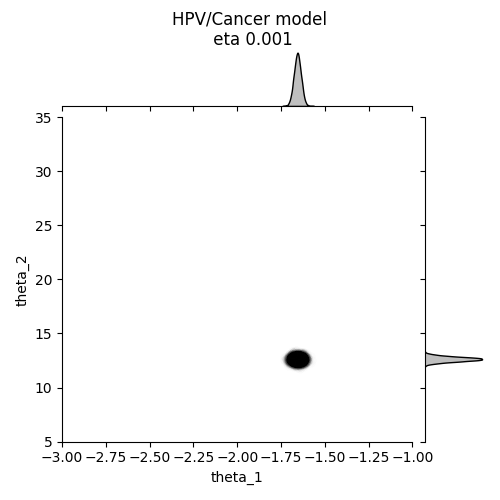}

  \caption[Epidemiology model MFVI]{
    Samples from the variational posterior distribution of the epidemiological model, obtained using a Mean-Field approximation.
    Plots are interpreted as in \cref{fig:epidemiology_mcmc}. We train one variational posterior for each rate $\eta=(0.001,0.1,1)$ separately.
    The approximations clearly underestimate posterior variance.
  }
  \label{fig:epidemiology_mfvi}
\end{figure}

\subsection{Random effects model}\label{subsec:exp_rnd_eff_extra}

Samples from the posterior distribution of the Random Effects model, obtained via \acrshort*{mcmc} are shown in \cref{fig:rnd_eff_mcmc}.

\begin{figure}[!ht]
  \centering
  \includegraphics[width=0.24\textwidth]{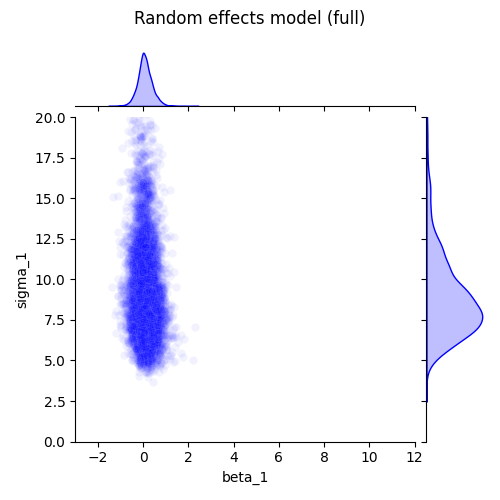}
  \includegraphics[width=0.24\textwidth]{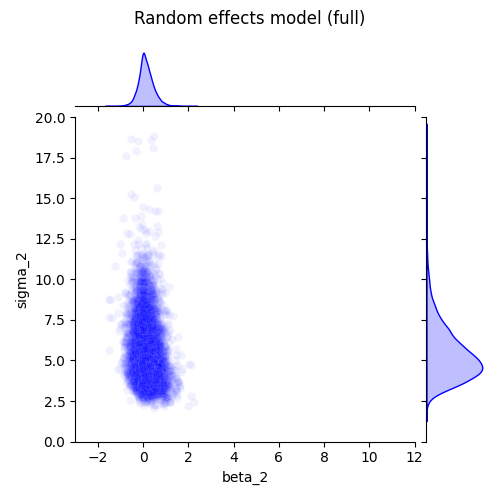}
  \includegraphics[width=0.24\textwidth]{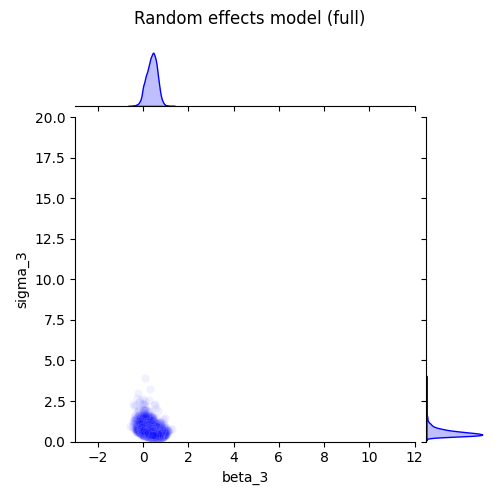}
  \includegraphics[width=0.24\textwidth]{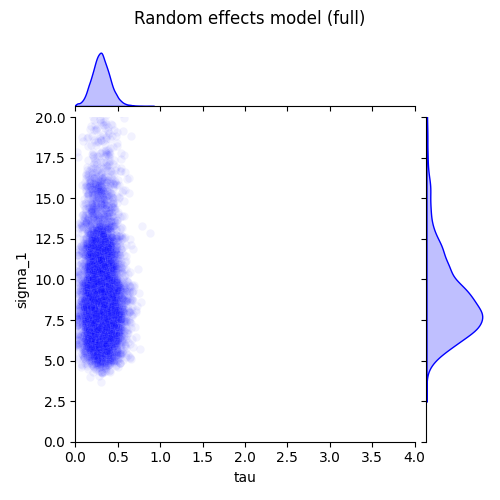}

  \includegraphics[width=0.24\textwidth]{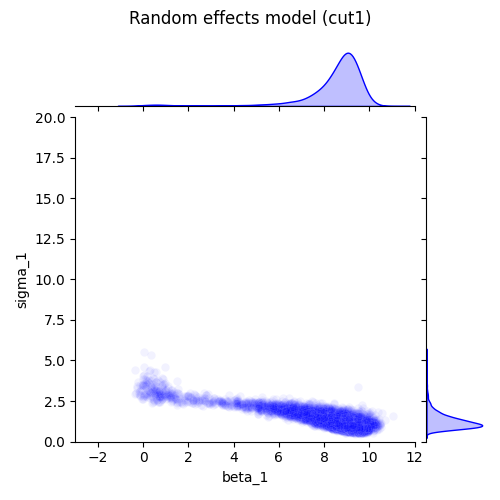}
  \includegraphics[width=0.24\textwidth]{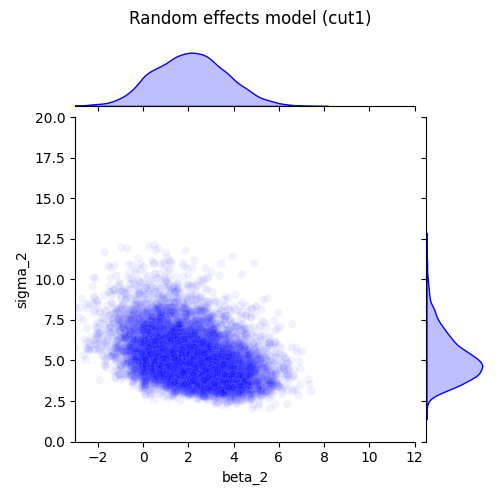}
  \includegraphics[width=0.24\textwidth]{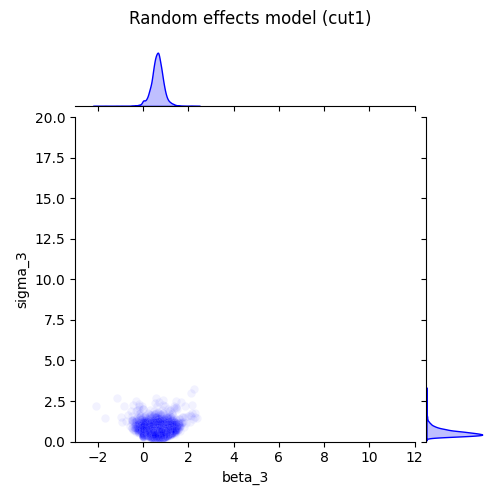}
  \includegraphics[width=0.24\textwidth]{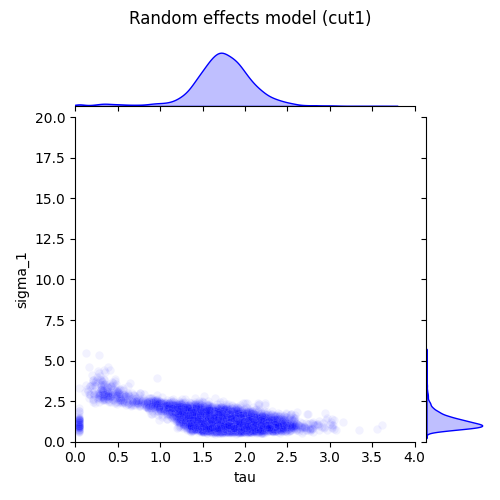}

  \includegraphics[width=0.24\textwidth]{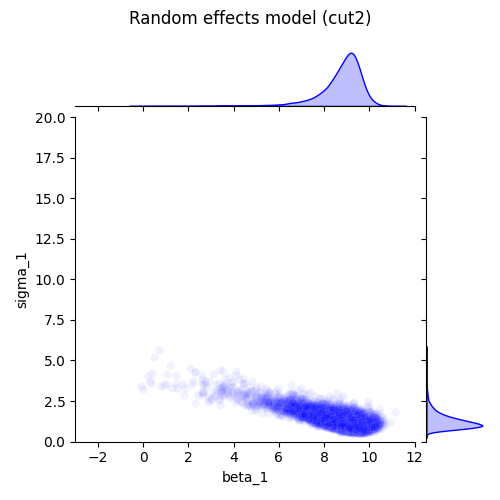}
  \includegraphics[width=0.24\textwidth]{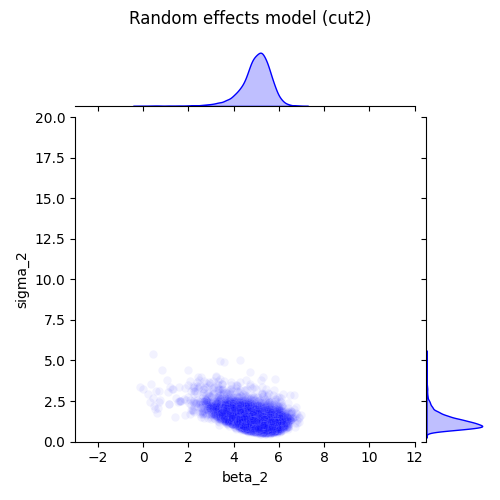}
  \includegraphics[width=0.24\textwidth]{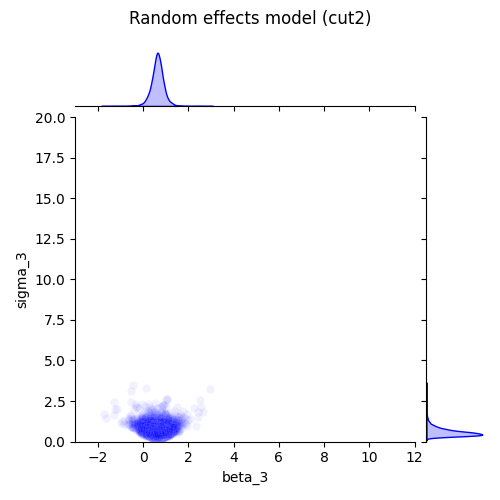}
  \includegraphics[width=0.24\textwidth]{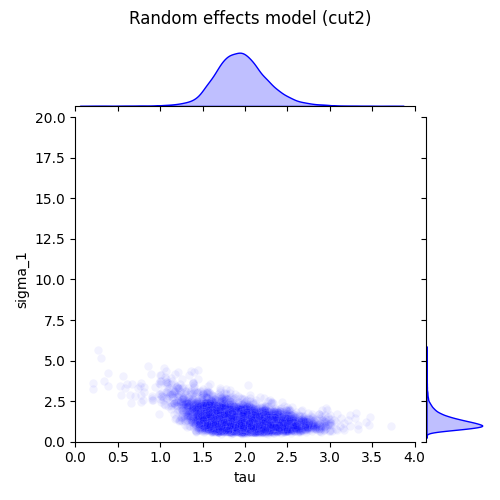}
  \caption[Epidemiology model MCMC]{
    Samples from the posterior distribution of the Random Effects model, obtained via \acrshort*{mcmc}.
    Each graph shows the joint distribution of a selected pair of parameters. Rows correspond to three modular \emph{feedback} configurations between groups: (Top row) Bayes, $\eta_1=...=\eta_{30}=1$; (Middle) One Cut module, $\eta_1=0$, $\eta_2=...=\eta_{30}=1$; (Bottom) Two Cut Modules, $\eta_1=\eta_2=0$, $\eta_3=...=\eta_{30}=1$.
  }
  \label{fig:rnd_eff_mcmc}
\end{figure}



\end{document}